\documentclass{article}

\PassOptionsToPackage{numbers, compress}{natbib}

\usepackage[main,final]{neurips2026/neurips_2026}

\usepackage[utf8]{inputenc} 
\usepackage[T1]{fontenc}    
\usepackage{hyperref}       
\usepackage{url}            
\usepackage{booktabs}       
\usepackage{amsfonts}       
\usepackage{nicefrac}       
\usepackage{microtype}      
\usepackage{xcolor}         
\usepackage{amsmath}
\usepackage{amsthm}
\usepackage{mathtools}
\usepackage{bm}
\usepackage{aliascnt}
\usepackage{cleveref}
\usepackage{mathrsfs}
\usepackage{graphicx}
\usepackage{subcaption}
\usepackage{appendix}

\usepackage{enumitem}

\renewcommand{\textcolor}[2]{#2}
\newcommand{\ag}[1]{\textcolor{blue}{#1}}

\usepackage{wrapfig}

\newenvironment{talign*}
 {\csname align*\endcsname}
 {\endalign}
\newenvironment{talign}
 {\csname align\endcsname}
 {\endalign}

\theoremstyle{plain}

\newtheorem{theorem}{Theorem}[section]

\newaliascnt{lemma}{theorem}
\newtheorem{lemma}[lemma]{Lemma}
\aliascntresetthe{lemma}

\newaliascnt{proposition}{theorem}
\newtheorem{proposition}[proposition]{Proposition}
\aliascntresetthe{proposition}

\newaliascnt{corollary}{theorem}

\aliascntresetthe{corollary}

\theoremstyle{definition}

\newaliascnt{definition}{theorem}
\newtheorem{definition}[definition]{Definition}
\aliascntresetthe{definition}

\newaliascnt{condition}{theorem}

\aliascntresetthe{condition}

\newaliascnt{problem}{theorem}

\aliascntresetthe{problem}

\newaliascnt{example}{theorem}

\aliascntresetthe{example}

\newtheorem{assumption}{Assumption}
\Crefname{assumption}{Assumption}{Assumptions}

\theoremstyle{remark}

\newaliascnt{remark}{theorem}
\newtheorem{remark}[remark]{Remark}
\aliascntresetthe{remark}

\newcommand{\E}{\mathbb{E}}
\newcommand{\R}{\mathbb{R}}
\newcommand{\Z}{\mathbb{Z}}

\newcommand{\Pb}{\mathbb{P}}
\newcommand{\Tb}{\mathbb{T}}
\newcommand{\calH}{\mathcal{H}}

\newcommand{\calT}{\mathcal{T}}
\newcommand{\calL}{\mathcal{L}}

\newcommand{\calF}{\mathcal{F}}
\newcommand{\calG}{\mathcal{G}}

\newcommand{\calK}{\mathcal{K}}
\newcommand{\calD}{\mathscr{D}}
\newcommand{\calO}{\mathcal{O}}

\newcommand{\calX}{\mathcal{X}}
\newcommand{\calP}{\mathcal{P}}
\newcommand{\calN}{\mathcal{N}}

\newcommand{\dd}{\mathrm{d}}

\newcommand{\bH}{\mathbf{H}}

\newcommand{\bJ}{\mathbf{J}}

\newcommand{\Id}{\mathrm{Id}}
\newcommand{\N}{\mathbb{N}}

\newcommand{\ip}[2]{\left\langle #1, #2 \right\rangle}
\newcommand{\mbf}{\mathbf{f}}
\newcommand{\ran}{\operatorname{Ran}}
\renewcommand{\ker}{\operatorname{Ker}}
\newcommand{\op}{\operatorname{op}}

\newcommand{\tr}{\mathrm{Tr}}

\newcommand{\pth}[1]{\left( #1 \right)}

\newcommand{\sth}[1]{\left\{ #1 \right\}}
\newcommand{\mmd}{\operatorname{MMD}}
\newcommand{\dmmd}{\operatorname{DrMMD}}
\newcommand{\smmd}{\operatorname{SrMMD}}
\newcommand{\ksd}{\operatorname{KSD}}
\newcommand{\1}{\mathbf{1}}

\newcommand{\expect}[2][]{\mathbb{E}_{#1}\left[ #2 \right]}
\newcommand{\prob}[2][]{\mathbb{P}_{#1}\left[ #2 \right]}
\title{Sobolev Regularized MMD Gradient Flow}

\author{%
    Chenyang Tian$^{1,}$\thanks{Equal Contribution. This work was done
when the first author was an undergraduate student at Tsinghua University.},
Bharath K. Sriperumbudur$^2$,
Arthur Gretton$^{1,3}$,
Zonghao Chen$^{1,}\footnotemark[1]$
     \\
\small $^1$University College London,
\small $^2$Pennsylvania State University,
\small $^3$Google Deepmind \\
    \texttt{\{chenyang.tian.26, zonghao.chen.22\}@ucl.ac.uk}, \\
    \texttt{\{bharathsv.ucsd, arthur.gretton\}@gmail.com} \\
}

\begin{document}

\maketitle

\begin{abstract}
We \textcolor{black}{analyze Sobolev descent~\citep{mroueh2019sobolev}}, a regularized variant of maximum mean discrepancy (MMD) gradient flow based on a gradient penalty on the MMD witness function. 
This regularization mitigates the non-convexity of the MMD objective, and yields provable \emph{global} convergence guarantees in MMD, in both continuous and discrete time. 
A surprising appeal is that our convergence analysis does not rely on isoperimetric assumptions on the target distribution. 
Instead, it is based on a regularity condition on the difference between kernel mean embeddings. 
We \textcolor{black}{extend Sobolev descent} such that it is applicable in both sampling (from an unnormalized target distribution)---using Stein kernels---and generative modeling settings, unlike previous works, where a gradient flow is suitable for only generative modeling or sampling but not both. 
We empirically validate the method on a broad range of tasks across both settings. 
\end{abstract}
\section{Introduction}
Sampling and generative modelling, while often studied as distinct tasks, can be cast within a common framework of optimization over the space of probability measures $\calP$ when only partial information about the target distribution $\pi$ is available. In sampling, $\pi$ is typically known only up to a normalizing constant, whereas in generative modelling it is accessible only through some i.i.d. samples. 
Indeed, algorithmic developments in one area have often inspired progress in the other. For instance, diffusion models, which have proved successful in generative modelling~\citep{songscore}, have recently motivated the development of new diffusion-based sampling methods~\citep{huangreverse,vachersampling}. Another example is provided by Wasserstein gradient flows, such as the maximum mean discrepancy (MMD) flow~\citep{arbel2019maximum} and the kernel Stein discrepancy (KSD) flow~\citep{korba2021kernel}: the former is studied in generative modelling, while the latter is primarily used for sampling. 

\ag{We address the case of a gradient flow based on a gradient-regularized integral probability metric (IPM), known as Sobolev descent \cite{mroueh2019sobolev}. This notion of gradient regularization   was originally popularized in generative adversarial networks (GANs), a class of generative models \ag{since} overshadowed by diffusion models: when an IPM is used as the critic, then \emph{gradient regularization} of the critic witness function leads to higher-quality images~\citep{roth2017stabilizing,gulrajani2017improved,arbel2018gradient,mroueh2019sobolev,mroueh2020unbalanced}.} \ag{Moving beyond the GAN setting, gradient regularization as in  \cite{mroueh2019sobolev}}  also improves the convergence of both flow-based generative modelling via MMD flow and sampling via KSD flow~\citep{korba2021kernel}. 
As KSD is a special case of MMD with the Langevin Stein kernel~\citep{hagrass2026minimax}, 
we refer to both resulting flows collectively as \emph{Sobolev-regularized Maximum Mean Discrepancy gradient flow} (SrMMD flow)---in the literature, it is referred to as \emph{Sobolev Descent} \cite{mroueh2019sobolev}---\textcolor{red}{emphasizing the gradient regularization imposed on the witness function.} 

\ag{In this work, we establish that} SrMMD flow enjoys provably fast convergence in both sampling and generative modelling. 
For sampling, its convergence rate is independent of isoperimetric characteristics of the target distribution $\pi$, denoted by $\beta$, such as log-Sobolev or Poincar\'{e} constants~\citep{bakry2014analysis}. 
In particular, the Unadjusted Langevin Algorithm (ULA) requires $\calO(\delta^{-1}\beta)$ iterations to achieve error $\delta>0$~\citep{chewi2024analysis,vempala2019rapid}, whereas SrMMD flow achieves the same accuracy in $\calO(\delta^{-1/\alpha})$ iterations, where $\alpha$ characterizes the regularity of the trajectory. 
This suggests a clear advantage of SrMMD flow for multi-modal targets, where isoperimetric constants $\beta$ may scale exponentially with the separation between modes~\citep{menz2014poincare,chewi2024logconcavesampling}. 
We admit, however, that our guarantee is stated in KSD, which is a weaker metric than the KL or $\chi^2$ divergence commonly used to analyze ULA~\citep{kanagawa2025controlling}. 

For generative modelling, although convergence in MMD is not the primary metric of practical interest, compared for instance with perceptual metrics such as the FID score~\citep{heusel2017gans}, our analysis provides a \emph{theoretical justification} that gradient regularization can accelerate convergence. 
\textcolor{red}{
Our analysis complements the original Sobolev Descent~\citep{mroueh2019sobolev} in the following aspects: (i) The RKHS can be possibly infinite-dimensional, unlike \citep{mroueh2019sobolev}, which handles only finite-dimensional RKHS; 
(ii) we establish a discrete-time convergence guarantee; and
(iii) we explicitly characterize the residual term arising from the regularization.} Together with the recent work on DrMMD flow~\citep{chen2025regularized}, which provides a theoretical justification for $L_2$ regularization of the witness/critic function, these results help explain why combining $L_2$ and gradient regularization leads to improved empirical performance in MMD GANs~\citep{arbel2018gradient,mroueh2020unbalanced}. 
This addresses a longstanding open question in the literature. 
Although both of these regularization schemes are most prominent in the GAN literature, they have also appeared in flow-based generative models~\citep{finlay2020train,galashov2024deep,gu2024lipschitz}.

\vspace{1mm}
\textbf{Contributions.}
We \textcolor{red}{analyze Sobolev Descent \citep{mroueh2019sobolev}, a.k.a.}~SrMMD flow, a regularized variant of MMD/KSD flow based on a gradient norm penalty on the witness function. 
The proposed framework applies both to generative modelling and to sampling from unnormalized targets, the latter through the use of Stein kernels.\ag{ We begin in \cref{{sec:background}} with an introduction to Wasserstein gradient flows on the MMD and KSD, and then describe the Sobolev-regularized MMD flow in  \cref{sec:mainSobolevMMD}. }
We prove in \cref{sec:convergence} that the SrMMD  gradient regularization yields \emph{global} convergence guarantees in both continuous- and discrete-time settings, under a source condition on the trajectory.  
In particular, the convergence rate measured in MMD/KSD is independent of the isoperimetric properties of the target distribution. 
In \cref{sec:particle_flow}, we show that SrMMD flow admits a tractable finite-sample implementation which, in contrast to diffusion-based approaches~\citep{song2019generative,ho2020denoising,huangreverse}, does not require learning a score function at each iteration. 
Finally, in \cref{sec:experiments}, we validate the proposed method empirically on a broad range of tasks, including sampling, color transfer, and student-teacher network.

\vspace{1mm}
\textbf{Notations.}
Let $\calP_2(\R^d)$ be the space of probability distributions on $\R^d$ with finite second moment. 
For a continuous mapping $T: \R^d \to \R^d$, $T_{\#} \mu$ denotes the push-forward measure of $\mu$. 
Given a kernel $k : \R^d\times\R^d \to \R$, we denote by $\calH$ its corresponding RKHS, endowed with 
the inner product $\ip\cdot\cdot_{\mathcal H}$ and norm $\|\cdot\|_{\mathcal H}$. 
For any probability measure $\mu \in \mathcal{P}(\mathbb{R}^d)$, $(L_2(\mu),\|\cdot\|_{L_2(\mu)})$ is the space of all real valued $\mu$-square integrable measurable functions on $\mathbb{R}^d$ with $\|f\|_{L_2(\mu)}^2:=\int_{\mathbb{R}^d}|f(x)|^2 \dd \mu(x)$. 
$(L_2^d(\mu),\|\cdot\|_{L_2^d(\mu)})$ is defined analogously with $\R^d$-valued functions. 
Given a differentiable function $f: \mathbb{R}^d \rightarrow \mathbb{R}$, let $\partial_{\ell} f(x)$ denote the partial derivative with respect to the $\ell$-th coordinate in $x$ and denote $\nabla f(x)=[\partial_1 f(x), \ldots, \partial_d f(x)]^{\top} \in \mathbb{R}^d$. Let $\partial_{1, \ell} k(x, x^{\prime})$ denote the partial derivative with respect to the $\ell$-th coordinate of $x$ and let $\partial_{1, \ell} \partial_{2, \ell} k(x, x^{\prime})$ denote the mixed partial derivative with respect to the $\ell$-th coordinate in $x$ and the $\ell$-th coordinate in $x^{\prime}$. 
Moreover, denote $\nabla_{1} k(x, x^{\prime})=[\partial_{1,1} k(x, x^{\prime}), \ldots, \partial_{1,d} k(x, x^{\prime})]^{\top} \in \mathbb{R}^d$, and define $\nabla_{2} k(x, x')$ analogously. 
Denote $[S]=\{0,1,\ldots,S\}$ for $S\in\N$. 
The notation $\pi\in\calP_2(\R^d)$ denotes explicitly the target distribution. 
We assume throughout that the kernel $k$ and the target density $\pi$ are sufficiently continuously differentiable for all displayed derivatives to be well defined. 

\section{Background \& Related Work}\label{sec:background}
In this section, we introduce the necessary background and present related work in the literature. 
\subsection{Wasserstein gradient flow}
Let $\calF:\mathcal{P}_2(\mathbb{R}^d)\to\mathbb{R}$ be a functional with a unique minimizer $\pi$. In many cases, $\calF$ is of the form $\calF(\cdot)=\mathcal{D}(\cdot \|\pi)$, where $\mathcal{D}$ is some probability divergence. 
The Wasserstein gradient flow (WGF) provides a principled approach to approximating $\pi$ by evolving a curve of probability measures $(\mu_t)_{t\ge 0}$ along the steepest descent direction of $\calF$ in the Wasserstein space $(\calP_2(\R^d), W_2)$, namely the space of probability measures endowed with the Wasserstein-2 distance~\citep{ambrosio2005gradient,santambrogio2017euclidean}. 
Formally, the WGF is characterized by the continuity equation \citep{ambrosio2005gradient}: given an initial distribution $\mu_0\in\mathcal{P}_2(\mathbb{R}^d)$,
\begin{align}\label{eq:cont_eqn}
    \partial_t \mu_t - \nabla \cdot\left(\mu_t \bm{v}[\mu_t] \right)=0, 
\end{align}
where the vector field $\bm{v}[\mu_t]:\R^d\to\R^d$ is the Fréchet subdifferential of $\mathcal{F}$ evaluated at $\mu_t$~\citep[Definition 11.1.1]{ambrosio2005gradient}. 
Under some mild growth conditions on $\bm{v}[\mu_t]$, the trajectory $(\mu_t)_{t\geq 0}$ is a unique absolutely continuous curve in $\calP_2(\R^d)$~\citep[Lemma 8.1.4]{ambrosio2005gradient}. 
Moreover, the trajectory $(\mu_t)_{t\geq 0 }$ admits the following characteristic (particle) representation: 
\begin{align}\label{eq:particleode}
\dd X_t =-\bm{v}[\mu_t](X_t) \dd t, \quad \mathrm{Law}(X_t) = \mu_t . 
\end{align}
Eq.~\eqref{eq:cont_eqn} and Eq.~\eqref{eq:particleode} provide two equivalent yet complementary viewpoints. Eq.~\eqref{eq:cont_eqn} gives the Eulerian perspective, which describes the evolution of the density, whereas Eq.~\eqref{eq:particleode} gives the Lagrangian perspective, which focuses on the trajectories of individual particles.
In practice, implementations of WGF are typically based on the Lagrangian formulation, which amounts to simulating an interacting particle system. Specifically, given a fixed time horizon $S\in\N$ and an initial collection of $N$ particles $\{x_0^{(i)}\}_{i=1}^N \sim \mu_0$, for $s \in [S]$, 
\begin{talign}\label{eq:particle_system}
    x_{s+1}^{(i)} = x_{s}^{(i)} - \gamma \bm{v}[ \hat{\mu}_s ](x_{s}^{(i)}), \quad \hat{\mu}_s = \frac{1}{N} \sum_{j=1}^N \delta_{x_{s}^{(j)}} . 
\end{talign}
Here, $\gamma>0$ is the step size. 
Compared with Eq.~\eqref{eq:particleode}, we see that the particle system in Eq.~\eqref{eq:particle_system} involves both a forward Euler discretization in \emph{time} and a finite-particle discretization in \emph{space}. 

\subsection{MMD and KSD gradient flow} 
The maximum mean discrepancy (MMD)~\citep{gretton2012kernel} between two probability measures is defined as the RKHS norm of the difference between their corresponding kernel mean embeddings~\citep{muandet2017kernel}, i.e.,  
$\mmd(\mu,\pi) =\|\int k(x, \cdot) \; \dd (\mu - \pi)(x)\|_\calH =:\|m_\mu - m_\pi\|_\calH$, where $m_\mu:=\int k(x,\cdot)\,\dd\mu(x)$ is called the kernel mean element w.r.t.~$\mu$. 
When the functional $\calF$ takes the form of a squared MMD, i.e., 
$\calF_{\mmd}(\mu) = \frac{1}{2} \mmd^2(\mu, \pi)$, the corresponding 
Fréchet subdifferential is $\bm{v}_{\mmd}[\mu](\cdot) = \int \nabla_2 k(x, \cdot) \; \dd (\mu - \pi)(x)$~\citep{arbel2019maximum}, which makes the MMD flow well-suited to generative modeling since $\bm{v}_{\mmd}[\mu](\cdot)$ can be estimated from samples from $\mu$ and $\pi$. Indeed, given 
$M$ i.i.d. target samples $\{y^{(i)}\}_{i=1}^M$ from $\pi$ and $N$ particles $\{x_s^{(j)}\}_{j=1}^N$ at iteration $s$, $\bm{v}_{\mmd}[\hat{\mu}_s]$ admits a direct unbiased plug-in estimator $\bm{v}_{\mmd}[\hat{\mu}_s] \approx \frac{1}{N} \sum_{j=1}^N \nabla_2 k(x_s^{(j)}, \cdot) - \frac{1}{M} \sum_{j=1}^M \nabla_2 k(y^{(j)}, \cdot)$. 
The kernel Stein discrepancy (KSD) is a special instance of MMD under a specific choice of kernel $k_\pi$ that satisfies  the Stein identity $\int k_\pi(x, y) \dd\pi(x)=0, \forall y\in\R^d$~\citep{liu2016kernelized,chwialkowski2016kernel}.  
Arguably, the most popular choice of $k_\pi$ is the \emph{Langevin Stein kernel}, defined through the \emph{score}, $\mathfrak{s}(x) = \nabla \log \pi(x)$, and a base kernel $k:\R^d\times\R^d\to\R$:  
\begin{talign}\label{eq:steinkernel}
k_{\pi}(x,y) := \mathfrak{s}(x)^\top \mathfrak{s}(y) k(x,y) + \mathfrak{s}(x)^\top \nabla_2 k(x,y) + \nabla_1 k(x,y)^\top \mathfrak{s}(y) + \mathrm{Tr}(\nabla_{1}\nabla_2 k(x,y)),
\end{talign}
where \textcolor{red}{$\mathrm{Tr}(\nabla_{1}\nabla_2 k(x,y)) = \sum_{\ell=1}^d \partial_{1,\ell}\partial_{2,\ell} k(x,y)$.}  
Under mild tail conditions on the base kernel $k$, $k_\pi$ satisfies the Stein's identity $\int k_\pi(x, y) \mathrm{d} \pi(x)=0$ for all $y\in\R^d$ (\Cref{lem:stein_identity}). 
As a result, the corresponding 
Fréchet subdifferential at iteration  $s$ equals $\bm{v}_{\ksd}[\hat{\mu}_s](\cdot) = \frac{1}{N}\sum_{j=1}^N \nabla_2 k_\pi(x_s^{(j)}, \cdot)$~\citep{korba2021kernel}.
This makes the KSD flow particularly suitable for sampling from the target $\pi$, which is known only up to a normalization constant.

%
%
To ensure the well-definedness of the MMD (resp. KSD) gradient flow, in the sense that the trajectory is absolutely continuous in $\calP_2(\R^d)$, unique and does not explode in finite-time, 
we need to impose the following sufficient conditions on the growth and boundedness of the kernel $k$ (resp. $k_\pi$). 
\begin{assumption}[Growth I]\label{ass:smmd_growth}
There exists a constant $M>0$ such that for any $x\in\R^d$, $|k(x, x)| \leq M (1 + \|x\|^4)$; and for any $i\in[d]$, $|\partial_{1,i} \partial_{2,i} k(x, x)| \leq M (1 + \|x\|^2)$. 
\end{assumption} 
\begin{assumption}[Local Boundedness]\label{ass:bounded}
Let $\calK=B(0,R)$ be a ball of radius $R$ in $\R^d$. 
 There exists a constant $B_R$ such that $\sup_{x\in\calK} \| \partial_{1,i} \partial_{1, j} k(x, \cdot) \|_\calH \leq B_R$ for any $i,j\in[d]$. 
\end{assumption} 
\textbf{Well-posedness.}\hspace{4mm}A proof of the well-posedness is provided in \Cref{prop:mmd_flow_defined}. 
Compared with the conditions imposed in the original MMD flow paper~\citep{arbel2019maximum}, our assumptions constitute a strict relaxation. In particular, Assumptions (A)–(D) in \cite{arbel2019maximum} enforce global boundedness of $\sup_{x\in\R^d} \| \partial_{1,i} k(x, \cdot) \|_\calH$ and $\sup_{x\in\R^d} \| \partial_{1,i} \partial_{1, j} k(x, \cdot) \|_\calH$, whereas here we only require \emph{local} boundedness. 
As a consequence, beyond the globally bounded kernels commonly used in the Wasserstein gradient flow literature—such as Gaussian and Matérn kernels—our \Cref{ass:smmd_growth,ass:bounded} also accommodate certain unbounded kernels, including the polynomial kernel $k(x,y) = (x^\top y+1)^2$, and the Stein kernel $k_\pi$ under some conditions on $\pi$~\citep{korba2021kernel}; see the next section for details. 
Extension to kernels with even faster growth would be challenging; indeed, classical theory of differential equations shows that linear-growth conditions on the vector field are crucial in ruling out finite-time blow-up~\citep[Problem 2.18]{teschl2012ordinary}. These conditions exclude, however, the negative distance kernel $k(x,y) = -\|x-y\|$~\citep{hertrich2023generative,hagemannposterior}, Riesz kernels $k(x,y) = -\|x-y\|^r$ for $r\in(0,2)$~\citep{sejdinovic13energy,altekruger2023neural,hertrich2024wasserstein} and the Coulomb kernels $k(x,y)=(d-2)\|x-y\|^{-(d-2)}$~\citep{boufadene2023global} due to the singular behaviour at the origin. 
Although the MMD gradient flows of these kernels exhibit strong empirical performances in practice, establishing rigorous well-posedness remains an open problem; existing results are essentially restricted to the torus~\citep{boufadene2023global,chizat2026quantitative} or the real line $\R^1$~\citep{duong2026wasserstein}. For this reason, we do not consider such kernels in the present work. 

\vspace{1mm}
\textbf{Convergence.}\hspace{4mm}
Beyond well-posedness, another central question for Wasserstein gradient flows is convergence, namely whether $\lim_{t\to\infty}\mu_t = \pi$ under some proper metric. 
A classical sufficient condition is that the functional $\mathcal{F}$ be convex along Wasserstein-2 geodesics, i.e., \emph{geodesically} convex~\citep{villani2009optimal}. However, this property typically holds only for $f$-divergences and a restricted class of log-concave target distributions $\pi$~\citep{wibisono2018sampling,chewi2024logconcavesampling}. By contrast, the MMD and KSD functionals are \emph{not} geodesically convex, even when $\pi$ is log-concave~\citep{arbel2019maximum,korba2021kernel}. To address this limitation, a recent line of work has sought to improve convergence of these WGFs through spectral regularization~\citep{chewi2020svgd,chen2025regularized,he2022regularized}. 


\section{Sobolev-regularized Maximum Mean Discrepancy (SrMMD) Flow}\label{sec:mainSobolevMMD}
In this section, we \ag{review the Sobolev descent procedure defined in \cite{mroueh2019sobolev}, aka  SrMMD flow, which is } a form of Sobolev regularization for MMD and KSD gradient flows. \textcolor{red}{We then establish a new well-posedness result for this flow.} 
In the next section, we study its convergence in both continuous and discrete time.  

\begin{definition}[SrMMD, \textcolor{red}{Definition 2 of \cite{mroueh2019sobolev}}]
\label{defi:smmd}
Let $\lambda>0$ be fixed. 
For two fixed $\mu,\pi\in\calP_2(\R^d)$, 
the Sobolev regularized Maximum Mean Discrepancy (SrMMD) is defined as
\begin{align}\label{eq:srmmd_defi}
\smmd(\mu,\pi)=\sup_{f\in\mathcal{H}:\,\|\nabla f\|^{2}_{L_2^d(\mu)}+ \lambda\|f\|_{\mathcal{H}}^{2}\leq 1}\int f \dd(\mu-\pi).   
\end{align}
\end{definition}
\begin{definition}[SrMMD flow, \textcolor{red}{Section 3 of \cite{mroueh2019sobolev}}] \label{defi:defofsmmdgf}
Let $\lambda>0$ be fixed. 
For an initial distribution $\mu_0\in\calP_2(\R^d)$, the SrMMD flow is a trajectory of probability measures $(\mu_t)_{t\ge 0}$ satisfying the continuity equation
$$\partial_t \mu_t -\nabla\cdot(\mu_t\nabla f_{\mu_t,\pi} )=0,$$ or equivalently, admitting a characteristic representation: 
\begin{align}\label{eq:defofsmmdgf}
   \dd X_t =-\nabla f_{\mu_t,\pi}(X_t) \dd t, \quad \mathrm{Law}(X_t) = \mu_t,
\end{align}
where $f_{\mu,\pi}\in \mathcal{H}$ is the witness function that attains the supremum in \Cref{defi:smmd}. 
\end{definition}
It is immediate that SrMMD belongs to the family of integral probability metrics~\citep{muller1997integral}. 
In contrast to $\mmd(\mu,\pi)=\|m_\mu - m_\pi\|_\calH=\sup_{\|f\|_\calH \leq 1} \int f \dd(\mu-\pi)$ whose hypothesis class is the unit ball of RKHS $\calH$~\citep{gretton2012kernel}, SrMMD enforces an additional \emph{gradient penalty} on the witness function $f$. 
The resulting gradient field \(\nabla f_{\mu,\pi}\) is then used in SrMMD flow to transport particles. 
Gradient penalization of the witness function has been widely used as a critical heuristic in generative modelling, both in classical generative adversarial networks~\citep[Eq.~(3)]{gulrajani2017improved}\citep{binkowski2018demystifying}\citep[Eq.(6)]{arbel2018gradient} and in recent flow-based approaches~\citep[Eq.~(9)]{galashov2024deep}. 

The witness function $f_{\mu,\pi}$ actually admits a closed-form expression. 
To see this, we consider a linear operator $D_\mu: \; f\in\calH\mapsto \nabla f$ which maps a function $f$ to its first-order derivatives.  
When the kernel satisfies $\int \partial_{1,i}\partial_{2,i} k(x, x) \dd \mu(x) < \infty$ for any $i\in[d]$, the operator $D_\mu$ is well defined, Hilbert-Schmidt, $\operatorname{ran}(D_\mu)\subseteq L_2^d(\mu)$, and it admits a well-defined adjoint operator $D_\mu^*:L_2^d(\mu)\to \calH$ (see \Cref{lem:D_mu}). 
Let $S_\mu=D_\mu^*D_\mu :\calH\to\calH$ be the \emph{gradient covariance} operator, which is bounded, linear, Hilbert-Schmidt (see \Cref{lem:hilbertschmidt}), satisfying $\langle f, S_\mu g\rangle_\calH = \langle D_\mu f, D_\mu g \rangle_{L_2^d(\mu)} = \langle\nabla f,\nabla g\rangle_{L_2^d(\mu)}$ for all $f,g\in \calH$. 

\begin{proposition}[Witness function $f_{\mu,\pi}$, \textcolor{red}{Proposition 1 of \cite{mroueh2019sobolev}}] \label{prop:expressionofsmmd}
Suppose $\int k(x,x) \dd\mu(x) < \infty$, $\int k(x,x) \dd\pi(x) < \infty$, $\int \partial_{1,i}\partial_{2,i} k(x, x) \dd \mu(x) < \infty$ for all $i\in[d]$. 
Then, $f _{\mu,\pi}$ attaining the supremum in the definition of SrMMD (\Cref{defi:smmd}) admits a closed form expression: $f_{\mu,\pi} = (S_\mu+\lambda \Id)^{-1}(m_\mu-m_\pi)$.
\end{proposition}
\textcolor{red}{
The original Proposition~1 of~\cite{mroueh2019sobolev} is proved under a bounded-kernel assumption. 
We extend this result to unbounded kernels, which is essential for the sampling setting. A proof is given in \Cref{apd:proof_of_expressionofsmmd} for completeness. } 
Although the expression for $f_{\mu,\pi}$ appears involved, since it depends on an inverse operator, we show in \Cref{sec:particle_flow} that its dual representation leads to a closed-form expression with a direct plug-in finite particle implementation. 
If the target $\pi$ is known only up to a normalization constant, then under a Stein kernel $k_\pi$, we have $f_{\mu,\pi} = (S_\mu+\lambda \Id)^{-1} m_\mu$ since $m_\pi=0$, which admits a direct plug-in finite particle implementation. 

The following result establishes the well-posedness of the SrMMD flow defined in \Cref{defi:defofsmmdgf}, and is proved in \Cref{sec:proof_ac_curve}.
\begin{theorem}[Well-posedness of the SrMMD flow]\label{thm:ac_curve} Suppose that \Cref{ass:bounded,ass:smmd_growth} hold, and let $T \in [0,\infty)$ be a finite time horizon. 
Let $\mu_0,\pi\in\calP_2(\R^d)$ with $\mmd(\mu_0,\pi) < C_{\mmd}$. 
Suppose $\mu_0$ admits a density and is supported on a compact domain. 
Then, $(\mu_t)_{t\in[0,T]}$ is a unique absolutely continuous curve in $\calP_2(\R^d)$, each $\mu_t$ admits a density, and the flow does not blow up in finite time. 
\end{theorem}
In the context of sampling, for a Stein kernel $k_\pi$ constructed from a smooth translation invariant base kernel $k$, \Cref{ass:smmd_growth,ass:bounded} hold when the score function satisfies $\|\mathfrak{s}(x)\| \leq M_0 (1 + \|x\|)$ and $\| \bJ \mathfrak{s}(x)\|_F \leq M_0 (1 + \|x\|)$ for some constant $M_0$ (\Cref{prop:suff_stein_kernel}). 
Here, $\bJ \mathfrak{s}(x) \in \R^{d \times d}$ denotes the Jacobian matrix of the score function $s$.
Compared with the original KSD flow of \cite{korba2021kernel}, which assumes $\|\mathfrak{s}(x)\| \leq M_0 (1 + \sqrt{\|x\|})$ and $\|\bJ \mathfrak{s}(x)\|_F \leq M_0 (1 + \sqrt{\|x\|})$, our \Cref{thm:ac_curve} allows weaker growth conditions on $s$. 
Our growth conditions on $s$ are satisfied by a broader class of targets, such as Gaussian-type and exponential-type distributions~\citep{liu2016stein,korba2021kernel}. 
A closely related work is \cite{duong2025wasserstein}, which studies MMD flows with positive and negative distance kernels on the real line, however, their Sobolev regularization is introduced in the quantile-space to ensure existence of the flow~\citep{duong2026wasserstein}. 

\begin{remark}\label{rem:wgf}
The SrMMD flow in \Cref{defi:defofsmmdgf} is \emph{not} the Wasserstein gradient flow of the functional $\calF(\mu):= \smmd(\mu, \pi)$. 
Indeed, such a gradient flow would require differentiating the operator $(S_\mu+\lambda\Id)^{-1}$ with respect to $\mu$, which is analytically cumbersome.
Instead, we introduce a Sobolev-regularized variant of the MMD flow, whose vector field is given by $\nabla f_{\mu,\pi} = \nabla (S_\mu+\lambda\Id)^{-1}(m_\mu-m_\pi)$, in contrast to the standard MMD flow, whose vector field is $\nabla (m_\mu-m_\pi)$. 
However, the SrMMD flow can be identified as the gradient flow of MMD functional under a different metric; see \Cref{sec:srmmd_geometry}. 
\end{remark}

\begin{remark}[Connection with DrMMD flow]
A closely related flow is the (De)-regularized MMD flow, which is the Wasserstein flow of  $\dmmd(\mu,\pi)=\sup_{\|f\|^{2}_{L_2(\pi)}+\lambda\|f\|_{\mathcal{H}}^{2}\leq 1}\int f \,\dd(\mu-\pi)$~\citep{chen2025regularized,neumayer2024wasserstein}. 
Compared with the SrMMD flow, there are two key differences: 
(i) SrMMD employs gradient regularization, and (ii) the regularization in SrMMD is defined with respect to $\mu$ instead of $\pi$. 
We also study a hybrid regularization that combines both $L_2$ and gradient penalties, referred to as Hybrid-regularized MMD (HrMMD) flow, which matches the actual regularization schemes used in MMD GANs~\citep{arbel2018gradient}. The HrMMD objective is the same as the Sobolev-Fisher discrepancy studied by 
\cite{mroueh2020unbalanced}, where it was subsequently used as the objective to be minimized in the 
Wasserstein-Fisher-Rao geometry. See \Cref{sec:hrmmd} for details. 
\end{remark}

\section{Global Convergence of SrMMD Gradient Flow}\label{sec:convergence}
In this section, we study convergence of the SrMMD flow $(\mu_t)_{t\geq 0}$ towards the target $\pi$. 
\begin{theorem}[Exponential decay in continuous time]\label{thm:ct_convergence}
Let $\lambda>0$. 
Let $T>0$ be a fixed time horizon. 
Suppose all conditions in \Cref{thm:ac_curve} hold.  
Additionally, suppose that $m_{\mu_t}-m_\pi\in \overline{\ran(S_{\mu_t})}$ for all $t\in[0,T]$. 
Then, for the SrMMD flow $(\mu_t)_{t\in[0,T]}$, 
\begin{talign*}
    \mmd^{2}(\mu_{T},\pi) \leq e^{-2T} \mmd^{2}(\mu_{0},\pi) + 2 e^{-2 T} \int_0^T e^{2 t} R(\lambda, t) \; \dd t,
\end{talign*}
where the residual $R(\lambda, t)$ satisfies $\lim_{\lambda\to 0+} R(\lambda, t)=0$ for all $t\in[0,T]$. 
If, in addition, there exists $0 < \alpha \leq 1$ such that $m_{\mu_t} - m_\pi\in \ran S_{\mu_t}^{\alpha/2}$ for all $t\in[0,T]$; then, $R(\lambda, t)=\calO(\lambda^\alpha)$. 
\vspace{-5pt}
\end{theorem}
\setlength{\columnsep}{18pt}
\begin{wraptable}{r}{0.56\textwidth}
\centering
\small
\setlength{\tabcolsep}{4.5pt}
\begin{tabular}{@{}lll@{}}
\toprule
Algorithm & Metric & Iteration complexity \\
\midrule 
LMC~\citep{chewi2024analysis,vempala2019rapid,pmlr-v83-cheng18a} & KL / $\chi^2$ & $\tilde{\calO}(\delta^{-1} \beta)$ \\
Averaged LMC~\citep{balasubramanian2022towards} & Fisher & $\calO(d^2\delta^{-2})$ \\
SVGD~\citep{balasubramanian2024improved} & $\mathrm{KSD}^2$ & $\calO(\delta^{-1})$ \\
R-SVGD~\citep{he2022regularized} & Fisher & $\calO(\delta^{-1})$ \\
DrMMD~\citep{chen2025regularized} & $\chi^2$ & $\tilde{\calO}(\delta^{-(\alpha+1)/\alpha} \beta)$ \\
SrMMD & $\mmd^2$ / $\mathrm{KSD}^2$ & $\tilde{\calO}(\delta^{-1/\alpha})$ \\
\bottomrule
\end{tabular}
\caption{Iteration complexity of several sampling methods. Here, $\beta$ denotes the isoperimetric constant of $\pi$, and $\alpha$ denotes the regularity parameter of the flow. 
The guarantees are stated in different metrics: $\chi^2$ divergence controls KL divergence and both typically control KSD~\citep{gorham2017measuring,kanagawa2025controlling}.}
\label{tab:iteration_complexity}
\end{wraptable}
The proof, which is given in \Cref{sec:proof_ct_convergence}, \textcolor{red}{builds on the dissipation analysis in Appendix~C.2 of~\cite{mroueh2019sobolev}. 
However, our result is more general: it holds for all $\lambda>0$, allows for possibly infinite-dimensional RKHSs, and accommodates unbounded kernels. 
Another technical distinction is that we do not impose a vanishing boundary condition; see \Cref{rem:no_boundary}.} 
The condition $m_{\mu_t}-m_\pi\in\overline{\ran(S_{\mu_t})}$ is mild: it holds whenever $\mu_t$ has full support on $\R^d$; see \Cref{lem:full_support}.
\Cref{thm:ct_convergence} shows that the SrMMD flow converges exponentially fast to the target distribution $\pi$, as measured either by the MMD (or KSD with a Stein kernel $k_\pi$), up to a residual barrier term. 
This barrier term is of explicit order $\calO(\lambda^\alpha)$ under a stronger condition 
$m_{\mu_t} - m_\pi\in \ran S_{\mu_t}^{\alpha/2}$, which quantifies the smoothness of $m_{\mu_t} - m_\pi$ relative to the operator $S_{\mu_t}$. Such source conditions are standard in nonparametric regression~\citep{cucker2007learning,fischer2020sobolev,chennested,chen2025regularized}. See \Cref{rem:source} for a detailed discussion. 
Our \Cref{thm:ct_convergence} shall be contrasted with the convergence of the original MMD flow (c.f. Theorem 6 of \cite{arbel2018gradient}) where the barrier term cannot be explicitly characterized. 

In the context of sampling from an unnormalized target $\pi$, most of the existing convergence rates are established under various tail conditions on $\pi$, such as Poincaré (sub-exponential tail) or log-Sobolev (sub-Gaussian tail)~\citep{wibisono2018sampling,vempala2019rapid,mousavi2023towards,chewi2024analysis}. 
By contrast, \Cref{thm:ct_convergence} relies only on conditions on $\pi$ that are implicitly encoded through the Stein kernel $k_\pi$ in \Cref{ass:bounded,ass:smmd_growth}. 
As a result, the SrMMD flow enjoys an explicit decay rate $\exp(-2T)$ independent of the tail conditions of $\pi$.
A likely explanation is that our convergence is stated with the weaker MMD/KSD metric compared with KL or $\chi^2$-divergence metrics as in ULA~\citep{chewi2024analysis,vempala2019rapid} or SVGD~\citep{korba2020non}. 
We postpone a detailed iteration-complexity analysis in this context to the discrete-time setting, after establishing the corresponding discrete-time convergence.


\begin{remark}[Source condition]\label{rem:source}
    To motivate the condition that $m_{\mu} - m_\pi\in \ran S_{\mu}^{\alpha/2}$, consider a periodic Mat\'ern kernel on a $d$-dimensional torus $\calX=\mathbb T^d=[0,2\pi)^d$ whose RKHS is norm equivalent to a Sobolev space $H^r(\calX)$ for $r\geq1$~\citep[Corollary 10.48]{wendland2004scattered}. 
    Then, according to \Cref{lem:sobolevequivalent},  $$\ran(S_\mu^{\alpha/2}) \cong H^{r(\alpha+1) - \alpha}(\calX) \cap \mathrm{span}\{1\}^{\perp_\calH}.$$ 
    Since $\langle m_\mu - m_\pi, 1\rangle_\calH = 0$, so the $\mathrm{span}\{1\}^{\perp_\calH}$ part is satisfied. 
    Next, we consider the $H^{r(\alpha+1) - \alpha}(\calX)$ part. 
    Since $r>1$, this implies that $m_\mu-m_\pi$ belongs to a smaller Sobolev space $H^{r(\alpha+1) - \alpha}(\calX) \subset H^r(\calX) \cong \calH$.  
    A sufficient condition for this to hold is when $\frac{\dd\pi}{\dd \mu}-1\in H^{r\alpha-r-\alpha}(\calX)$ (\Cref{lem:suff_range}), which is weaker than the condition imposed in DrMMD flow that $\frac{\dd\mu}{\dd\pi}-1 \in H^{r\alpha}(\calX)$. 
\end{remark}
\textcolor{red}{Next, we consider SrMMD descent, the forward Euler discretization of the SrMMD flow, which is referred to as \emph{Sobolev descent} in~\cite{mroueh2019sobolev}}. Specifically, given an initial distribution $\mu_0\in\calP_2(\R^d)$, a learning rate $\gamma >0$, and a fixed time horizon $S\in\N$, SrMMD descent is a sequence of probability measures $(\mu_s)_{s\in [S]}$ defined by: 
\begin{talign}\label{eq:srmmd_euler}
    \mu_{s+1} = ( \Id - \gamma \nabla f_{\mu_s,\pi})_{\#} \mu_s .
\end{talign}

Similarly as in continuous time, the SrMMD descent scheme is well-posed under \Cref{ass:bounded,ass:smmd_growth}. 
We now turn to the convergence analysis of the SrMMD descent. 

\begin{assumption}[Growth II]\label{ass:growth_2}
There exists a constant $M>0$ such that for any $x, y\in\R^d$, $\|\bH_1 k(x, y)\|_{\mathrm{op}} \leq M (1 + \|y\|^2)$, \ag{where $[\bH_1 k(x, y)]_{i,j} := \partial_{1,i}\partial_{1,j} k(x, y)$ denotes the Hessian of $k$ with respect to its first argument and $\|\cdot\|_{\mathrm{op}}$ is the matrix operator norm.  }
\end{assumption}
The additional condition controls the growth of the Hessian of the kernel $k$. For a Langevin Stein kernel $k_\pi$, this condition is satisfied when the score function $\mathfrak{s}$ is Lipschitz and has a bounded Hessian (See \Cref{prop:suff_stein_kernel}). Such additional conditions on $\mathfrak{s}$ are commonly used for \emph{discrete-time} convergence analysis of sampling, e.g., \cite[Assumption 2]{balasubramanian2024improved}, \cite[Assumption A2]{korba2020non}, \cite{chewi2024analysis}, etc.

\begin{theorem}[Exponential decay in discrete time]\label{thm:dt_convergence}
Suppose \Cref{ass:smmd_growth,ass:bounded,ass:growth_2} hold. 
Let $S\in\N$ be a fixed time horizon. 
Let $\gamma\in(0,\frac{1}{2})$ be a fixed step size. 
Let $(\mu_s)_{s\in[S]}$ be the SrMMD descent defined in Eq.~\eqref{eq:srmmd_euler}. 
Suppose that $(\mu_s)_{s\in[S]}$ has uniformly bounded second moment $C>0$. 
Suppose there exists $0 < \alpha \leq 1$ such that $m_{\mu_s} - m_\pi = S_{\mu_s}^{\alpha/2} q_s$ with $q_s\in\calH$ for all $s\in[S]$. 
Then, 
\begin{talign*}
    \mmd^2(\mu_S, \pi) \leq \exp(-2 S\gamma) \mmd^2(\mu_0, \pi) + 4 (d C_0)^\alpha \gamma^{1+\alpha} \sum_{s=0}^{S-1} (1-2 \gamma)^{S-1-s} \|q_s\|_{\cal H}^2.
\end{talign*}
Here, the constant $C_0$ only depends on $C$ and $M$. 
If additionally there exists a constant $Q$ independent of $S$ such that $\|q_s\|_\calH^2\le Q$ for all $s\in[S]$, then with $\lambda = dC_0\gamma$, we have
\begin{talign*}
    \mmd^2(\mu_S, \pi) \le\exp(-2 S \gamma)\mmd^2(\mu_0, \pi) + 4 (dC_0)^\alpha \gamma^{\alpha}Q  . 
\end{talign*}
\end{theorem}
\vspace{-5pt}
The proof is given in \Cref{sec:proof_dt_convergence}. 
We are now ready to analyze the iteration complexity of SrMMD descent. To reach final error $\mmd^2(\mu_S, \pi) \leq \delta$ (or equivalently $\ksd^2(\mu_S, \pi) \leq \delta$), it suffices to pick a step size $\gamma = \calO(\delta^{1/\alpha})$ and iteration number $S = \calO(\delta^{-1/\alpha} \log(\delta^{-1}))$. 
Among existing results, the one closest in spirit to ours is Stein variational gradient descent (SVGD)~\citep{liu2016stein}, which has recently been shown to admit an iteration complexity of $\calO(\delta^{-1})$~\citep{balasubramanian2024improved} also under the KSD metric. 
The iteration complexity of SrMMD is constrained by $0 < \alpha \leq 1$ and does not improve for further smoothness $\alpha > 1$. 
This limitation arises from essentially using Tikhonov regularization $(S_\mu+\lambda\Id)^{-1}$ in the formulation of SrMMD---a well-known saturation phenomenon~\citep{neubauer1997converse,li2024towards,meunier2024optimal}. 
Such a saturation can be avoided by using other spectral regularization strategies, 
such as Showalter, Landweber, or cutoff regularization~\citep{engl1996regularization,BAUER200752,hagrass2022spectral,hagrass2023spectralgof}. 
Moreover, the KSD guarantee of SVGD in \cite{balasubramanian2024improved} is \emph{ergodic}, in the sense that it controls the averaged quantity
$\frac{1}{S}\sum_{s=1}^S \ksd^2(\mu_s,\pi)$ in contrast to our last-iterate guarantee.

Beyond this, most existing iteration-complexity results in the sampling literature rely on structural constants  $\beta$ of the target $\pi$, such as Poincaré or log-Sobolev constants~\citep{vempala2019rapid,chewi2020svgd,chewi2024analysis}. These constants, however, are exponential in the separation between modes when the distribution $\pi$ is multi-modal~\citep{menz2014poincare,chewi2024logconcavesampling}, making convergence too slow to occur in a reasonable number of iterations.  
In contrast, our convergence result does not rely on such constants, and therefore does not suffer from the multi-modality of $\pi$. 
The price we pay here is the regularity condition $m_\mu-m_\pi\in\ran(S_\mu^{\alpha/2})$ which we have discussed in \Cref{rem:source}.  
We admit that this condition is hard to verify in practice. 
See \Cref{tab:iteration_complexity} for a comparison on the iteration complexities of all methods. 
The original KSD flow is shown to converge under a Stein-Poincaré inequality~\citep{korba2020non}---a condition that is proved never to hold~\citep{duncan2019geometry}. 
This highlights another advantage of Sobolev regularization. 
Recent theoretical advances have also established iteration complexity bounds for diffusion-based samplers that are independent of $\beta$; see, e.g.,~\citep{lee2023convergence,chen2023improved,li2024towards}.

\begin{figure}[t]
    \centering

    \includegraphics[width=0.99\linewidth, trim=40 0 0 0,
            clip]{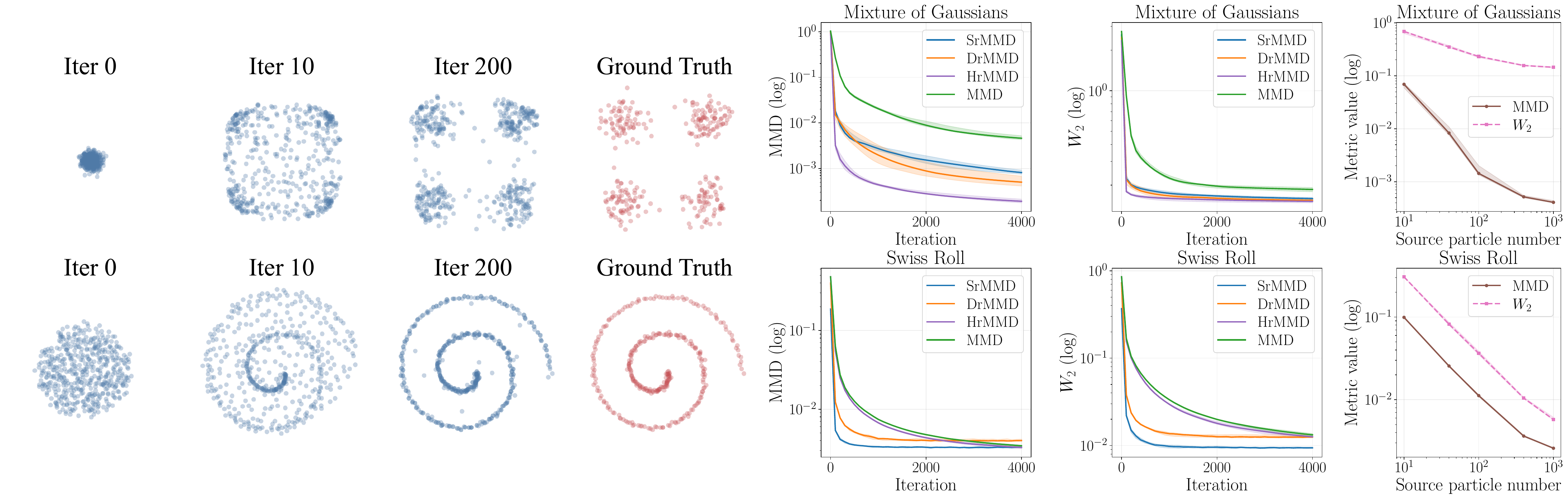}
    \caption{\textbf{Top row:} Mixture of Gaussians. \textbf{Bottom row:} Swiss roll. From left to right: particle evolution during training; MMD and \(W_2\) versus iteration; and final MMD and \(W_2\) versus the number of source particles \(N\). Results are aggregated over 10 random seeds; solid curves show the median, and shaded regions show the 25\%--75\% percentile. 
    }
    \label{fig:all_experiments}
    \vspace{-10pt}
\end{figure}
 \vspace{-3pt}
\section{Particle Implementation of SrMMD Descent}\label{sec:particle_flow}
\vspace{-3pt}
In practice, the SrMMD descent scheme in Eq.~\eqref{eq:srmmd_euler} is implemented by a system of interacting particles~\citep[Section 5]{mroueh2019sobolev}. 
Let $\{x_0^{(i)}\}_{i=1}^N$ be $N$ i.i.d samples from the initial distribution $\mu_0$ on $\R^d$. 
For a learning rate $\gamma >0$, and a fixed time horizon $S\in\N$, 
the SrMMD particle descent is defined as: 
\begin{align}\label{eq:particledescent}
    x_{s+1}^{(i)} = x_s^{(i)}-\gamma\nabla f_{\hat \mu_s, \pi}(x_s^{(i)}), \quad \forall s \in[S] . 
\end{align}
where $\hat\mu_s = \frac1N\sum_{i=1}^N \delta_{x_s^{(i)}}$ 
denotes the empirical distribution of the particles at time step $s$. 
\Cref{prop:closeformofsmmd} below shows that the witness function $f_{\hat \mu_n, \pi}$ admits a tractable closed-form expression. 
\begin{proposition}\label{prop:closeformofsmmd}
Let $\hat\mu =\frac1N \sum_{i=1}^N \delta_{x^{(i)}}$ be an empirical distribution. 
Define $D_{XX}\in\R^{(Nd)\times N}$ with $[D_{XX}]_{(i,l), j}=\partial_{1,l} k(x^{(i)}, x^{(j)})$, and $H_{XX}\in\R^{(Nd) \times (Nd)}$ with $[H_{XX}]_{(i,l), (j,m)} = \partial_{1,l} \partial_{2,m} k(x^{(i)}, x^{(j)})$.\footnote{Strictly speaking, $D_{XX}$ and $H_{XX}$ here are higher-order tensors. To simplify notation, we implicitly identify them with their vectorized matrix representations, so that all subsequent computations can be written using standard matrix operations.} 
Also, denote $\mathbb K_X: \R^d \to \R^N$ and $\mathbb D_X: \R^d \to \R^{Nd}$: 
\begin{talign*}
    \mathbb K_X(z) = [ k(x^{(1)},z), \ldots, k(x^{(N)},z)]^\top, \quad 
    \mathbb D_X(z) = [ \partial_{1, 1} k(x^{(1)},z), \ldots, \partial_{1, d} k(x^{(N)},z)]^\top . 
\end{talign*}
Let $\1_N\in\R^N$ be a vector of all $1$s. 
Then, the witness function $f_{\hat\mu, \pi}\in \calH$ is given by 
\begin{talign*}
    f_{\hat \mu, \pi}(z) &= \frac{\1_N^\top \mathbb K_X(z)}{N \lambda}  -\frac{\E_{\pi}[k(Y, z)]}{\lambda} - \mathbb D_X(z)^\top (H_{XX} + N\lambda\Id)^{-1} \pth{\frac{D_{XX}\1_N}{N \lambda} - \frac{\E_{\pi}[\mathbb D_X(Y)]}{\lambda} }.
\end{talign*}
\end{proposition}
Proof is provided in \Cref{sec:proof_closed_form}. 
In the context of generative modeling when $\pi$ is known up to $M$ i.i.d samples $\{y^{(j)}\}_{j=1}^M$, the above expectation with respect to $\pi$ is replaced by an empirical average of these $M$ samples. 
In the context of sampling when $\pi$ is known up to its normalization constant, the above expectation with respect to $\pi$ vanishes under a Stein kernel $k_\pi$. 
The additional gradient $\nabla f_{\hat{\mu},\pi}$ required in Eq.~\eqref{eq:particledescent} can be obtained with automatic differentiation libraries such as JAX~\citep{jax2018github}.

Next, we analyze finite-particle convergence of SrMMD in Eq.~\eqref{eq:particledescent}, under the following assumption. 

\begin{assumption}\label{ass:particle-conditions}
We impose the following conditions.
\begin{enumerate}[topsep=0pt, partopsep=0pt, leftmargin=*]
    \item\label{assm:global-boundedness}
    \emph{Global boundedness.}
    There exists a constant $B>0$ such that, for all $i,j\in[d]$,
        $\sup_{x\in\R^d}
        \|\partial_{1,i}\partial_{1,j}k(x,\cdot)\|_{\calH}
        \le B$,
    where the derivatives are taken in $\calH$.
    \item\label{assm:bounded-moments}
    \emph{Bounded fourth moments.}
    \ $\{\mu_s\}_{s\ge0}$ defined in
    Eq.~\eqref{eq:srmmd_euler} satisfies
    $\sup_{s\ge0}\int\|x\|^4\,\dd\mu_s(x)<\infty$.
\end{enumerate}

\end{assumption}

\begin{theorem}[Convergence of finite-particle implementation of SrMMD]\label{thm:particle-convergence}
Let $\mu_0,\pi\in \calP_2(\R^d)$, let $\{x_0^{(i)}\}_{i=1}^N$ be  i.i.d. with $\mu_0$, let $S\in \mathbb N$ be a fixed time horizon. Suppose there exists $0<\alpha\le 1$ such that $m_{\mu_s}-m_\pi=S_{\mu_s}^{\alpha/2}q_s$ with $q_s\in \calH$ for any $s\in [S]$, and there exists a constant $Q<\infty$ such that $\sup_{1\le s\le S}\|q_s\|_\calH^2\le Q$. Fix $\lambda>0$ and $0<\gamma \le \frac12$ such that $4dB\gamma (\mmd(\mu_0,\pi)+1)\le \lambda$. Then there exists a constant $C$ depending only on $d,B,\gamma,\lambda, \mmd(\mu_0,\pi), \alpha,Q$, the kernel $k$, and the uniform fourth-moment bound in~\Cref{ass:particle-conditions}, such that for each $\delta\in (0,1)$ and $N\ge C/\delta$, with probability at least $1-\delta$, 
$$\mmd(\hat \mu_s,\pi)\le \mmd(\mu_s,\pi)+C[\log(N\delta)]^{-\alpha/(2-\alpha)},\quad \forall\, 0\le s\le S.$$
\end{theorem}
The proof is given in~\Cref{apd:finite-particle}. To obtain the high-probability bound on $\mmd(\hat\mu_s,\pi)$ holds simultaneously for all $0\le s\le S$, we truncate particle and population updates up to a chosen iteration, then use the monotonicity of $\mmd(\hat\mu_s,\pi)$ and the decay of $\mmd(\mu_s,\pi)$ to control later iterations. 
Combining \Cref{thm:particle-convergence} and \Cref{thm:dt_convergence} provides an upper bound on $\mmd(\hat \mu_s,\pi)$ in terms of iteration $S$ and particle number $N$. 
Compared with closely related prior work, the finite-particle bounds for both MMD~\cite{arbel2019maximum} and DrMMD~\cite{chen2025regularized} flows are obtained via propagation-of-chaos arguments in terms of the expected $W_2$ distance. These bounds achieve an $\mathcal{O}(1/\sqrt{N})$ rate with respect to the particle number $N$, but grow exponentially with the number of iterations. \cite{balasubramanian2024improved} studied the finite-particle convergence of SVGD and obtained an $\mathcal{O}(1/\sqrt{N})$ convergence rate, improving upon the previous $\mathcal{O}(1/\sqrt{\log\log N})$ rate established in~\cite{shi2024finite}. More recently, \cite{balasubramanian2026finite} studied the finite-particle convergence of the drifting model proposed in~\cite{deng2026generative}, obtaining a conditional root residual-velocity rate of $\mathcal{O}(N^{-1/(d+4)})$ under uniform regularity assumptions. In contrast, our result provides a uniform-in-iteration bound, albeit at the cost of a logarithmic convergence rate with respect to $N$.

\begin{figure}
    \centering
    \includegraphics[width=0.85\linewidth]{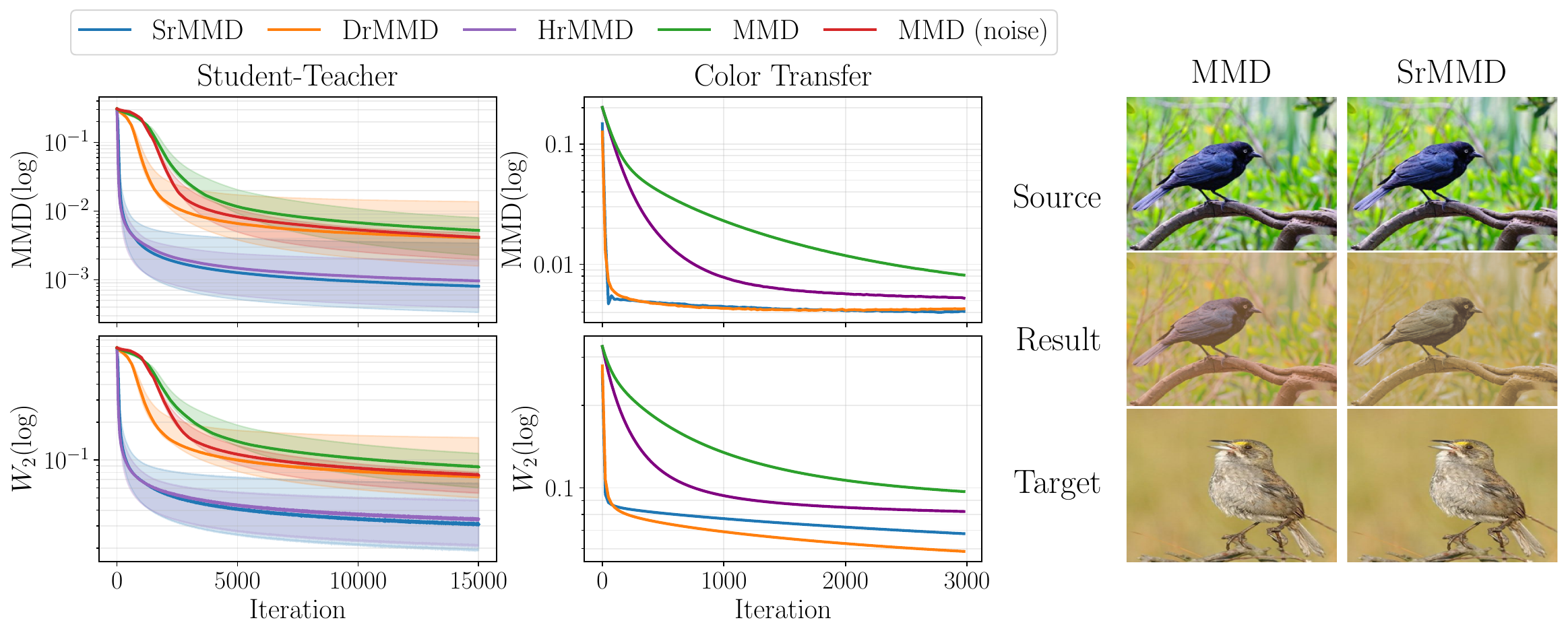}
    \caption{
    \textbf{Left:} Student-teacher. \textbf{Middle \& Right:} Color transfer. 
    Results are aggregated over 10 random seeds; solid curves show the median, and shaded regions show the 25\%--75\% percentile.}
    \label{fig:st_and_color}
    \vspace{-10pt}
\end{figure}

\section{Numerical Experiments}\label{sec:experiments}
\vspace{-3pt}
In this section, we conduct extensive experiments to evaluate the empirical performance of SrMMD in the context of both generative modelling ($\pi$ is known from i.i.d. samples) and sampling ($\pi$ is known up to a normalization constant). The code is available at \url{https://github.com/FerdTian/SrMMD}.

\textbf{Toy examples.}
\hspace{4mm}
We consider two synthetic settings in which 1) the target distribution $\pi$ is a mixture of two-dimensional Gaussian distributions, and 2) $\pi$ has the shape of a Swiss roll. 
The Gaussian mixture provides a standard multimodal benchmark for assessing whether a sampling method can recover all modes of the target, whereas the Swiss roll is a simple example of a distribution supported on a manifold. 
We compare the SrMMD flow with the vanilla MMD flow~\citep{arbel2019maximum}, the DrMMD flow~\citep{chen2025regularized}, and HrMMD flow \cite{mroueh2020unbalanced}---\cite{mroueh2020unbalanced} refers to HrMMD objective as kernel Sobolev-Fisher discrepancy---, a hybrid of DrMMD and SrMMD that penalizes both the $L_2$ norm and the gradient norm of the witness function; see \Cref{sec:hrmmd} for details. 
For the Gaussian mixture experiments, we use the Gaussian kernel $k(x,y)=\exp(-(2\sigma^2)^{-1} \|x-y\|^2)$ with length scale $\sigma>0$, so that the kernel and the target $\pi$ satisfy all  \Cref{ass:bounded,ass:smmd_growth,ass:growth_2}. 
For the Swiss roll, we use the Riesz kernel $k(x,x')=-\|x-x'\|$; this choice is intentional and is meant to test whether the SrMMD flow continues to perform well even when the conditions on kernels are not perfectly met. 

For all flows considered here, we use regularization parameter $\lambda = 0.1$ and step size $\gamma = 0.1$, except for HrMMD where $\gamma = 1.0$.
An ablation study on different values of $\lambda$ and $\sigma$ is given in 
\Cref{sec:more_toy}. All flows are simulated for $4{,}000$ iterations when convergence is empirically observed. 
From \Cref{fig:all_experiments}, SrMMD flow substantially outperforms the vanilla MMD flow. 
This is consistent with the theory: MMD flow has only polynomial convergence under restrictive conditions~\citep{arbel2019maximum}, whereas SrMMD flow enjoys exponential convergence; see \Cref{thm:ct_convergence,thm:dt_convergence}. 
The performance of SrMMD flow is comparable to that of DrMMD flow, suggesting that either $L_2$ or gradient regularization improves convergence. 
Finally, the hybrid HrMMD flow achieves the best performance on Gaussian mixtures, which is expected because it finds a best combination of both the regularization schemes.



\vspace{1mm}
\textbf{Student-Teacher network.}\hspace{4mm}
Next, we consider mean-field student-teacher networks, a more challenging and high-dimensional benchmark for MMD flows~\citep{arbel2019maximum,chen2025regularized}. 
The teacher network is $\Psi_\pi^\mathsf{T}(z)=\int\psi(z,\theta)\,\dd\pi(\theta)$,
where $\pi$ is the target teacher distribution, and the student network is $\Psi_\mu^\mathsf{S}(z) = \int\psi(z,\theta)\,\dd \mu(\theta)$,
where $\mu$ is the student distribution. 
Training the student network amounts to the optimization problem: 
$\min_{\mu} \E_{z\sim\mathbb{P}_{\mathrm{data}}}[\|\Psi_\pi^\mathsf{T}(z)-\Psi_\mu^\mathsf{S}(z)\|^2]$,
where the data distribution $\mathbb{P}_{\mathrm{data}}$ is a uniform spherical distribution on $\mathbb{R}^{50}$. 
This is equivalent to MMD flow under a new kernel, given by
$k(\theta, \theta') = \E_{z\sim\mathbb{P}_{\mathrm{data}}}[\psi(z,\theta)^\top\psi(z,\theta')]$, which is bounded and hence satisfies all \Cref{ass:bounded,ass:smmd_growth,ass:growth_2}. 
See \Cref{sec:more_student} for details.  

We compare SrMMD flow with the vanilla MMD flow (both with and without noise injection), DrMMD flow, and the hybrid version, HrMMD flow. 
All of them are simulated for $15,000$ iterations on 1000 training samples and then evaluated on a held-out 1000 validation samples.  
We set the regularization parameter to $\lambda = 0.1$ and step size $\gamma=0.1$ for all the flows. 
An ablation study on different values of $\lambda$ is given in 
\Cref{sec:more_student}. 
In the left panel of \Cref{fig:st_and_color}, both SrMMD and HrMMD flows achieve faster convergence and better final performance.






\begin{figure}[t]
    \centering
    \begin{subfigure}[t]{0.24\linewidth}
        \centering
        \includegraphics[
            width=\linewidth,
            height=3.5cm,
            keepaspectratio,
            trim=0 0 0 0,
            clip
        ]{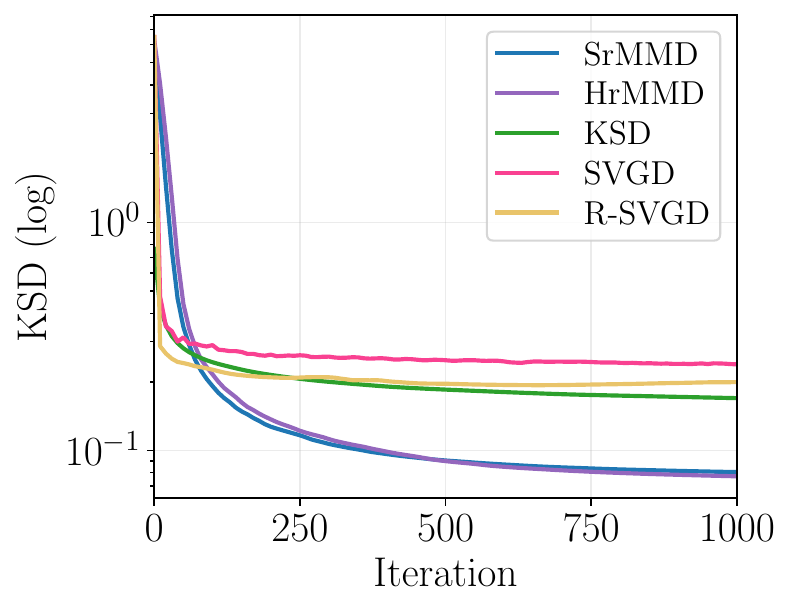}
    \end{subfigure}
    \begin{subfigure}[t]{0.75\linewidth}
        \centering
        \includegraphics[
            width=\linewidth,
            height=3.5cm,
            keepaspectratio,
            trim=0 0 0 0,
            clip
        ]{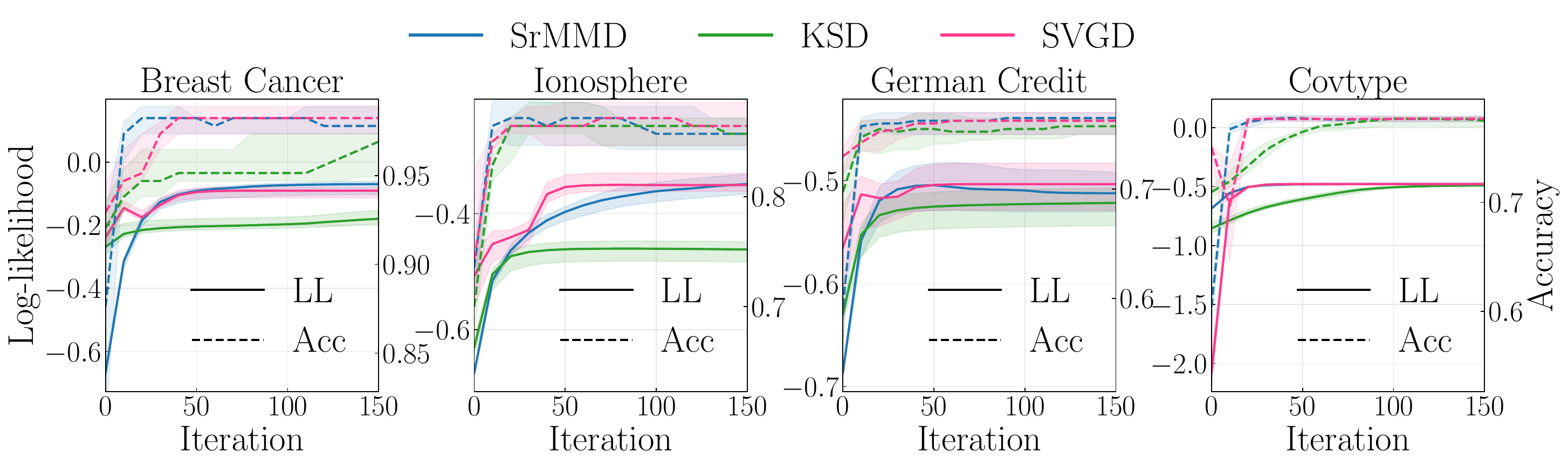}
    \end{subfigure}
    \caption{
    \textbf{Left:} Sampling from Gaussian mixtures. \textbf{Rest:} Bayesian logistic regression. Solid curves show the median, and shaded regions show the 25\%--75\% percentile. 
    }
    \label{fig:sampling}
    \vspace{-10pt}
\end{figure}

\vspace{1mm}
\textbf{Image color transfer.}\hspace{4mm}
We consider an image color transfer task following \cite{mroueh2019sobolev}, where the goal is to modify a source image $A$ so that its color distribution matches that of a target image $B$.
The task can be formulated as the alignment of the empirical color distributions of the source and target image in the RGB space $[0,1]^3$. 
To this end, we sample $N$ color particles from the source image and also $N$ color particles from the target image. 
We then evolve the source particles toward the target particles using our SrMMD flow, and compare it against the vanilla MMD flow, DrMMD flow, and HrMMD flow. After obtaining the transported particles $Y_T$, we reconstruct the recolored image by nearest-neighbor assignment; see \Cref{sec:more_color} for additional details. 
As shown in the middle panel of \Cref{fig:st_and_color}, SrMMD, DrMMD, and HrMMD all exhibit faster decay in both MMD and $W_2$ distance than the vanilla MMD flow, demonstrating the positive effect of regularization on convergence. This improvement is also visible in the final recolored images; see result figures in the right panel of \Cref{fig:st_and_color}, where SrMMD more faithfully transfers the target color to the source image. 

\vspace{1mm}
\textbf{Sampling.}\hspace{4mm}
While the previous experiments concern the ability of SrMMD flow as a generative model, here we evaluate the performance of SrMMD as a sampling method, under a Stein kernel $k_\pi$. 
The considered baselines are KSD flow~\citep{korba2021kernel} and Stein variational gradient descent (SVGD)~\citep{liu2016stein}. DrMMD flow is not applicable here since its regularization requires access to the unknown $\|f\|_{L_2(\pi)}$.  

First, we consider sampling from a two-dimensional mixture of 10 Gaussian distributions, a challenging multi-modal setting. 
The Stein kernel here is constructed from a base Gaussian kernel $k(x,x')=\exp(-(2\sigma^2)^{-1} \|x-x'\|^2)$ with bandwidth $\sigma=0.3$. 
In the left panel of \Cref{fig:sampling}, we perform SrMMD flow with $N=500$ source particles. 
More experiments with other values of $N$ can be found in \Cref{sec:more_sampling}. 
Beyond KSD flow and SVGD, we also compare against regularized SVGD (R-SVGD)~\citep{he2022regularized}, a spectral regularized version of SVGD.  
For SrMMD, HrMMD, and KSD flows, we use the best tuned step sizes $\gamma=0.01$ for KSD and $\gamma=0.1$ for SrMMD and HrMMD. 
For SVGD and R-SVGD, we use the best-tuned step size $\gamma=0.01$. 
The regularization parameters are $\lambda=0.5$ for SrMMD and HrMMD, and $\lambda=0.01$ for R-SVGD, each selected from a pre-specified grid according to final performance. 
The left panel of \Cref{fig:sampling} shows that SrMMD attains the fastest KSD decay and the lowest final KSD value. 

Next, we consider the task of Bayesian logistic regression, following \cite{liu2016stein,korba2021kernel}. Given a dataset $\calD$ with a Bernoulli likelihood, the target distribution $\pi$ is the posterior
$
\pi(x)\propto p(\calD \mid x)\,p(x),
$
which is available up to an intractable normalizing constant. 
This benchmark is considerably more challenging than the Gaussian mixture example above, since the resulting posterior distributions are higher-dimensional, with dimensions corresponding to \textsc{Breast Cancer} ($d=30$), \textsc{Ionosphere} ($d=34$), \textsc{German Credit} ($d=24$), and \textsc{Covtype} ($d=54$). 
We simulate SrMMD flow, KSD flow, and SVGD, each with $N=20$ particles, and report the evolution of test accuracy and test log-likelihood.
From \Cref{fig:sampling}, we see that SrMMD flow significantly improves upon the original KSD flow, while matching the performance of SVGD across all tasks. 
Since KSD flow is known to underperform relative to SVGD~\citep{korba2021kernel}, the Sobolev regularization provides a promising alternative.  


 \vspace{-5pt}
\section{Discussion}
 \vspace{-2pt}
We \textcolor{red}{analyze} SrMMD flow (a.k.a.~Sobolev Descent), a flow with a gradient penalty to the MMD/KSD witness function for globally convergent continuous- and discrete-time schemes, applicable for both generative modelling and sampling. 
A primary limitation of SrMMD flow is its computational cost: each iteration requires inverting an $Nd\times Nd$ matrix, leading to a cost of $\calO(N^3 d^3)$. 
Nystr\"om approximations~\citep{williams2000using}, random features~\citep{rahimi2007random}, and caching of vector field~\citep{galashov2024deep} are standard approaches for reducing such computational burdens. 
Since the present paper is primarily theoretical, our empirical evaluation of SrMMD as a sampling method is restricted to small-scale benchmarks, where it achieves performance comparable to SVGD. 
A natural next step is to assess the performance of SrMMD in larger-scale tasks where SVGD has proved effective, e.g., earthquake modelling~\citep{corrales2025annealed}.

\newpage
\bibliographystyle{abbrvnat}  
\bibliography{main}  

\newpage

\begin{appendices}

\crefalias{section}{appendix}
\crefalias{subsection}{appendix}
\crefalias{subsubsection}{appendix}

\makeatletter
\@addtoreset{equation}{section} 
\makeatother
\setcounter{equation}{0}
\renewcommand{\theequation}{\thesection.\arabic{equation}}
\renewcommand{\theHequation}{\thesection.\arabic{equation}} 

\newcommand{\appsection}[1]{
  \refstepcounter{section}
  \section*{Appendix \thesection: #1}
  \addcontentsline{toc}{section}{Appendix \thesection: #1}
}

\onecolumn

{\hrule height 1mm}
\vspace*{-0pt}
\section*{\LARGE\bf \centering Supplementary Material
}
\vspace{8pt}
{\hrule height 0.1mm}
\vspace{24pt}



\section{Proof of Main Results}
\paragraph{Additional notations}
Denote a $d$-fold tensor product Hilbert space $\calH^d = \calH \otimes \ldots \otimes \calH$ and we say $\mbf \in \mathcal{H}^d$ if $\mbf = \left[f_1, \cdots, f_d\right]^{\top}$ such that $f_i \in \mathcal{H}$ for all $i \in[d]$. Denote $(L_2^d(\mu),\|\cdot\|_{L_2^d(\mu)})$ as the space of all vector valued $\mu$-square integrable measurable function with $\|f\|_{L_2^d(\mu)}^2:=\int_{\mathbb{R}^d}\|f(x)\|^2 \dd\mu(x)$. 
Define $\mathcal{L}^d$ as the Lebesgue measure on $\R^d$. $\mu \ll \calL^d$ means that a measure $\mu$ is absolutely continuous with respect to the Lebesgue measure. 

\subsection{Proof of \Cref{prop:expressionofsmmd}}\label{apd:proof_of_expressionofsmmd}
By definition of $S_\mu$, there is $\langle f, S_\mu f\rangle_\calH = \|\nabla f\|_{L_2^d(\mu)}^2$. 
By a direct calculation,
\begin{align*}
    \smmd(\mu,\pi) &= \sup_{f\in\mathcal{H}: \; \langle f, ({S}_{\mu} + \lambda\Id) f \rangle_\calH \leq 1}\langle f, m_\mu - m_\pi \rangle_\calH \\
    &= \sup_{f\in\mathcal{H}: \; \langle f, ({S}_{\mu} + \lambda\Id) f \rangle_\calH \leq 1} \left\langle ({S}_{\mu} + \lambda\Id)^{-1/2} ({S}_{\mu} + \lambda\Id)^{1/2} f, m_\mu - m_\pi \right\rangle_\calH \\
    &= \sup_{f\in\mathcal{H}: \; \langle f, ({S}_{\mu} + \lambda\Id) f \rangle_\calH \leq 1} \left\langle  ({S}_{\mu} + \lambda\Id)^{1/2} f, ({S}_{\mu} + \lambda\Id)^{-1/2} (m_\mu - m_\pi) \right\rangle_\calH \\
    &\leq \sup_{f\in\mathcal{H}: \; \langle f, ({S}_{\mu} + \lambda\Id) f \rangle_\calH \leq 1} \|({S}_{\mu} + \lambda\Id)^{1/2} f \|_\calH \cdot \|({S}_{\mu} + \lambda\Id)^{-1/2} (m_\mu - m_\pi) \|_\calH \\
    &= \|({S}_{\mu} + \lambda\Id)^{-1/2} (m_\mu - m_\pi) \|_\calH^2 .
\end{align*}
The maximum is admitted at 
$$
f=\bar{f}_{\mu,\pi}=\left( \|({S}_{\mu} + \lambda\Id)^{-1/2} (m_\mu - m_\pi) \|_\calH \right)^{-1} ({S}_{\mu} + \lambda\Id)^{-1}(m_\mu-m_\pi) . 
$$
For reasons of presentation and computational convenience, we instead work in the paper with the unnormalized witness
$$
f_{\mu,\pi}=(S_{\mu}+\lambda\Id)^{-1}(m_\mu-m_\pi).
$$
This is without loss for our purposes, since SrMMD itself is not used elsewhere in the paper. 

\subsection{Proof of \Cref{thm:ac_curve}}\label{sec:proof_ac_curve}
\begin{proof}
The proof of the first claim is a direct application of \cite[Theorem 8.3.1]{ambrosio2005gradient}, where we only need to show that $\int_0^T \| \nabla f_{\mu_t,\pi} \|_{L_2^d(\mu_t)} \dd t < \infty $. 
From Eq.~\eqref{eq:d_dt_mmd}, we have that 
$\frac{\dd}{\dd t}\mmd^2(\mu_t,\pi) < 0$. This indicates that $\mmd^2(\mu_t, \pi)$ is monotonically decreasing in time along the trajectory. 
Hence, since $\mmd(\mu_0,\pi) < C_{\mmd}$, there is $\sup_{t\in[0,\infty)} \mmd(\mu_t,\pi) < C_{\mmd} $. 

By definition of $f_{\mu_t,\pi}$ at any $t\in[0,T]$, we have 
\begin{align*}
    \| \nabla f_{\mu_t,\pi} \|_{L_2^d(\mu_t)}^2 &= \left\langle S_{\mu_t}(S_{\mu_t}  +\lambda\Id)^{-1} (m_{\mu_t} -m_\pi), (S_{\mu_t}  +\lambda\Id)^{-1} (m_{\mu_t} -m_\pi) \right\rangle_{\calH} \\
    &\leq \lambda^{-1} \left\| m_{\mu_t} -m_\pi \right\|_\calH^2 = \lambda^{-1} \mmd(\mu_t,\pi)^2 \leq \frac{C_{\mmd}^2}{\lambda}.  
\end{align*}
Therefore, we have proved that $\int_0^T \| \nabla f_{\mu_t,\pi} \|_{L_2^d(\mu_t)} \dd t \leq \frac{TC_{\mmd}}{\sqrt{\lambda}}$. By the converse part of \cite[Theorem 8.3.1]{ambrosio2005gradient}, $(\mu_t)_{t\in[0,T]}$ is an absolutely continuous curve with respect to $W_2$ in $\calP_2(\R^d)$. 

Next, to show uniqueness, we resort to \cite{bonnet2021differential}. To this end, we need to verify that hypothesis D1 in \cite{bonnet2021differential} is satisfied. 
Let $\calK=B(0,R)$ be a ball of radius $R$ in $\R^d$. 
Take any sequence $x_n \rightarrow x$ in $\calK$. Then, for any $t\in[0, T]$, 
\begin{align*}
    \left\| \nabla f_{\mu_t,\pi}(x_n) - \nabla f_{\mu_t,\pi}(x) \right\| \leq \|f_{\mu_t,\pi}\|_\calH \cdot \left\| \nabla_1 k(x, \cdot) - \nabla_1 k(x_n, \cdot) \right\|_{\calH^d} \to 0. 
\end{align*}
The last step holds from $\|f_{\mu_t,\pi}\|_\calH \leq\frac{\| m_{\mu_t} -m_\pi\|_\calH}{\lambda} < \infty$, and from \Cref{lem:local_lip}. Hence, D1(i) is satisfied. 
D1(iii) and D1(iv) are proved in \Cref{lem:bonnet}. 
Finally, D1(ii) is satisfied by the following derivations:
\begin{align}
    \| \nabla f_{\mu_t,\pi}(x) \| &= \left\| \nabla (S_{\mu_t}  +\lambda\Id)^{-1} (m_{\mu_t} -m_\pi) (x) \right\| \nonumber \\
    &\leq \sum_{j=1}^d \left| \left\langle \partial_{1, j}k(x,\cdot) , \; (S_{\mu_t}+\lambda\Id)^{-1} (m_{\mu_t} -m_\pi) \right\rangle_\calH \right| \nonumber \\
    &\leq \sum_{j=1}^d \| \partial_{1, j}k(x,\cdot)\|_\calH \cdot \frac{\| m_{\mu_t} -m_\pi\|_\calH}{\lambda} \leq d \sqrt{M} (\|x\| + 1)  \cdot \frac{C_{\mmd}}{\lambda} \label{eq:bound_f_mu_pi}. 
\end{align}
The last step holds by the upper bound on $|\partial_{1, j}\partial_{2, j} k(x,x)|$ in \Cref{ass:smmd_growth}, as well as the upper bound on $\mmd(\mu_t,\pi)$.  

Hence, by Theorem 4 of \cite{bonnet2021differential}, there exists a trajectory solving the continuity inclusion on any finite time interval $[0,T]$. 
In the SrMMD flow setting, once existence is established, uniqueness follows separately from the Cauchy--Lipschitz theory for continuity equations, for instance from Theorem 2 of \cite{bonnet2021differential}, since the velocity field satisfies the corresponding growth and local Lipschitz conditions as proved above. 
As the argument applies to every finite interval $[0,T]$, the trajectory cannot explode in finite time.

Finally, we are going to show that $\mu_t$ is absolutely continuous with respect to Lebesgue for all $t$. 
For any compact domain $B\subset \R^d$, since the vector field $\nabla f_{\mu_t,\pi}$ is locally Lipschitz and locally bounded as proved above, by \cite[Lemma 8.1.4]{ambrosio2005gradient}, the following ordinary differential equation admits a unique maximal solution in $[0,T]$. 
\begin{align*}
    \mathbf{T}(0, x) = x, \quad \frac{\dd}{\dd t} \mathbf{T}(t, x) = \nabla f_{\mu_t,\pi} \left(\mathbf{T}(t, x)\right), 
\end{align*}
Moreover, the SrMMD flow iterate at time $t$ has $\mu_t=\left(\mathbf{T}(t, \cdot) \right)_{\#} \mu_0$. Under the local Lipschitz and linear-growth assumptions, the mapping $\mathbf{T}(t, \cdot)$ is locally bi-Lipschitz on bounded sets, as proved in \Cref{lem:local_bilip_flow}; hence $\mathbf{T}(t, \cdot)^{-1}$ maps Lebesgue-null sets to Lebesgue-null sets. Therefore, if $\mu_0 $ is absolutely continuous with respect to Lebesgue measure $ \mathcal{L}^d$ on $\R^d$, then $\mu_t$ is also absolutely continuous with respect to Lebesgue measure $ \mathcal{L}^d$ on $\R^d$. 

\end{proof}

\subsection{Proof of \Cref{thm:ct_convergence}}\label{sec:proof_ct_convergence}
\begin{proof}
First, we consider the time-derivative, $\frac12\frac{\dd}{\dd t}\mmd^2(\mu_t,\pi)$.  
Let $(X_t)_{t\ge 0}$ be the stochastic process defined in Eq.~\eqref{eq:defofsmmdgf} which satisfies $\dd X_t =-\nabla f_{\mu_t,\pi}(X_t) \dd t$. 
Then, consider any fixed test function $h\in \calH$, differentiating along the characteristic and using the chain rule yields
\begin{align*}
    \left\langle \frac{\dd}{\dd t} m_{\mu_t}, h \right\rangle_{\calH}
    = \frac{\dd}{\dd t}\E[h(X_t)]
    = \E\left[\nabla h(X_t)^\top \frac{\dd X_t}{\dd t}\right]
    = - \E\left[\nabla h(X_t)^\top \nabla f_{\mu_t,\pi}(X_t)\right].
\end{align*}
The first equality holds from the continuity of the inner product in $\calH$.  
The second equality holds by the following additional conditions on $h$ (\Cref{lem:differentiation_along_characteristics}): 
\begin{align}\label{eq:uniform_intergrable}
    \mathbb{E} \left[ \int_0^T\left|\nabla h(X_t)^{\top} \frac{\dd}{\dd t} X_t \right| \mathrm{d} t \right] = \mathbb{E} \left[ \int_0^T\left|\nabla h(X_t)^{\top} \nabla f_{\mu_t,\pi}(X_t) \right| \mathrm{d} t \right] < \infty .
\end{align}
Since $\mathrm{Law}(X_t)=\mu_t$, this becomes
\begin{align}\label{eq:weak_time_derivative_embedding}
    \left\langle \frac{\dd}{\dd t} m_{\mu_t}, h \right\rangle_{\calH}
    = - \int \nabla h(x)^\top \nabla f_{\mu_t,\pi}(x)\,\dd \mu_t(x).
\end{align}
Next, we will apply Eq.~\eqref{eq:weak_time_derivative_embedding} with the fixed test function
$h = m_{\mu_t}-m_\pi$ at time $t$.
Since Eq.~\eqref{eq:weak_time_derivative_embedding} holds for any test function in $\calH$, it also holds for $h=m_{\mu_t}-m_\pi$. 
To see why $m_{\mu_t}-m_\pi$ satisfies the conditions outlined in  Eq.~\eqref{eq:uniform_intergrable}, note that $\|\nabla(m_{\mu_t}-m_\pi)(x)\| = \|\bm{v}_{\mmd}[\mu_t](x)\|\leq d C_{\mmd} (\|x\| + 1)$ proved in Eq.~\eqref{eq:bound_v_mmd}, and also $ \|\nabla f_{\mu_t,\pi}(x) \| \leq d \sqrt{M} (\|x\| + 1)  \cdot \frac{C_{\mmd}}{\lambda}$ proved in Eq.~\eqref{eq:bound_f_mu_pi}. 
Therefore, since $(\mu_t)_{t\in[0,T]}$ is an absolutely continuous curve in $\calP_2(\R^d)$, we have verified Eq.~\eqref{eq:uniform_intergrable}. 
So we obtain
\begin{align*}
    &\quad \frac12 \frac{\dd}{\dd t}\mmd^2(\mu_t,\pi)
    = \frac12 \frac{\dd}{\dd t}\|m_{\mu_t}-m_\pi\|_{\calH}^2
    = \left\langle \frac{\dd}{\dd t} m_{\mu_t}, m_{\mu_t}-m_\pi \right\rangle_{\calH} \\
    &= - \int \nabla (m_{\mu_t}-m_\pi)(x)^\top \nabla f_{\mu_t,\pi}(x)\,\dd \mu_t(x).
\end{align*}
\begin{remark}[No boundary condition]\label{rem:no_boundary}
    \textcolor{red}{Therefore, the same dissipation identity as Appendix C.2 of \cite{mroueh2019sobolev} can be derived directly from the characteristic representation, without invoking integration by parts, and hereby avoiding a boundary condition.} 
\end{remark}
By definition of the gradient covariance operator $S_{\mu_t}$ and by definition of $f_{\mu_t,\pi}$ in \Cref{prop:expressionofsmmd}, we obtain
\begin{align}\label{eq:d_dt_mmd}
\frac{\dd}{\dd t} \frac{1}{2} \,\mmd^{2}(\mu_{t},\pi) &= -\Big\langle\nabla(m_{\mu}-m_{\pi}),\nabla\left({S}_{\mu_{t}}+\lambda \Id\right)^{-1}(m_{\mu}-m_{\pi})\Big\rangle_{L_{2}^d(\mu_{t})} \nonumber \\
&= -\Big\langle {S}_{\mu_{t}}(m_{\mu_{t}}-m_{\pi}),({S}_{\mu_{t}}+\lambda \Id)^{-1}\left(m_{\mu_{t}}-m_{\pi}\right)\Big\rangle_{\mathcal{H}} < 0. 
\end{align}
From Lemma~\ref{lem:hilbertschmidt}, the operator $S_{\mu_t}:\calH\to\calH$ is trace-class, Hilbert-Schmidt, and compact. 
Let $\{\varrho_{i,t}, e_{i,t}\}_{i\in\N}$ be  the countable eigensystem of the operator ${S}_{\mu_t}$ and $\{e_{i,t}\}_{i\in\N}$ is an orthonormal system of $\calH$. 

Since $m_{\mu_t}-m_\pi\in \overline{\ran(S_{\mu_t})}$, we may let $m_{\mu_t}-m_\pi = \sum_{i=0}^\infty m_i e_i$ with $\mmd^2(\mu_t, \pi)=\sum_{i=0}^\infty m_{i,t}^2$, which is finite by the proof of \Cref{thm:ac_curve}. 
Then, we have
\begin{align*}
    \frac{\dd}{\dd t} \frac{1}{2} \,\mmd^{2}(\mu_{t},\pi) &= -\sum_{i=0}^\infty \frac{\varrho_{i,t}}{\varrho_{i,t} + \lambda} m_{i,t}^2 =-\sum_{i=0}^\infty m_{i,t}^2+\lambda\sum_{i=0}^\infty \frac{m_{i,t}^2}{\varrho_{i,t}+\lambda} \\
    &= -\mmd^{2}(\mu_{t},\pi) + \lambda\sum_{i=0}^\infty \frac{m_{i,t}^2}{\varrho_{i,t}+\lambda} := -\mmd^{2}(\mu_{t},\pi) + R(\lambda, t) . 
\end{align*}
Since $|\frac{\lambda m_{i,t}^2}{\rho_{i,t}+\lambda}|\le m_{i,t}^2$, and $\sum_{i=0}^\infty m_{i,t}^2<\infty$,
by dominated convergence theorem (See, e.g., ~\citep[Theorem 1.19]{evans2015measure}), 
we have $\lim_{\lambda\to 0+} R(\lambda, t)  = \lim_{\lambda\to 0+}\sum_{i=0}^\infty \frac{\lambda m_{i,t}^2}{\varrho_{i,t}+\lambda}=0$.
Since this holds for all $t\in[0,T]$, we have $\sup_{t\in[0,T]} \lim_{\lambda\to 0+} R(\lambda, t)=0$. 
Integrating the above differential equation from $0$ to $T$ yields 
\begin{align*}
    \mmd^{2}(\mu_{T},\pi) = e^{-2T} \mmd^{2}(\mu_{0},\pi) + 2 e^{-2 T} \int_0^T e^{2 t} R(\lambda, t) . 
\end{align*}
Under additional assumptions that $m_{\mu_t} - m_\pi\in \ran S_{\mu_t}^{\alpha/2}$ for all $t\in[0,T]$, i.e., 
there exists $q_t\in \calH$ such that $m_{\mu_t} - m_\pi=  S_{\mu_t}^{\alpha/2} q_t$. 
Recall that $m_{\mu_t}-m_\pi = \sum_{i=0}^\infty m_{i,t} e_{i,t}$. 
Then, we can write the above equation as
\begin{align*}
    R(\lambda, t) = \lambda \sum_{i=0}^\infty \frac{m_{i,t}^2 }{\varrho_{i,t} + \lambda} = \lambda \sum_{i=0}^\infty \left( \frac{ \varrho_{i,t}^{\alpha}}{\varrho_{i,t} + \lambda} \right) \varrho_{i,t}^{-\alpha}m_{i,t}^2 \leq \lambda \lambda^{\alpha-1} \sum_{i=0}^\infty  \varrho_{i,t}^{-\alpha}m_{i,t}^2 = \lambda^{\alpha} \|q_t\|_\calH^2 . 
\end{align*}
\end{proof}

\subsection{Proof of \Cref{thm:dt_convergence}}\label{sec:proof_dt_convergence}

\begin{proof}
We know that $\mu_{s+1}=(\Id-\gamma \nabla f_{\mu_s,\pi})_{\#} \mu_s$ where $f_s$ is the witness function $f_{\mu_s,\pi}=({S}_{\mu_s}+\lambda \Id)^{-1}\,(m_{\mu_s}-m_{\pi})$ and $\gamma >0$ is the step size. 
In the proof, we use $f_s := f_{\mu_s,\pi}$ to simplify notation. 
Consider $\rho_t=(\mathrm{I}-t \nabla f_s)_{\#} \mu_s$, which is an interpolation between the distributions within two time steps, such that $\rho_0=\mu_s$, and $\rho_\gamma=\mu_{s+1}$. First, we know that,
\begin{align*}
    &\quad \mmd^2 (\rho_t, \pi) = \iint k(x, z) \dd \rho_t(x) \dd \rho_t(z) - 2 \iint k(x, z) \dd \rho_t(x) \dd \pi(z) + \text{const} \\
    &= \iint k(x- t \nabla f_\mathfrak{s}(x), z - t \nabla f_s(z)) \dd \mu_\mathfrak{s}(x) \dd \mu_s(z) - 2 \iint k(x- t \nabla f_\mathfrak{s}(x), z) \dd \mu_\mathfrak{s}(x) \dd \pi(z) + \text{const} . 
\end{align*}
Here, const represents the term that is independent of time $t$. 
So we have, for the first order derivative, 
\begin{align*}
\begin{aligned}
& \left.\frac{\dd}{\dd t}\right|_{t=0} \frac{1}{2} \mmd^2(\rho_t, \pi) = -\int \nabla f_\mathfrak{s}(x)^{\top}\left(\int \nabla_1 k(x, z) \dd \mu_s(z)-\int \nabla_1 k(x, z) \dd \pi(z)\right) d \mu_\mathfrak{s}(x) \\
= & -\left\langle\nabla\left(m_{\mu_s}-m_\pi\right), \nabla\left({S}_{\mu_s}+\lambda \Id\right)^{-1}\left(m_{\mu_s}-m_\pi\right)\right\rangle_{L_2^d(\mu_s)} \\
= & -\left\langle {S}_{\mu_s} \left(m_{\mu_s}-m_\pi\right),\left({S}_{\mu_s}+\lambda \Id\right)^{-1}\left(m_{\mu_s}-m_\pi\right)\right\rangle_{\mathcal{H}} .  
\end{aligned}
\end{align*}
This is the first-order term, which is the same as the derivations in the continuous-time analogue. 
Next, we define $\phi_t(x)=x-t \nabla f_\mathfrak{s}(x)$.  
We consider the second-order derivative, 
\begin{align*}
&\quad \;\frac{\dd^2}{\dd t^2} \frac{1}{2} \mmd^2(\rho_t, \pi) \\
& =\underbrace{\left\|\int \nabla f_\mathfrak{s}(x)^{\top} \nabla_1 k(x-t\nabla f_\mathfrak{s}(x), \cdot) \dd \mu_\mathfrak{s}(x)\right\|_{\mathcal{H}}^2}_{(i)} \\
&+\underbrace{\int \nabla f_\mathfrak{s}(x)^{\top}\left(\int \mathbf{H}_1 k(x-t\nabla f_\mathfrak{s}(x), z) \dd \rho_t(z)-\int \mathbf{H}_1 k(x-t\nabla f_\mathfrak{s}(x), z) \dd \pi(z)\right) \nabla f_\mathfrak{s}(x) \dd \mu_\mathfrak{s}(x)}_{(i i)} . 
\end{align*}
For the first term $(i)$, we have
\begin{align*}
&(i)= \iint \nabla f_\mathfrak{s}(x)^\top \nabla_1\nabla_2 k(\phi_t(x), \phi_t(y)) \nabla f_\mathfrak{s}(y) \; \dd \mu_\mathfrak{s}(x) \dd \mu_\mathfrak{s}(y)  \\
&= \iint \left\langle \nabla f_\mathfrak{s}(x) \nabla f_\mathfrak{s}(y)^\top, \;  \nabla_1\nabla_2 k(\phi_t(x), \phi_t(y)) \right\rangle_F \; \dd \mu_\mathfrak{s}(x) \dd \mu_\mathfrak{s}(y) \\
&\leq \iint \left\| \nabla f_\mathfrak{s}(x) \nabla f_\mathfrak{s}(y)^\top \right\|_F \cdot \left\| \nabla_1\nabla_2 k(\phi_t(x), \phi_t(y)) \right\|_F \; \dd \mu_\mathfrak{s}(x) \dd \mu_\mathfrak{s}(y) \\
&\leq \left( \iint \left\| \nabla f_\mathfrak{s}(x) \nabla f_\mathfrak{s}(y)^\top \right\|_F^2 \dd \mu_\mathfrak{s}(x) \dd \mu_\mathfrak{s}(y) \right)^{\frac{1}{2}} \cdot \left( \iint \left\| \nabla_1\nabla_2 k(\phi_t(x), \phi_t(y)) \right\|_F^2 \dd \mu_\mathfrak{s}(x) \dd \mu_\mathfrak{s}(y) \right)^{\frac{1}{2}} \\
&= \| \nabla f_s\|_{L_2^d(\mu_s)}^2 \cdot \left( \iint \left\| \nabla_1\nabla_2 k(x, y) \right\|_F^2 \dd \rho_t(x) \dd \rho_t(y) \right)^{\frac{1}{2}} .
\end{align*}
Since for any $i,j\in[d]$, 
\begin{align*}
    &\quad \iint \left| \partial_{1, i} \partial_{2,j} k(x, y) \right|^2 \dd \rho_t(x) \dd \rho_t(y) \leq \iint \left| \partial_{1, i} \partial_{2, i} k(x, x) \right| \cdot \left| \partial_{1, j} \partial_{2, j} k(y, y) \right| \dd \rho_t(x) \dd \rho_t(y) \\
    &\leq \left( \int M^2 (1 + \|x\|)^2\dd \rho_t(x)  \right)^{\frac{1}{2}} \cdot \left( \int M^2 (1 + \|y\|)^2 \dd \rho_t(y) \right)^{\frac{1}{2}} \leq M^2 4C . 
\end{align*}
The second last inequality holds by the growth condition in \Cref{ass:smmd_growth}, and the last inequality holds by the convexity of $\|\cdot\|^2$: 
\begin{align*}
    \int\|z\|^2 \dd \rho_t(z) \leq(1-t\gamma^{-1}) \int\|x\|^2 \dd \mu_\mathfrak{s}(x) + t\gamma^{-1} \int\|y\|^2 \dd \mu_{s+1}(y) \leq C.
\end{align*} 
The last inequality holds by the uniform bound on the second moment assumed in the statement of the theorem. 
Therefore, we have proved that
\begin{align*}
    (i) \leq d M^2 4C \| \nabla f_s\|_{L_2^d(\mu_s)}^2 . 
\end{align*}
For the second term $(ii)$, we have
\begin{align*}
&\quad \iint \nabla f_\mathfrak{s}(x)^{\top} \mathbf{H}_1 k(x-\gamma \nabla f_\mathfrak{s}(x), z)  \nabla f_\mathfrak{s}(x)\dd \rho_t(z) \dd \mu_\mathfrak{s}(x) \\
&\leq \iint \left\| \nabla f_\mathfrak{s}(x) \right\|^2 \cdot \left\| \mathbf{H}_1 k(x-\gamma \nabla f_\mathfrak{s}(x), z) \right\|_{op} \; \dd \mu_\mathfrak{s}(x) \dd \rho_t(z) \\
&\leq \left( \int \| \nabla f_\mathfrak{s}(x)\|^2\; \dd \mu_\mathfrak{s}(x)  \right) \cdot \left( \int M(1 + \|z\|^2) \rho_t(z) \right) \\
&= \| \nabla f_s\|_{L_2^d(\mu_s)}^2 (M + MC). 
\end{align*}
The second-to-last line holds under the growth conditions in \Cref{ass:growth_2}. 
We can repeat the same derivations above to obtain an upper bound on the second term for $(ii)$ as well. 
Combine the above upper bounds on $(i)$ and $(ii)$, then we can obtain
\begin{align*}
    \;\frac{\dd^2}{\dd t^2} \frac{1}{2} \mmd^2(\rho_t, \pi) &\leq d (4M^2 C + 2 (M + MC)) \| \nabla f_s\|_{L_2^d(\mu_s)}^2 =: d C_0 \| \nabla f_s\|_{L_2^d(\mu_s)}^2 \\
    &= d C_0 \langle {S}_{\mu_s} ({S}_{\mu_s}+\lambda \Id)^{-1}\,(m_{\mu_s} -m_{\pi}), ({S}_{\mu_s}+\lambda \Id)^{-1}\,(m_{\mu_s} -m_{\pi}) \rangle_\calH . 
\end{align*}
The constant $C_0$ only depends on the other constants $M, C$ in the assumptions. 
Putting the above together,
\begin{align*}
    &\quad \frac{1}{2} \mmd^2(\mu_{s+1}, \pi)-\frac{1}{2} \mmd^2(\mu_s, \pi)\\
    &=\left.\frac{d}{d t}\right|_{t=0} \frac{1}{2}\mmd^2\left(\rho_t, \pi\right) \gamma + \int_0^\gamma(\gamma-t) \frac{\dd^2}{\dd t^2} \frac{1}{2} \mmd^2\left(\rho_t, \pi\right) \dd t \\
    &\leq -\gamma \left\langle {S}_{\mu_s} \left(m_{\mu_s}-m_\pi\right), \left({S}_{\mu_s}+\lambda \Id\right)^{-1} \left(m_{\mu_s}-m_\pi\right)\right\rangle_{\mathcal{H}} \\
    &\qquad + \gamma^2 dC_0 \langle {S}_{\mu_s} ({S}_{\mu_s}+\lambda \Id)^{-1}\,(m_{\mu_s} -m_{\pi}), ({S}_{\mu_s}+\lambda \Id)^{-1}\,(m_{\mu_s} -m_{\pi}) \rangle_\calH . 
\end{align*}
Since $m_{\mu_s}-m_\pi\in \ran  S_{\mu_s}^{\alpha/2}$, there exists $q_s\in \calH$ such that $m_{\mu_s}-m_\pi=  S_{\mu_s}^{\alpha/2} q_s$. Recall that  $m_{\mu_s}-m_\pi = \sum_{i=0}^\infty m_{i,s} e_{i,s}$, where $\{\varrho_{i,s}, e_{i,s} \}_{i\in\N}$ is the eigensystem of the operator ${S}_{\mu_s}$. Then, we can write the above equation as
\begin{align}
    &\ \ -\sum_{i=0}^\infty \frac{\gamma \varrho_{i,s}}{\varrho_{i,s} + \lambda} m_{i,s}^2 + dC_0 \sum_{i=0}^\infty \frac{\gamma^2 \varrho_{i,s}}{(\varrho_{i,s} + \lambda)^2} m_{i,s}^2 \nonumber \\
    &=-\gamma \sum_{i=0}^\infty m_{i,s}^2 + \gamma\sum_{i=0}^\infty \frac{\lambda \varrho_{i,s}^{\alpha}}{\varrho_{i,s} + \lambda} \varrho_{i,s}^{-\alpha} m_{i,s}^2 + dC_0 \gamma^2\sum_{i=0}^\infty \pth{\frac{ \varrho_{i,s}^{\frac{1+\alpha}2}}{\varrho_{i,s} + \lambda}}^2 \rho_{i,s}^{-\alpha}m_{i,s}^2 \nonumber \\
    &\stackrel{\text{(a)}}\le -\gamma \mmd^2(\mu_s ,\pi) +\gamma \lambda^{\alpha} \sum_{i=0}^\infty \rho_{i,s}^{-\alpha}m_{i,s}^2 + dC_0 \gamma^2\lambda^{\alpha-1}\|q_s\|_{\cal H}^2 \label{eq:source} \\
    &= -\gamma \mmd^2(\mu_s ,\pi) +\gamma \lambda^{\alpha} \|q_s\|_{\cal H}^2 + dC_0\gamma^2\lambda^{\alpha-1}\|q_s\|_{\cal H}^2, \nonumber
\end{align}
where (a) follows from Lemma~\ref{lem:younginequality}.
The second term goes to $0$ as $\lambda\to0$, and the third term goes to $\infty$ as $\lambda\to0$ because $\alpha < 1$. To balance these two terms, we take $\lambda$ such that $\gamma \lambda^{\alpha} \|q_s\|_{\cal H}^2 = dC_0\gamma^2\lambda^{\alpha-1}\|q_s\|_{\cal H}^2$, which gives $\lambda = dC_0\gamma$. So we have
\begin{align*}
    \mmd^2(\mu_{s+1}, \pi)\le (1-2 \gamma) \mmd^2(\mu_s ,\pi) + 4 (dC_0)^\alpha \gamma^{1+\alpha} \|q_s\|_{\cal H}^2  . 
\end{align*}
Since $0<\gamma<\frac{1}{2}$, after iterating $s$ from $1$ to $S$, we obtain 
\begin{align*}
    \mmd^2(\mu_S, \pi)\le (1-2\gamma)^S \mmd^2\left(\mu_0, \pi\right) + 4 (dC_0)^\alpha \gamma^{1+\alpha} \sum_{s=0}^{S-1} (1-2 \gamma)^{S-1-s} \|q_s\|_{\cal H}^2.
\end{align*}
Furthermore, if $\|q_s\|^2\le Q$ holds for all $0\le s\le S$, then
\begin{align*}
    \mmd^2(\mu_S, \pi) &\le  \exp(-2S \gamma) \mmd^2\left(\mu_0, \pi\right) + 2 (dC_0)^\alpha \gamma^{1+\alpha} Q \sum_{s=0}^{\infty} (1-2 \gamma)^{S-1-s}\\
    &=\exp(-2 S \gamma)\mmd^2\left(\mu_0, \pi\right) + 4 (dC_0)^\alpha \gamma^{\alpha} Q . 
\end{align*}
\end{proof}

\subsection{Proof of \Cref{prop:closeformofsmmd}}\label{sec:proof_closed_form}
\begin{proof}
Define the following operators: 
\begin{align*}
    &D_X:\calH\to \R^{Nd},\quad 
    f\mapsto \frac1{\sqrt N} \bigl[\partial_{1} f(x^{(1)}),\cdots, \partial_{d} f(x^{(N)})\bigr]^\top,\\
    &D_X^*:\R^{Nd}\to \calH, \quad 
    v\mapsto \frac1{\sqrt N}\sum_{i=1}^N\sum_{l=1}^d v_{(i,l)}\,\partial_{1,l} k(x^{(i)},\cdot).
\end{align*}
Then $S_{\hat\mu}=D_X^*D_X$ and $D_XD_X^*=\frac1N H_{XX}$.
The witness function, by definition, admits the representation
\begin{align*}
    f_{\hat\mu,\pi}
    = (S_{\hat\mu}+\lambda \Id)^{-1}(m_{\hat\mu}-m_{\pi}) =(D_X^*D_X+\lambda \Id)^{-1}\left(\frac1N\sum_{i=1}^N k(x^{(i)},\cdot)- \E_{\pi}[k(Y, \cdot)] \right).
\end{align*}
By the Woodbury identity,
$(D_X^*D_X+\lambda \Id)^{-1}
=\frac1\lambda \Id-\frac1\lambda D_X^*(D_XD_X^*+\lambda \Id)^{-1}D_X$, 
we obtain
\begin{align*}
    f_{\hat\mu,\pi}
    &=
    \frac1\lambda \left(\frac1N\sum_{i=1}^N k(x^{(i)},\cdot) - \E_{\pi}[k(Y, \cdot)]  \right)
    -\frac1\lambda D_X^*(D_XD_X^*+\lambda \Id)^{-1}D_X(m_{\hat\mu} - m_\pi).
\end{align*}
Next, note that
$(D_XD_X^*+\lambda \Id)^{-1}
= N\,(H_{XX}+N\lambda \Id)^{-1}$. 
Moreover, for the mean embeddings,
\begin{align*}
    D_X m_{\hat\mu} = \frac{1}{N} D_X [\mathbb{K}_X \1_N^\top ] = \frac{1}{N\sqrt{N}} D_{XX}\1_N, \quad
    D_X m_{\pi} = \frac{1}{\sqrt{N}} \E_{\pi}[\mathbb D_X(Y)] .
\end{align*}
Putting these pieces together,
\begin{align*}
    f_{\hat\mu,\pi}
    &=
    \frac1{N\lambda}\1_N^\top \mathbb K_X
    -\frac1{\lambda} \E_{\pi}[k(Y, \cdot)]
    -\frac1\lambda \mathbb D_X^\top (H_{XX}+N\lambda \Id)^{-1}
    \left(\frac1N D_{XX}\1_N-\E_{\pi}[\mathbb D_X(Y)] \right),
\end{align*}
which is the desired expression.
\end{proof}

\paragraph{Additional notations for \Cref{lem:full_support}}: 
All orthogonal complements and closures are taken with respect to the RKHS $\calH$. 
In particular, $\calG^{\perp_{\calH}}$ denotes the orthogonal complement of a subspace $\calG\subseteq\calH$, and $\overline{\calG}$ denotes its closure in $\calH$.

\begin{lemma}
\label{lem:full_support}
Let $\mu\in\calP_2(\R^d)$ have full support on $\R^d$. 
Let $D_\mu:\calH\to L_2^d(\mu)$ and $D_\mu^\ast:L_2^d(\mu)\to\calH$ be the operators defined in the main text. 
Let $S_\mu:=D_\mu^*D_\mu$. 
If the RKHS $\calH$ contains constant functions, then $\overline{\ran(S_\mu)}=\ker(S_\mu)^{\perp_\calH}=\mathrm{span}\{1\}^{\perp_\calH}$; otherwise $\overline{\ran(S_\mu)}=\{0\}^{\perp_\calH}$. 
In both cases, $m_{\mu}-m_\pi\in\overline{\ran(S_{\mu})}$. 
\end{lemma}
\begin{proof}
Since $S_\mu=D_\mu^*D_\mu$, for any $f\in\calH$ we have
\begin{align*}
    \langle f,S_\mu f\rangle_{\calH}
    =
    \|D_\mu f\|_{L_2^d(\mu)}^2
    =
    \E_{X\sim\mu}\bigl[\|\nabla f(X)\|^2\bigr].
\end{align*}
Hence $f\in\ker(S_\mu)$ if and only if $\nabla f=0$ $\mu$-almost surely. Since $\nabla f$ is continuous and $\mu$ has full support on $\R^d$, this implies $\nabla f(x)=0$ for all $x\in\R^d$. Indeed, if $\nabla f(x_0)\neq 0$ for some $x_0\in\R^d$, then by continuity of the derivative of $f$ there exists an open ball $B(x_0,r)$ on which $\|\nabla f\|>0$. Since $\mu$ has full support, $\mu(B(x_0,r))>0$, contradicting $\nabla f=0$ $\mu$-almost surely. Therefore $\nabla f=0$ everywhere, and since $\R^d$ is connected, $f$ is constant. 

By the preceding argument, under the full-support assumption on $\mu$, the null space of $S_\mu$ consists precisely of the constant functions that belong to $\calH$. 
Thus, if $1\in\calH$, then $\ker(S_\mu)=\mathrm{span}\{1\}$, whereas if $1\notin\calH$, then $\ker(S_\mu)=\{0\}$. 
The latter case occurs, for instance, for the Gaussian kernel on $\R^d$ \citep[Corollary 4.44]{steinwart2008support}. 
Since $S_\mu$ is self-adjoint, $\overline{\ran(S_\mu)}=\ker(S_\mu)^{\perp_{\calH}}$. 
If $1\notin\calH$, this gives $\overline{\ran(S_\mu)}=\calH$, and the claim is immediate. 
If $1\in\calH$, then $\langle m_{\mu}-m_\pi,1\rangle_{\calH}=\int 1\,\dd\mu-\int 1\,\dd\pi=1-1=0$, and hence $m_\mu-m_\pi\in \mathrm{span}\{1\}^{\perp_{\calH}}=\overline{\ran(S_\mu)}$.
\end{proof}

\paragraph{Additional notations for \Cref{lem:sobolevequivalent}}: 
For two normed vector spaces $A, B, A \cong B$ means that $A$ and $B$ are norm equivalent, i.e., their sets coincide, and the corresponding norms are equivalent. In other words, there are constants $c_1, c_2>0$ such that $c_1\|h\|_A \leq \|h\|_B \leq c_2\|h\|_A$ holds for all $h \in A$. For a function $f:\Tb^d\to\R$ defined on the d-dimensional torus, we write its Fourier series as $\widehat{f}_n = \int_{\Tb^d} f(x)\exp(-\mathrm i n^\top x) \dd x$ for any $n\in\Z^d$. 

\begin{lemma}\label{lem:sobolevequivalent}
    Let $\mathbb T^d=[0,2\pi)^d$ be a $d$-dimensional torus, and let $\mu$ be a probability measure on $\mathbb T^d$ with Lebesgue density $p$ satisfying $0<c<p(x)<C<\infty\text{ a.e. on }\mathbb T^d$ for some universal constants $c, C$. Let $k:\R^d\times\R^d\to \R$ be a periodic translation-invariant Mat\'ern kernel, i.e., $k(x,x')=q(x-x')$, whose Fourier series satisfies $\widehat{q}_n=(1+\|n\|^2)^{-r}$ for some $r>\max\{d/2,1\}$. Let $\mathcal H=\mathcal H(\mathbb T^d)$ denote the RKHS induced by $k$ on $\mathbb T^d$. Then, for $0<\alpha\le 1$,
    $$\ran(S_\mu^{\alpha/2})\cong H^{r(\alpha+1)-\alpha}(\mathbb T^d)\,\cap \mathrm{span}\{1\}^{\perp_\calH},$$
where $H^\ell(\mathcal X)$ denotes the Sobolev space of order $\ell$ on $\mathcal X$.
\end{lemma}

\begin{proof}
Denote $e_n(x)=e^{\mathrm i n^\top x}$ for all $n\in \mathbb Z^d$. Since $k(x,x')=q(x-x')$ is periodic and translation-invariant, we know that $\calH\cong H^r(\mathbb T^d)$; (See, e.g., \cite[Example on Sobolev spaces on the torus]{Koltchinskii2010SPARSITYIM}). 
Moreover,
$$q(x)=\sum_{n\in \mathbb Z^d}\hat q_n e_n(x),\qquad \hat q_n=(1+\|n\|^2)^{-r}>0.$$
Therefore, the RKHS can be written as
$$\calH=\Bigl\{f=\sum_{n\in \mathbb Z^d} \hat{f}_n e_n(x): \|f\|_\calH^2=\sum_{n\in \mathbb Z^d}\frac{|\hat{f}_n|^2}{\hat q_n}<\infty\Bigr\},$$
and $\phi_n=\hat q_n^{1/2}e_n,\,n\in \mathbb Z^d$ is an orthonormal basis of RKHS $\calH$. 

We first consider the Lebesgue measure $\mathfrak{m}$ on $\mathbb T^d$. Since $r\geq 1$, for any $f=\sum_n \hat{f}_ne_n$, $g=\sum_n\hat{g}_ne_n\in \calH$, we have
$$\langle f,S_{\mathfrak{m}} g\rangle_\calH=\int_{\mathbb T^d}\nabla f(x)^\top \nabla g(x) \dd x=\sum_n \|n\|^2\hat{f}_n\overline{\hat{g}_n}.$$
This also implies $S_{\mathfrak{m}} \phi_n=\tau_n\phi_n$, and $\tau_n=\|n\|^2\hat q_n$. 
Therefore, $f\in \ran(S_{\mathfrak{m}}^{\alpha/2})$ if and only if there exists $g=\sum_n b_n\phi_n$ such that
$$f=S_{\mathfrak{m}}^{\alpha/2}g=\sum_n \tau_n^{\alpha/2} b_n\phi_n=\sum_n\tau_n^{\alpha/2}b_n\hat q_n^{1/2}e_n=\sum_n \|n\|^\alpha \hat q_n^{\frac{1+\alpha}2}b_ne_n.$$
By the uniqueness of Fourier expansion, we have $\hat{f}_n=\|n\|^\alpha \hat q_n^{\frac{1+\alpha}2}b_n.$
Since $g\in \calH$, it requires $\sum_n|b_n|^2<\infty$, i.e.,
$$\infty>\sum_{n\ne 0}\frac{|\hat{f}_n|^2}{\|n\|^{2\alpha}|\hat q_n|^{1+\alpha}}
=\sum_{n\ne 0}\|n\|^{-2\alpha}(1+\|n\|^2)^{r(1+\alpha)}|\hat{f}_n|^2,$$
and $\hat{f}_0=0$. Since for $n\neq 0$,
$$\|n\|^{-2\alpha}(1+\|n\|^2)^{r(1+\alpha)}\asymp (1+\|n\|^2)^{r(1+\alpha)-\alpha},$$
we obtain
$$\ran(S_{\mathfrak{m}}^{\alpha/2})=\Bigl\{f\in \calH: \hat{f}_0=0,\, \sum_{n\ne 0}(1+\|n\|^2)^{r(1+\alpha)-\alpha}|\hat{f}_n|^2<\infty\Bigr\}.$$

Since
$$H^\ell(\mathbb T^d)=\Bigl\{f\in L_2(\mathbb T^d): \sum_n(1+\|n\|^2)^{\ell}|\hat{f}_n|^2<\infty\Bigr\},$$
and $\hat{f}_0=0$ if and only if $\langle f,1\rangle_\calH=\hat q_0^{-1}\hat{f}_0=0$, if and only if $f\in \mathrm{span}\{1\}^{\perp_\calH}$, therefore
$$\ran(S_{\mathfrak{m}}^{\alpha/2})\cong H^{r(\alpha+1)-\alpha}(\mathbb T^d)\cap \mathrm{span}\{1\}^{\perp_\calH}.$$

Finally, we consider a general probability measure $\mu$ on $\mathbb T^d$. Recall that
$$\langle f, S_\mu g\rangle_\calH=\int_{\mathbb T^d}\nabla f(x)^\top \nabla g(x) p(x)\dd x.$$
Since $c\le p\le C$, we have
$$c\langle f,S_{\mathfrak{m}} f\rangle_{\mathcal H}
\le\langle f,S_\mu f\rangle_{\mathcal H} \le C\langle f, S_{\mathfrak{m}} f\rangle_{\mathcal H},
\quad f\in \mathcal H.$$ 
Now consider the restrictions of $S_\mu$ and $S_{\mathfrak{m}}$ on $\mathcal H_0:=\{1\}^{\perp_\mathcal H}$. On $\mathcal H_0$, we have $c\,S_{\mathfrak{m}}\preceq S_\mu\preceq C\,S_{\mathfrak{m}}$.
By the Löwner--Heinz inequality, one obtains $c^\alpha S_{\mathfrak{m}}^\alpha \preceq S_\mu^\alpha \preceq C^\alpha S_{\mathfrak{m}}^\alpha,\; \forall\,0<\alpha\le 1$, 
still on $\mathcal H_0$. Then, an application of Douglas' lemma~\citep{douglas1966majorization} to $S_\mu^{\alpha/2}$ and $S_{\mathfrak{m}}^{\alpha/2}$ on $\mathcal H_0$ shows that
\begin{align}\label{eq:s_equivalence}
    \ran(S_\mu^{\alpha/2})\cong \ran(S_{\mathfrak{m}}^{\alpha/2}),\quad \forall\,0<\alpha\le 1 .
\end{align}
Combining this with the previous result for $S_{\mathfrak{m}}$, we obtain
$$\ran(S_\mu^{\alpha/2})
\cong
H^{r(\alpha+1)-\alpha}(\mathbb T^d)\cap \mathrm{span}\{1\}^{\perp_{\mathcal H}}.$$
\end{proof}

\begin{lemma}\label{lem:suff_range}
Let $\mathbb T^d=[0,2\pi)^d$ be a $d$-dimensional torus. Let $k$ be a periodic Mat\'ern kernel on a $d$-dimensional torus $\mathbb T^d$ whose RKHS is norm equivalent to a Sobolev space $H^r(\Tb)$. 
Let $\mu$ be a distribution on $\Tb$ that is equivalent to the Lebesgue measure $\mathfrak{m}$; and also $\frac{\dd \mu}{\dd \mathfrak{m}}\in C^\infty(\Tb)$. 
If $\frac{\dd \pi}{\dd \mu} - 1\in H^{r\alpha-r-\alpha}(\Tb)$, then $m_\mu-m_\pi\in \ran(S_\mu^{\alpha/2})$. 
\end{lemma}
\begin{proof}
Following the same reasoning as in the lemma above, it suffices to first consider the case where $\mu$ is the Lebesgue measure $m$ on $\mathbb T^d$. Write $e_n(x)=e^{\mathrm i n^\top x}$ for $n\in\mathbb Z^d$. Since $k(x,y)=q(x-y)$ is a periodic Mat\'ern kernel whose RKHS is norm equivalent to $H^r(\mathbb T^d)$, we have
\[
q(x)=\sum_{n\in\mathbb Z^d}\hat q_n e_n(x),
\quad
\hat q_n \asymp (1+\|n\|^2)^{-r}.
\]
We have that $\phi_n=\hat q_n^{1/2}e_n,\,n\in \mathbb Z^d$ is an orthonormal basis of RKHS $\calH$. 
Define the following two operators: 
\begin{align*}
    \iota_{\mathfrak{m}}: \mathcal{H} \rightarrow L_2(\mathfrak{m}), \quad f \mapsto f, \quad \text{and} \quad \iota_{\mathfrak{m}}^*: L_2(\mathfrak{m}) \rightarrow \mathcal{H}, \quad \iota_{\mathfrak{m}}^* f(\cdot)=\int_{\Tb} k(x, \cdot) f(x) \mathrm{d}x .
\end{align*}
We thus have
\begin{align*}
    [\widehat{m_\mu-m_\pi}]_n &= \langle m_\mu-m_\pi, \phi_n \rangle_\calH = \left\langle \iota_{\mathfrak{m}}^\ast \left(\frac{\dd \pi}{\dd m} - 1\right) , \phi_n \right\rangle_\calH = \left\langle \frac{\dd \pi}{\dd \mathfrak{m}} - 1, \iota_{\mathfrak{m}} \phi_n \right\rangle_{L_2(\Tb)} \\
    &= \hat{q}_n^{1/2} \left\langle \frac{\dd \pi}{\dd \mathfrak{m}} - 1, e_n \right\rangle_{L_2(\Tb)} . 
\end{align*}
By the proof of \Cref{lem:sobolevequivalent}, we know that a sufficient condition for $m_\mu-m_\pi\in\ran(S_\mu^{\alpha/2})$ is that
\begin{align*}
    \sum_n |[\widehat{m_\mu-m_\pi}]_n|^2 \left( \|n\|^2 \hat{q}_n \right)^{-\alpha} < \infty \iff \sum_n \left\langle \frac{\dd \pi}{\dd \mathfrak{m}} - 1, e_n \right\rangle_{L_2(\Tb)}^2 (1 + \|n\|^2)^{-r(1-\alpha) -\alpha} < \infty. 
\end{align*}
By the definition of the Sobolev space $H^{r\alpha-r-\alpha}(\Tb)$, the proof is concluded. 

Next, we consider the more general case where $\mu$ is a distribution on $\Tb$ that is equivalent to the Lebesgue measure $\mathfrak{m}$; and also $\frac{\dd \mu}{\dd \mathfrak{m}}\in C^\infty(\Tb)$. Following the same argument above, we can show that $m_\mu - m_\pi = \iota_{\mathfrak{m}}^\ast (\frac{\dd \mu}{\dd \mathfrak{m}} \cdot (\frac{\dd \pi}{\dd \mu}-1))$. 
Since $\frac{\dd \pi}{\dd \mu}-1\in H^{r\alpha-r-\alpha}(\Tb)$ and $\frac{\dd \mu}{\dd \mathfrak{m}}\in C^\infty(\Tb)$, there is $(\frac{\dd \mu}{\dd \mathfrak{m}} \cdot (\frac{\dd \pi}{\dd \mu}-1))\in H^{r\alpha-r-\alpha}(\Tb)$ and hence $m_\mu - m_\pi\in\ran(S_{\mathfrak{m}}^{\alpha/2})$. The proof is concluded by that $\ran(S_\mu^{\alpha/2})\cong \ran(S_{\mathfrak{m}}^{\alpha/2})$ proved in Eq.~\eqref{eq:s_equivalence} above. 
\end{proof}

\subsection{Geometry of SrMMD flow}\label{sec:srmmd_geometry}
In this section, we expand on \Cref{rem:wgf} and discuss the geometry underlying the SrMMD flow. 

First, we recall that, in the formal Riemannian interpretation of Wasserstein gradient flows, the tangent space $T_\mu \calP_2(\R^d)$ at a measure $\mu\in\calP_2(\R^d)$ is identified with a subspace of vector fields in $L_2^d(\mu)$. 
More precisely, it is commonly taken as the $L_2^d(\mu)$-closure of gradients of smooth compactly supported functions~\citep{villani2009optimal}. 
The Wasserstein metric tensor is then given by the kinetic-energy inner product $g_\mu(v,w) := \langle v,w\rangle_{L_2^d(\mu)}$. 
Indeed, for a smooth functional $\calF:\calP_2(\R^d)\to\R$ with first variation $\frac{\delta \calF}{\delta \mu}$, an absolutely continuous curve satisfying
$\partial_t \mu_t+\nabla\cdot(\mu_t v_t)=0$ obeys the formal chain rule~\citep{chewi2024logconcavesampling}: 
\[
    \frac{\dd}{\dd t}\calF(\mu_t)
    =
    \int_{\R^d}
    \nabla \frac{\delta \calF}{\delta \mu_t}(x)^\top v_t(x)
    \,\dd\mu_t(x).
\]
Thus the gradient of $\calF$ with respect to the optimal transport geometry, $g_\mu$ defined above, is represented by $\mathrm{grad}_W \calF(\mu)=\nabla \frac{\delta \calF}{\delta \mu}$, and the steepest descent velocity is
$v_t=-\nabla \frac{\delta \calF}{\delta \mu_t}$. 

Now, consider the SrMMD flow. 
Recall the operator $D_\mu^\ast:L_2^d(\mu)\to \calH$ defined in the main text. 
Define a Sobolev modified Riemannian metric tensor $g_\mu^\lambda$ with $\lambda>0$: 
for two tangent vector fields $v, w \in T_\mu \mathcal{P}_2(\mathbb{R}^d)$, 
\begin{align}
    g_\mu^\lambda(v, w):=\left\langle D_\mu^* v, D_\mu^* w\right\rangle_{\mathcal{H}}+\lambda\langle v, w\rangle_{L_2^d(\mu)} .
\end{align}
\begin{theorem}
    Suppose that \Cref{ass:bounded,ass:smmd_growth} hold. 
    Suppose $m_\mu-m_\pi \in \overline{\operatorname{Ran}(S_\mu)}$. The steepest descent of the MMD functional $\mu\mapsto \frac{1}{2}\mmd^2(\mu, \pi)$ under the metric $g_\mu^\lambda$ with $\lambda>0$ yields the SrMMD velocity field $v=-\nabla\left(S_\mu+\lambda \Id \right)^{-1}\left(m_\mu-m_\pi\right)$.
\end{theorem}
\begin{proof}
    Let $\phi\in C_c^\infty(\R^d)\cap\calH$ such that $\nabla \phi \in \calT_\mu\calP_2(\R^d)$. 
    Consider an absolutely continuous curve $t\mapsto (\Id + t\nabla \phi)_{\#} \mu=: \eta_t$. 
    Following the same derivation at the start of \Cref{sec:proof_ct_convergence}, we obtain that 
    \begin{align*}
        \left.\frac{\dd}{\dd t}\right|_{t=0}  \frac{1}{2}\mmd^2(\eta_t, \pi) = \int \nabla (m_{\mu}-m_\pi)(x)^\top \nabla \phi(x) \,\dd \mu(x) = \left\langle m_{\mu}-m_\pi, \, S_\mu \phi \right\rangle_\calH. 
    \end{align*}
    On the other hand, by definition in metric geometry, a gradient vector field $\operatorname{grad}_\mu \mathcal{E}=\nabla f$ (for some witness potential $f \in \mathcal{H}$ ) must satisfy the duality pairing: $g_\mu^\lambda(\nabla f, \nabla \phi) = \left.\frac{\dd}{\dd t}\right|_{t=0}  \frac{1}{2}\mmd^2(\eta_t, \pi)$. 
    From $g_\mu^\lambda$ defined above, we have
    \begin{align*}
g_\mu^\lambda(\nabla f, \nabla \phi) & =\left\langle S_\mu f, S_\mu \phi\right\rangle_{\mathcal{H}}+\lambda\left\langle f, S_\mu \phi\right\rangle_{\mathcal{H}} =\left\langle\left(S_\mu+\lambda \mathrm{Id}\right) f, S_\mu \phi\right\rangle_{\mathcal{H}} . 
\end{align*}
Comparing the two equations yields: $\left\langle\left(S_\mu+\lambda \mathrm{Id}\right) f, S_\mu \phi\right\rangle_{\mathcal{H}} = \left\langle m_\mu-m_\pi, S_\mu \phi\right\rangle_{\mathcal{H}}$. 
This equality must hold for all test functions $\phi \in \mathcal{H}$, meaning it holds for all elements $S_\mu \phi\in\overline{\operatorname{Ran}(S_\mu)}$. Since $m_\mu-m_\pi \in \overline{\operatorname{Ran}(S_\mu)}$ and we can freely pick $f\in\overline{\operatorname{Ran}(S_\mu)}$, we can drop the inner product with $S_\mu \phi$ to uniquely identify the operator equation:
$$
\left(S_\mu+\lambda \Id\right) f=m_\mu-m_\pi .
$$
Inverting the strictly positive-definite operator ($S_\mu+\lambda\Id$), we identify the optimal witness potential $f$ :
$
f=\left(S_\mu+\lambda \Id \right)^{-1}\left(m_\mu-m_\pi\right) .
$
Consequently, the steepest descent vector field is 
$
v=-\nabla f=-\nabla\left(S_\mu+\lambda \mathrm{Id}\right)^{-1}\left(m_\mu-m_\pi\right) .
$
\end{proof}

\section{Auxiliary Results}
\begin{lemma}\label{lem:younginequality}
    For any $\lambda>0$ and $s\in[0,1]$, we have 
    $$\sup_{t\ge 0}\frac{t^s}{t+\lambda}\le \lambda^{s-1}.$$
\end{lemma}
\begin{proof}
    Since $a^s\le a+1$ for any $a\ge 0$ and $0\le s\le 1$, the lemma follows from 
    $$\pth{\frac t\lambda}^s\le 1+\frac{t}{\lambda}=\frac{t+\lambda}{\lambda}.$$
\end{proof}

\begin{lemma}\label{lem:tc-hs-equivalence} Given two Hilbert spaces $\calH_1$, $\calH_2$.
    For any bounded linear operator $T: \calH_1\to \calH_2$, $T$ is Hilbert-Schmidt if and only if $T^*T$ is trace-class. Moreover, $\tr(T^*T)=\|T\|_{\text{HS}}^2$. 
\end{lemma}
\begin{proof}
    We write $\mathbb N$ for either $\mathbb N=\{1,2,\cdots\}$ or $\mathbb N=\{1,2,\cdots, N\}$ for some fixed finite integer $N$.

First, assume that $T$ is Hilbert-Schmidt. For any orthonormal basis $\{e_n\}_{n\in \mathbb N}$ of $\calH_1$,
\begin{align*}
    \sum_{n\in\mathbb N}\langle e_n, T^*T e_n\rangle_{\calH_1}
    =
    \sum_{n\in \mathbb N}\|T e_n\|_{\calH_1}^2
    <\infty.
\end{align*}
Since $T^*T$ is a positive operator, it follows that $T^*T$ is trace-class and
\begin{align*}
\tr(T^*T)=\sum_{n\in \mathbb N}\|T e_n\|_{\calH_1}^2<\infty .
\end{align*}

Conversely, assume that $T^*T$ is trace-class. Since $T^*T$ is positive, for any orthonormal basis $\{e_n\}_{n\in \mathbb N}$,
\begin{align*}
\tr(T^*T)
=
\sum_{n\in \mathbb N}\langle e_n,T^*T e_n\rangle_{\calH_1}
=
\sum_{n\in\mathbb N}\|T e_n\|_{\calH_1}^2
<\infty.
\end{align*}
This implies that
\begin{align*}
\|T\|_{\mathrm{HS}}^2=\sum_{n\in\mathbb N}\|T e_n\|^2<\infty,
\end{align*}
i.e., $T$ is Hilbert--Schmidt.
\end{proof}

\begin{lemma}\label{lem:D_mu}
Suppose that the kernel $k$ is continuously differentiable with respect to both of its arguments, and that it satisfies $\int \partial_{1,i}\partial_{2,i} k(x, x) \dd \mu(x) < \infty$ for any $i\in [d]$. Then the operator $D_\mu:f\in\calH\mapsto \nabla f$ is well defined, linear, and bounded. Furthermore, its adjoint operator $D_\mu^*: L_2^d(\mu)\to\calH$ is also a bounded linear operator. 
\end{lemma}
\begin{proof}
    From Corollary 4.36 of \cite{steinwart2008support}, an immediate consequence of the continuously differentiability of $k$ is that all functions $f\in\calH$ are differentiable, its derivative belongs to $\calH$, and satisfies the reproducing property: $\partial_i f(x)=\left\langle f, \partial_{1, i} k(x, \cdot)\right\rangle_{\mathcal{H}}$ for any $i\in\{1, \ldots, d\}$.
    Next, we are going to show that $\nabla f\in L_2^d(\mu)$. Too see why, notice that for any $i\in\{1, \ldots, d\}$, 
    \begin{align*}
        \int \left|\partial_i f(x)\right|^2 \dd \mu(x) \leq\|f\|_{\mathcal{H}}^2 \int\left\|\partial_{1, i} k(x, \cdot)\right\|_{\mathcal{H}}^2 \dd \mu(x) = \|f\|_{\mathcal{H}}^2 \int \partial_{1, i} \partial_{2, i} k(x, x) \; \dd\mu(x) <\infty .  
    \end{align*}
    As a result, $\operatorname{ran}(D_\mu)\subseteq L_2^d(\mu)$. 
    Another immediate consequence of the above equation is that the operator $D_\mu$ is bounded. 

    Furthermore, consider another operator
    \begin{align*}
        D_\mu^*: L_2^d(\mu) \rightarrow \mathcal{H}, \quad g = \left(g_1, \ldots, g_d\right) \mapsto \sum_{i=1}^d \int g_i(x) \partial_{1, i} k(x, \cdot) \dd \mu(x) \in \calH. 
    \end{align*}
    To see why $D_\mu^*$ is the adjoint operator of $D_\mu$, one can easily verify that $\langle f, D_\mu^\ast g\rangle_\calH = \langle D_\mu f,  g\rangle_{L_2^d(\mu)}$. 
    This mapping is well-defined: for each $i$, the function $x \mapsto g_i(x) \partial_{1, i} k(x, \cdot)$ is $\mathcal{H}$-valued and Bochner integrable with respect to $\mu$. Indeed, by the Cauchy-Schwarz inequality,
    \begin{align*}
        \left\|\int g_i(x) \partial_{1, i} k(x, \cdot) \dd \mu(x)\right\|_{\mathcal{H}} &\leq \int\left|g_i(x)\right|\left\|\partial_{1, i} k(x, \cdot)\right\|_{\mathcal{H}} \dd \mu(x) \\ &\leq\left\|g_i\right\|_{L_2(\mu)}\left(\int\left\|\partial_{1, i} k(x, \cdot)\right\|_{\mathcal{H}}^2 \dd \mu(x)\right)^{1 / 2} . 
    \end{align*} 
    so $D_\mu^*$ is also a bounded linear operator $L_2^d(\mu) \rightarrow \mathcal{H}$. 
    The proof is thus concluded. 
\end{proof}

\begin{lemma}\label{lem:hilbertschmidt}
Suppose that the kernel $k$ is continuously differentiable with respect to both of its arguments, and that it satisfies $\int \partial_{1,i}\partial_{2,i} k(x, x) \dd \mu(x)< \infty$ for any $i\in [d]$. The operator $S_\mu=D_\mu^\ast D_\mu: \calH\to\calH$ is trace-class, Hilbert-Schmidt, and compact.
\end{lemma}
\begin{proof}
Following \Cref{lem:tc-hs-equivalence}, since $S_\mu=D_\mu^*D_\mu$, it suffices to prove that $D_\mu$ is Hilbert-Schmidt. 
Define $D_{\mu,i}:\calH\to L_2(\mu)$ by $(D_{\mu,i}f)(x)=\partial_{i}f(x)=\langle f,\partial_{2,i}k(x,\cdot)\rangle_{\calH}$ for each $i=1,2,\cdots, d$. 
Then, for any orthonormal basis $\{e_j\}_{j\ge 1}$ of $\calH$, and for any fixed $i$,
\begin{align*}
    \|D_{\mu,i}\|_{\mathrm{HS}}^2=\sum_{j=1}^\infty \|D_{\mu,i}e_j\|_{L_2(\mu)}^2=\sum_{j=1}^\infty \int|\partial_ie_j(x)|^2d\mu(x)=\int\sum_{j=1}^\infty|\langle e_j,\partial_{2,i}k(x,\cdot)\rangle|^2d\mu(x).
\end{align*}
By Parseval's identity, $\sum_{j=1}^\infty|\langle e_j,\partial_{2,i}k(x,\cdot)\rangle|^2=\|\partial_{2,i}k(x,\cdot)\|_{\calH}^2=\partial_{1,i}\partial_{2,i}k(x,x)$. Therefore,
\begin{align*}
    \|D_\mu\|_{\mathrm{HS}}^2=\sum_{i=1}^d \|D_{\mu,i}\|^2_{\mathrm{HS}}=\sum_{i=1}^d\int \partial_{1,i}\partial_{2,i}k(x,x)d\mu(x)<\infty.
\end{align*}
\end{proof}

\begin{lemma}\label{lem:stein_identity}
    Suppose the kernel satisfies the tail condition $\lim_{\|x\|\to\infty} \|x\|^{d-1} \pi(x)\sqrt{k(x,x)}=0$. 
    For any vector-valued function $f:\R^d\to\R^d\in\calH^d$, it satisfies $\int (f(x)^\top \nabla \log \pi(x) + \nabla\cdot f(x)) \; \dd \pi(x) = 0$. 
\end{lemma}
\begin{proof}
Note that
\begin{align*}
    &\quad \int_{\R^d} (f(x)^\top \nabla \log \pi(x) + \nabla\cdot f(x)) \dd \pi(x) = \int_{\R^d} (f(x)^\top \nabla \pi(x) + \nabla\cdot f(x) \pi(x)) \dd x \\
    &= \lim_{R\to\infty} \int_{B_R} \nabla \cdot(\pi(x) f(x)) d x= \lim_{R\to\infty} \int_{\partial B_R} \pi(x) f(x)^\top n(x) \dd \omega(x),
\end{align*}
where $B_R:=\{x:\|x\| \leq R\}$ is a ball of radius $R$, $n(x)$ is the outward unit normal vector, and
$\dd \omega$ is the ($d-1$)-dimensional surface measure. 
Next, notice that $|f(x)^\top n(x)|\leq \sum_{i=1}^d|f_i(x)| \leq \sum_{i=1}^d \|f_i\|_\calH \sqrt{k(x, x)} = \|f\|_{\calH^d} \sqrt{k(x, x)}$. Therefore, 
\begin{align*}
    &\quad \int_{\partial B_R} \pi(x) f(x)^\top n(x) \dd \omega(x) \leq \int_{\partial B_R} \pi(x) \|f\|_{\calH^d} \sqrt{k(x, x)} \dd \omega(x) \\
    &\leq\|f\|_{\calH^d}\left| \int_{\partial B_R} 1 \; \dd \omega(x) \right| \sup _{\|x\|=R} \pi(x) \sqrt{k(x, x)} = \|f\|_{\calH^d} C_d R^{d-1} \sup _{\|x\|=R} \pi(x) \sqrt{k(x, x)} . 
\end{align*}
So we obtain
\begin{align*}
    \left| \int_{\R^d} (f(x)^\top \nabla \log \pi(x) + \nabla\cdot f(x)) \dd \pi(x) \right| \leq \|f\|_{\calH^d} C_d \lim_{\|x\|\to \infty} \|x\|^{d-1} \pi(x) \sqrt{k(x, x)} . 
\end{align*}
By the assumption in the statement of the lemma, the limit is $0$, and the proof is concluded.  
\end{proof}

\begin{lemma}\label{lem:bonnet}
Suppose \Cref{ass:bounded} holds. Let $\calK=B(0,R)$ be a ball of radius $R$ in $\R^d$. 
Let $\mu\in\calP_2(\calK)$. 
Let $\bm{v}_{\mmd}[\mu]$ be the vector field of the MMD flow defined in \Cref{sec:background}.
Then, there exists $l_{\calK}>0$ such that for any $\mu \in \calP_2(\calK)$, we have $\operatorname{Lip}(\bm{v}_{\mmd}[\mu](\cdot); \calK) \leq l_{\calK}$.  Also, there exists $L_{\calK}>0$ such that for any $\mu, \nu \in \calP_2(\calK)$, we have $\sup_{x\in\calK} \left\|\bm{v}_{\mmd}[\mu](x) -\bm{v}_{\mmd}[\nu](x) \right\| \leq \calL_{\calK} W_2(\mu, \nu)$.
The same holds for the vector field of SrMMD flow $\nabla f_{\mu,\pi}$ in \Cref{prop:expressionofsmmd} as well, where $f_{\mu,\pi} = (S_\mu+\lambda \Id)^{-1}(m_\mu-m_\pi)$.  
\end{lemma}
\begin{proof}
Recall from the main text that $\bm{v}_{\mmd}[\mu](\cdot) = \int \nabla_2 k(x, \cdot) \; \dd (\mu - \pi)(x)$. 
Therefore, we have
\begin{align*}
    &\quad \sup_{x_1, x_2\in\calK} \frac{\|\bm{v}_{\mmd}[\mu](x_1) - \bm{v}_{\mmd}[\mu](x_2)\|}{\|x_1 - x_2\|} \\
    &= \sup_{x_1, x_2\in\calK} \frac{\left\| \int (\nabla_2 k(x, x_1) - \nabla_2 k(x, x_2)) \dd (\mu - \pi)(x) \right\|}{\|x_1 - x_2\|} \\
    &\leq \sup_{x_1, x_2\in\calK} \left\| \int k(x, \cdot) \dd (\mu - \pi)(x) \right\|_\calH \cdot \left\| \nabla_1 k(x_1, \cdot) - \nabla_1 k(x_2, \cdot)\right\|_{\calH^d} \cdot \|x_1 - x_2\|^{-1} \\
    &\leq 2 \sup_{x\in\calK} \sqrt{k(x, x)} d B_R:= \ell_\calK . 
\end{align*}
The last inequality holds by \Cref{lem:local_lip}. 
So we have proved the first claim for $\bm{v}_{\mmd}$.  

Next, let $\Gamma$ be an optimal coupling between $\mu$ and $\nu$. Then, for any $x_1\in \calK$,
\begin{align*}
    \bm{v}_{\mmd}[\mu](x_1)-\bm{v}_{\mmd}[\nu](x_1)
    &= \int \nabla_2 k(x,x_1)\,\dd(\mu-\nu)(x) \\
    &= \int \bigl(\nabla_2 k(a,x_1)-\nabla_2 k(b,x_1)\bigr)\,\dd \Gamma(a,b).
\end{align*}
Therefore,
\begin{align*}
    &\quad \sup_{x_1\in\calK} \left\|\bm{v}_{\mmd}[\mu](x_1)-\bm{v}_{\mmd}[\nu](x_1)\right\| \leq \sup_{x_1\in\calK} \int \left\|\nabla_2 k(a,x_1)-\nabla_2 k(b,x_1)\right\| \,\dd\Gamma(a,b) \\
    &\leq \sup_{x_1\in\calK} \sup_{a,b\in\calK}
    \frac{\left\|\nabla_2 k(a,x_1)-\nabla_2 k(b,x_1)\right\|}{\|a-b\|}
    \int \|a-b\|\,\dd\Gamma(a,b) \\
    &\leq d B_R^2 W_2(\mu,\nu) =: \calL_\calK W_2(\mu,\nu) .  
\end{align*}
Now, for each $i=1, \ldots, d$, 
\begin{align*}
    \sup _{x_1 \in  \calK} \sup _{a, b \in  \calK} \frac{\left|\partial_{2, i} k(a, x_1)-\partial_{2, i} k(b, x_1)\right|}{\|a-b\|} \leq \sup _{x_1 \in  \calK} \sup _{c \in  \calK}\left\|\nabla_1 \partial_{2, i} k(c, x_1)\right\| \leq \sqrt{d} B_R^2 ,
\end{align*} 
where we used the mean value theorem, and therefore, 
\begin{align*}
    \sup _{x_1 \in \calK} \sup _{a, b \in \calK} \frac{\left\|\nabla_2 k(a, x_1)-\nabla_2 k(b, x_1)\right\|}{\|a-b\|} \leq d B_R^2 . 
\end{align*} 
So we have proved the second claim for the $\bm{v}_{\mmd}$. 

Next, we are going to prove that both claims hold for the vector field of the SrMMD flow as well. 
\begin{align*}
    &\sup_{x_1, x_2\in\calK} \frac{\|\nabla f_{\mu,\pi}(x_1) - \nabla f_{\mu,\pi}(x_2)\|}{\|x_1 - x_2\|} \leq \|f_{\mu,\pi}\|_\calH \cdot \left\| \nabla_1 k(x_1, \cdot) - \nabla_1 k(x_2, \cdot)\right\|_{\calH^d} \cdot \|x_1 - x_2\|^{-1} \\
    &\leq 2 \lambda^{-1} \sup_{x\in\calK} \sqrt{k(x, x)} B_R . 
\end{align*}
And we also have, 
\begin{align*}
    &\quad \sup_{x_1\in\calK} \left\|\nabla f_{\mu,\pi}(x_1) - \nabla f_{\nu,\pi}(x_1) \right\| \\
    &\leq \sup_{x_1 \in\calK} \left\| \left\langle \nabla k(x_1, \cdot), \; (S_\mu +\lambda)^{-1}(m_\mu-m_\pi) - (S_\nu +\lambda)^{-1}(m_\nu - m_\pi) \right\rangle_{\calH^d} \right\| \\
&\leq \sup_{x_1 \in\calK} \|\nabla k(x_1, \cdot)\|_{\calH^d} \cdot \left\| (S_\mu +\lambda)^{-1}(m_\mu-m_\pi) - (S_\nu +\lambda)^{-1}(m_\nu - m_\pi) \right\|_\calH . 
\end{align*}
Note that
\begin{align*}
    &\quad \left\| (S_\mu +\lambda)^{-1}(m_\mu-m_\pi) - (S_\nu +\lambda)^{-1}(m_\nu - m_\pi) \right\|_\calH \\
    &\leq \left\| (S_\mu +\lambda)^{-1}(m_\mu-m_\nu) \right\|_\calH + \left\| \bigl( (S_\mu +\lambda)^{-1} - (S_\nu +\lambda)^{-1} \bigr)(m_\nu-m_\pi) \right\|_\calH \\
    &\leq \lambda^{-1} \|m_\mu-m_\nu\|_\calH  + \left\| ({S}_{\mu} + \lambda\Id)^{-1}(S_\nu-S_\mu)(S_\nu+\lambda)^{-1}(m_\nu-m_\pi) \right\|_\calH \\
    &\leq \lambda^{-1} \|m_\mu-m_\nu\|_\calH + \lambda^{-2}\|S_\nu-S_\mu\|_{\op}\,\cdot \|m_\nu-m_\pi\|_\calH \\
    &\leq \lambda^{-1} \sqrt{d} B_R W_2(\mu,\nu) + 2 \lambda^{-2} \sup_{x\in\calK}\sqrt{k(x,x)}\, d^{3 / 2} B_R^2 W_2(\mu,\nu) \\
    &= \frac{2 \sup_{x\in\calK}\sqrt{k(x,x)}\, d^{3 / 2} B_R^2 + \lambda \sqrt{d} B_R }{\lambda^2} W_2(\mu, \nu).
\end{align*}
In the above chain of derivations, we use the resolvent identity
$(A+\lambda)^{-1}-(B+\lambda)^{-1}=(A+\lambda)^{-1}(B-A)(B+\lambda)^{-1}$. 
Moreover, if $\Gamma$ is an optimal coupling between $\mu$ and $\nu$, then
\begin{align*}
    \|m_\mu-m_\nu\|_{\calH}\leq \int \|k(x,\cdot)-k(y,\cdot)\|_\calH\,\dd\Gamma(x,y) \leq \int \sqrt{d} B_R \|x-y\|\,\dd\Gamma(x,y) = \sqrt{d} B_R W_2(\mu,\nu).
\end{align*}
The same argument yields $\|S_\mu-S_\nu\|_{\text {op }} \leq 2 d^{3 / 2} B_R^2 W_2(\mu, \nu)$ as well.
Therefore, we have proved both claims for the vector field of SrMMD flow as well. 
\end{proof}

\begin{lemma}\label{lem:local_lip}
    Suppose \Cref{ass:bounded} holds with the constant $B_R$. Let $\calK=B(0,R)$ be a ball of radius $R$ in $\R^d$. Then, $\| k(x,\cdot) - k(y,\cdot)\|_{\calH} \leq \sqrt{d} B_R \|x-y\|$, and $\| \nabla_1 k(x,\cdot) - \nabla_1 k(y,\cdot)\|_{\calH^d} \leq d B_R \|x-y\|$. 
\end{lemma}
\begin{proof}
Fix $x,y\in\calK$, and define $\gamma(t)=y+t(x-y)$ for $t\in[0,1]$. Since $\calK$ is convex, $\gamma(t)\in\calK$ for all $t\in[0,1]$. By the fundamental theorem of calculus for Hilbert-space-valued functions,
\[
k(x,\cdot)-k(y,\cdot)
=
\int_0^1 \frac{\dd}{\dd t} k(\gamma(t),\cdot)\,\dd t
=
\int_0^1 \sum_{i=1}^d (x_i-y_i)\,\partial_{1,i}k(\gamma(t),\cdot)\,\dd t.
\]
Therefore,
\[
\|k(x,\cdot)-k(y,\cdot)\|_\calH
\leq
\int_0^1 \sum_{i=1}^d |x_i-y_i|\,\|\partial_{1,i}k(\gamma(t),\cdot)\|_\calH\,\dd t
\leq
\sqrt{d}\,B_R \|x-y\|.
\]
The same holds for $\| \nabla_1 k(x,\cdot) - \nabla_1 k(y,\cdot)\|_{\calH^d}$ as well. 
\end{proof}

\begin{proposition}\label{prop:mmd_flow_defined}
Suppose \Cref{ass:smmd_growth,ass:bounded} hold. 
Fix a finite time horizon $T \in [0,\infty)$. 
Suppose $\mmd(\mu_0,\pi) < C_{\mmd}$, and $\mu_0,\pi\in\calP_2(\R^d)$. 
Then, $(\mu_t)_{t\geq0}$ is well-defined, in the sense that the trajectory is a unique absolutely continuous curve in $\calP_2(\R^d)$ and does not explode in finite time. 
\end{proposition}
\begin{proof}
To prove the existence of MMD gradient flow, we need to check that all four hypotheses of D1 in \cite{bonnet2021differential} are satisfied. 


Since $\mu_t$ is the MMD flow, there is $\frac{\dd}{\dd t}\mmd^2(\mu_t, \pi) = -\int_{\R^d} \left\|\nabla f_{\mu_t,\pi}(x)\right\|^2 \,\dd\mu_t(x)\le 0$, where $f_{\mu_t,\pi}=m_{\mu_t}-m_\pi$. 
This indicates that $\mmd^2(\mu_t, \pi)$ is monotonically decreasing in time along the trajectory. 
Hence, since $\mmd(\mu_0,\pi) < C_{\mmd}$, there is $\sup_{t\in[0,\infty)} \mmd(\mu_t,\pi) < C_{\mmd} $. 
Let $\calK=B(0,R)$ be a ball of radius $R$ in $\R^d$. Take any sequence $x_n \rightarrow x$ in $\calK$. 
\begin{align*}
    &\left\| \bm{v}_{\mmd}[\mu_t](x_n) - \bm{v}_{\mmd}[\mu_t](x) \right\| = \left\| \int \nabla_2 k(y, x) \; \dd (\mu_t - \pi)(y) - \int \nabla_2 k(y, x_n) \; \dd (\mu_t - \pi)(y) \right\| \\
    &\leq \|m_{\mu_t}-m_\pi\|_\calH \cdot \left\| \nabla_1 k(x, \cdot) - \nabla_1 k(x_n, \cdot) \right\|_{\calH^d} \to 0.
\end{align*}
The last step holds from $\|m_{\mu_t}-m_\pi\|_\calH\le C_{\mmd}$ as proved above, and from \Cref{lem:local_lip}. Hence, D1(i) is satisfied. 

D1(ii) is satisfied by the following derivations:
\begin{align}
    \| \bm{v}_{\mmd}[\mu_t](x) \| &= \left\| \nabla (m_{\mu_t} -m_\pi) (x) \right\| \leq \sum_{j=1}^d \left| \left\langle \partial_{1, j}k(x,\cdot) , \; (m_{\mu_t} -m_\pi) \right\rangle_\calH \right| \nonumber \\
    &\leq \sum_{j=1}^d \| \partial_{1, j}k(x,\cdot)\|_\calH \cdot C_{\mmd}  \leq d C_{\mmd} (\|x\| + 1) \label{eq:bound_v_mmd}.  
\end{align}
D1(iii) and D1(iv) are checked in \Cref{lem:bonnet}. Therefore, it shows that the hypotheses \textnormal{(DI)} of \cite{bonnet2021differential} are satisfied. Hence, by Theorem 4 of \cite{bonnet2021differential}, there exists a trajectory solving the continuity inclusion on any finite time interval $[0,T]$. 

In our MMD flow setting, once existence is established, uniqueness follows separately from the Cauchy--Lipschitz theory for continuity equations, for instance from Theorem 2 of \cite{bonnet2021differential}, since the velocity field satisfies the corresponding growth and local Lipschitz conditions as proved above. 
As the argument applies on every finite interval $[0,T]$, the trajectory cannot explode in finite time. 
\end{proof}

\begin{lemma}[Local bi-Lipschitz regularity of the characteristic flow]\label{lem:local_bilip_flow}
Let $T>0$, and let $b:[0,T]\times \R^d \to \R^d$ be a time-dependent vector field such that:

\begin{enumerate}
    \item it is locally Lipschitz: for every $R>0$, there exists $\ell_R \in L^1(0,T)$ such that for a.e. $t\in[0,T]$ and for all $x,y\in B_R$, $|b_t(x)-b_t(y)| \le \ell_R(t) |x-y|$; 
    \item it has at most linear growth: there exists $m\in L^1(0,T)$ such that for a.e. $t\in[0,T)$ and all $x\in \R^d$,
    $|b_t(x)| \le m(t)(1+\|x\|)$. 
\end{enumerate}
Let $X(t,x)$ be the unique global solution of the characteristic flow: 
\[
\partial_t X(t,x)=b_t(X(t,x)),
\qquad X(0,x)=x.
\]
Then, for every $R_0>0$ and every $t\in[0,T]$, the map $X_t:B_{R_0}\to \R^d, x\mapsto X(t,x)$, is bi-Lipschitz onto its image. More precisely, if
$
R_* := (1+R_0) \exp(\int_0^T m(s)\,\dd s)-1$, 
then for all $x,y\in B_{R_0}$ and all $t\in[0,T]$,
\[
\exp\Bigl(-\int_0^t \ell_{R_*}(s)\,\dd s\Bigr)|x-y|
\le |X_t(x)-X_t(y)|
\le \exp\Bigl(\int_0^t \ell_{R_*}(s)\,\dd s\Bigr)|x-y|.
\]
In particular, $X_t$ is locally bi-Lipschitz on bounded sets.
\end{lemma}

\begin{proof}
Fix $R_0>0$ and $x\in B_{R_0}$. By the linear growth assumption,
\[
\|X(t,x)\|
\le \|x\| + \int_0^t m(s)\bigl(1+\|X(s,x)\|\bigr)\,\dd s.
\]
Hence
\[
1+\|X(t,x)\|
\le 1+\|x\|+\int_0^t m(s)\bigl(1+\|X(s,x)\|\bigr)\,\dd s.
\]
By Gr\"onwall's inequality,
\[
1+\|X(t,x)\|
\le (1+\|x\|)\exp\Bigl(\int_0^t m(s)\,\dd s\Bigr)
\le (1+R_0)\exp\Bigl(\int_0^T m(s)\,\dd s\Bigr).
\]
Therefore,
\[
X(s,x)\in B_{R_*}
\qquad \text{for all } s\in[0,T], \ x\in B_{R_0}.
\]

Now fix $x,y\in B_{R_0}$. Since both trajectories remain in $B_{R_*}$, we have
\[
\|X_t(x)-X_t(y)\|
\le \|x-y\|+\int_0^t \ell_{R_*}(s)\,\|X_\mathfrak{s}(x)-X_\mathfrak{s}(y)\|\,\dd s.
\]
Another application of Gr\"onwall's inequality yields
\[
\|X_t(x)-X_t(y)\|
\le \exp\Bigl(\int_0^t \ell_{R_*}(s)\,\dd s\Bigr)\|x-y\|.
\]
For the reverse bound, fix $0\le s\le t$. Using the integral formulation of the ODE,
\[
X_\mathfrak{s}(x)-X_\mathfrak{s}(y)
=
X_t(x)-X_t(y)-\int_s^t \bigl(b_r(X_r(x))-b_r(X_r(y))\bigr)\,\dd r.
\]
Thus
\[
\|X_\mathfrak{s}(x)-X_\mathfrak{s}(y)\|
\le \|X_t(x)-X_t(y)\|
+\int_s^t \ell_{R_*}(r)\,\|X_r(x)-X_r(y)\|\,\dd r.
\]
Applying Gr\"onwall's inequality on the interval $[s,t]$, we obtain
\[
\|X_\mathfrak{s}(x)-X_\mathfrak{s}(y)\|
\le \exp\Bigl(\int_s^t \ell_{R_*}(r)\,\dd r\Bigr) \|X_t(x)-X_t(y)\|.
\]
Taking $s=0$ gives
\[
\|x-y\|
\le \exp\Bigl(\int_0^t \ell_{R_*}(r)\,\dd r\Bigr)\|X_t(x)-X_t(y)\|.
\]
Equivalently,
\[
\|X_t(x)-X_t(y)\|
\ge \exp\Bigl(-\int_0^t \ell_{R_*}(r)\,\dd r\Bigr)\|x-y\|.
\]
Combining the two bounds proves the claim.
\end{proof}

\begin{lemma}
\label{lem:differentiation_along_characteristics}
Let $(X_t)_{t\in[0,T]}$ be an absolutely continuous stochastic process in $\R^d$. 
Let $h\in C^1(\R^d)$ satisfy
$\E\bigl[ |h(X_0)| \bigr] < \infty$ and
$\E\left[ \int_0^T \left|
\nabla h(X_t)^\top \dot X_t
\right| \dd t \right] < \infty$.
Then the map $t\mapsto \E[h(X_t)]$ is absolutely continuous on $[0,T]$, and
$\frac{\dd}{\dd t}\E[h(X_t)] = \E\left[
\nabla h(X_t)^\top \dot X_t
\right]$ for a.e. $t\in[0,T]$. 
\end{lemma}

\begin{proof}
Since $X_t$ is absolutely continuous and $h\in C^1(\R^d)$, the pathwise chain rule gives
\[
    h(X_t) - h(X_0)
    =
    \int_0^t
        \nabla h(X_s)^\top \dot X_s
    \dd s
\]
for every $t\in[0,T]$, almost surely. Taking expectations yields
\[
    \E[h(X_t)] - \E[h(X_0)]
    =
    \E\left[
        \int_0^t
            \nabla h(X_s)^\top \dot X_s
        \dd s
    \right].
\]
By the assumed integrability condition, Fubini's theorem applies, and therefore
\[
    \E[h(X_t)] - \E[h(X_0)]
    =
    \int_0^t
        \E\left[
            \nabla h(X_s)^\top \dot X_s
        \right]
    \dd s.
\]
Hence $t\mapsto \E[h(X_t)]$ is absolutely continuous. Differentiating the last
display gives
$\frac{\dd}{\dd t}\E[h(X_t)] = \E\left[
\nabla h(X_t)^\top \dot X_t
\right]$ for a.e. $t\in[0,T]$. 
\end{proof}

\begin{proposition}\label{prop:suff_stein_kernel}
Let the base kernel $k$ be a translation invariant kernel $k(x,y)=\phi(\|x-y\|)$. 
Suppose $\phi:\R\to\R$ is a four-times differentiable even function with locally bounded derivatives. 
Suppose there exists $M_0>0$ such that $\|\mathfrak{s}(x)\| \leq M_0 (1 + \|x\|)$, $\| \bJ \mathfrak{s}(x)\|_F \leq M_0 (1 + \|x\|)$ and that $\mathfrak{s}$ has locally bounded second order derivatives. 
Then, the Langevin Stein kernel $k_\pi$ satisfies both \Cref{ass:smmd_growth,ass:bounded}. 
If additionally, $\mathfrak{s}$ is Lipschitz and has a uniformly bounded Hessian, then the Langevin Stein kernel $k_\pi$ also satisfies \Cref{ass:growth_2}. 
\end{proposition}
\begin{proof}
Recall that the Stein kernel $k_\pi$ is
\begin{align*}
    k_{\pi}(x,y) := \mathfrak{s}(x)^\top \mathfrak{s}(y) k(x,y) + \mathfrak{s}(x)^\top \nabla_2 k(x,y) + \nabla_1 k(x,y)^\top \mathfrak{s}(y) + \nabla \cdot_1 \nabla_2 k(x,y) . 
\end{align*}
Then, 
\begin{align*}
    |k_{\pi}(x,x)| &\leq |\phi(0)| \|\mathfrak{s}(x)\|^2 + 2 \|\mathfrak{s}(x)\| |\phi'(0)| + d |\phi''(0)| \\
    &\leq \big( |\phi(0)| M_0^2 + 2 M_0 |\phi'(0)| + d |\phi''(0)| \big)  (1 + \|x\|^2) . 
\end{align*}
This proves that $k_\pi$ satisfies the first growth condition in \Cref{ass:smmd_growth}. 

Denote $\bJ \mathfrak{s}(x) \in \R^{d \times d}$ as the Jacobian matrix of the score function. 
From the derivation in Appendix B.1 of \cite{korba2021kernel}, we have
\begin{align*}
\begin{aligned}
&\quad  \nabla_{1}\nabla_2 k_\pi(x, y) = \left( \bJ \mathfrak{s}(x)^{\top} \mathfrak{s}(y)\right) \nabla_2 k(x, y)^{\top} + \mathfrak{s}(x)^{\top} \mathfrak{s}(y) \nabla_{1}\nabla_2 k(x, y) + \bJ \mathfrak{s}(y)^{\top} \bJ \mathfrak{s}(x) k(x, y) \\ &+ \nabla_1 k(x, y) \mathfrak{s}(x)^{\top} \bJ \mathfrak{s}(y) + \bH_2 k(x, y) \bJ \mathfrak{s}(x) + \nabla_{1} \bH_2 k(x, y) \mathfrak{s}(x) + \nabla_{2} \bH_1 k(x, y) \mathfrak{s}(y)  \\
&\qquad + \bJ \mathfrak{s}(y)^{\top} \bH_1 k(x, y) + \nabla_{1}\nabla_2 \operatorname{Tr}\left(\nabla_2 \nabla_1 k(x, y)\right) . 
\end{aligned}
\end{align*}
Since $\phi$ is an even function, its odd derivatives at $0$ equal to $0$. So we have
\begin{align*}
    \nabla_{1}\nabla_2 k_\pi(x, x) &= \|\mathfrak{s}(x)\|^2 \nabla_{1}\nabla_2 k(x, x) + \bJ \mathfrak{s}(x)^\top \bJ \mathfrak{s}(x) k(x, x) + 2 \bH_2 k(x, x) \bJ \mathfrak{s}(x) \\
    &+ \nabla_{1}\nabla_2 \operatorname{Tr}\left(\nabla_2 \nabla_1 k(x, x) \right) \\
    &= -\|\mathfrak{s}(x)\|^2 \phi''(0)\Id + \bJ \mathfrak{s}(x)^\top \bJ \mathfrak{s}(x) \phi(0) + 2 \phi''(0) \bJ \mathfrak{s}(x) + 4(d+2) \phi^{(4)}(0) \Id. 
\end{align*}
Therefore, we obtain
\begin{align*}
    \|\nabla_{1}\nabla_2 k_\pi(x, x)\|_F &\leq d |\phi''(0)| \|\mathfrak{s}(x)\|^2 + |\phi(0)| \|\bJ \mathfrak{s}(x)\|_F^2 + 2 |\phi''(0)| \| \bJ \mathfrak{s}(x)\|_F + 4(d+2) |\phi^{(4)}(0) | \\
    &\leq (|\phi''(0)| M_0^2 + |\phi(0)| M_0^2 + 2 |\phi''(0)| M_0 + 4(d+2) |\phi^{(4)}(0) |)(1 + \|x\|^2) .
\end{align*}
This proves that $k_\pi$ satisfies the second growth condition in \Cref{ass:smmd_growth}. 
Next, since $\mathfrak{s}, \bJ \mathfrak{s}, \nabla^2 \mathfrak{s}$ and $\phi, \phi', \phi'', \phi^{(4)}$ are all locally bounded, \Cref{ass:bounded} is also satisfied. 

The second claim is proved in Lemma 1 of \cite{korba2021kernel}, specifically, in Eq.~(40)-(42). 
\end{proof}

{

\section{Convergence Analysis of the Particle Implementation of SrMMD}\label{apd:finite-particle}\label{app:finite-particle}

For any $x\in\R^d$, we define the pointwise derivative operator $\calD_x:\calH\to\R^d$, $f\mapsto\nabla f(x)$, whose adjoint is $\calD_x^*u=\sum_{i=1}^d u_i\partial_{1,i}k(x,\cdot)$. The derivative covariance operator $S_\mu$ can be written as $S_\mu=\int\calD_x^*\calD_x\dd\mu(x)$. Throughout this section, $\mu_0,\pi\in\calP_2(\R^d)$ and $\lambda>0$. Under~\Cref{ass:particle-conditions}, the Newton--Leibniz formula gives $\|\partial_{1,i}k(x,\cdot)-\partial_{1,i}k(y,\cdot)\|_\calH\le B\sum_{j=1}^d|x_j-y_j|\le\sqrt d B\|x-y\|$. Summing over $i$ yields $\|\calD_x-\calD_y\|_{\mathrm{HS}}\le dB\|x-y\|$. Thus $\|\calD_x\|_{\mathrm{HS}}\le\|\calD_0\|_{\mathrm{HS}}+dB\|x\|$, and integration from $0$ to $x$ gives $\|k(x,\cdot)\|_\calH\le\|k(0,\cdot)\|_\calH+\|\calD_0\|_{\mathrm{op}}\|x\|+dB\|x\|^2/2$. Since $\|\calD_x^*\calD_x\|_{\mathrm{HS}}\le\|\calD_x\|_{\mathrm{HS}}^2$, the quadratic growth of the kernel features and derivative covariance ensures that $m_\mu$ and $S_\mu$ are well-defined for $\mu\in\calP_2(\R^d)$. The linear growth of the derivative evaluation operator implies at most linear growth of $\nabla f_{\mu,\pi}$, so each population update belongs to $\calP_2(\R^d)$.

First, we show that the uniform fourth-moment bound and the global
derivative bound in~\Cref{ass:particle-conditions} imply the required
moment bounds for the kernel and its mixed second derivatives.

\begin{lemma}\label{lem:particle-moment-bounds}
Under~\Cref{ass:particle-conditions}, there exists a constant $\sigma>0$ such that, for all $s\ge0$,
\begin{align*}
    \int\|k(x,\cdot)-m_{\mu_s}\|_{\calH}^2\,\dd\mu_s(x)&\le\sigma^2,\quad 
    \sum_{i,j=1}^d\int|\partial_{1,i}\partial_{2,j}k(x,x)|^2\,\dd\mu_s(x)\le\sigma^2.
\end{align*}
\end{lemma}

\begin{proof}
Set $R_4=\sup_{s\ge0}\int\|x\|^4\,\dd\mu_s(x)<\infty$ and $A=\|\calD_0\|_{\mathrm{HS}}$.
By~\Cref{assm:global-boundedness} and the Newton--Leibniz formula, $\|\calD_x\|_{\mathrm{HS}}\le A+dB\|x\|$ and $\|k(x,\cdot)-k(0,\cdot)\|_{\calH}\le A\|x\|+dB\|x\|^2/2$.
Cauchy--Schwarz inequality therefore give
    $\int\|k(x,\cdot)-m_{\mu_s}\|_{\calH}^2\,\dd\mu_s(x)
    \le\int\|k(x,\cdot)-k(0,\cdot)\|_{\calH}^2\,\dd\mu_s(x)\le 2A^2R_4^{1/2}+\frac{d^2B^2}{2}R_4$.
By the reproducing property and $\|\calD_x^*\calD_x\|_{\mathrm{HS}}\le\|\calD_x\|_{\mathrm{HS}}^2$,
$ \sum_{i,j=1}^d|\partial_{1,i}\partial_{2,j}k(x,x)|^2
    =\|\calD_x^*\calD_x\|_{\mathrm{HS}}^2
    \le\|\calD_x\|_{\mathrm{HS}}^4
    \le 8A^4+8d^4B^4\|x\|^4,$
where the last inequality uses $(a+b)^4\le8(a^4+b^4)$ for $a,b\ge0$.
Integrating this bound and choosing a common $\sigma^2\ge1$ proves the claim.
\end{proof}

Let $\sigma>0$ be the constant in~\Cref{lem:particle-moment-bounds}, and choose the constant $C$ in~\Cref{thm:particle-convergence} to be at least $e\vee\sigma^2$.

\paragraph{Step 1: Bounding the one-step propagation error and proving monotonicity.}
First, we control the one-step error propagation of SrMMD descent. 

\begin{proposition}\label{prop:one-step-propagation}
    Suppose~\Cref{ass:particle-conditions} holds, and let $\mu\in\calP_2(\R^d)$ satisfy $\mmd(\mu,\pi)\le\mmd(\mu_0,\pi)+1$. Let $0<t\le1/2$ and $4dBt(\mmd(\mu_0,\pi)+1)\le\lambda$. Set $\rho_t=(\Id-t\nabla f_{\mu,\pi})_\#\mu$. Then we have 
    \begin{align*}
        \mmd^2(\rho_t,\pi)\le (1-t)\mmd^2(\mu,\pi)+t\lambda\langle m_\mu-m_\pi,f_{\mu,\pi}\rangle_\calH.
    \end{align*}
\end{proposition}

\begin{proof}
    Since $S_\mu$ is positive semidefinite and $0\le u/(u+\lambda)\le1$ for $u\ge0$, spectral calculus gives $\|S_\mu f_{\mu,\pi}\|_\calH^2=\langle S_\mu (S_\mu+\lambda \Id)^{-1}(m_\mu-m_\pi), S_\mu f_{\mu,\pi}\rangle_\calH\le \langle m_\mu-m_\pi, S_\mu f_{\mu,\pi}\rangle_\calH$ and 
    \begin{equation}\label{eq:velocity-bound-1}
    \|\nabla f_{\mu,\pi}\|_{L_2^d(\mu)}^2=\langle (S_\mu+\lambda\Id)^{-1}(m_\mu-m_\pi),S_\mu f_{\mu,\pi} \rangle_\calH\le \lambda^{-1}\langle m_\mu-m_\pi, S_\mu f_{\mu,\pi}\rangle_\calH.\end{equation} Write $m_\mu-m_\pi=\sum_{i=0}^\infty m_{i} e_i$, where $\{\varrho_i,e_i\}_{i\in\mathbb N}$ is an orthonormal eigensystem of $S_\mu$.  Then
    \begin{equation}\label{eq:velocity-bound-2}
    \begin{aligned}
        \|\nabla f_{\mu,\pi}\|_{L_2^d(\mu)}^2&=\langle (S_\mu+\lambda\Id)^{-1}(m_\mu-m_\pi),S_\mu(S_\mu+\lambda\Id)^{-1}(m_\mu-m_\pi)\rangle_\calH\\
        &=\sum_{i=0}^\infty \frac{\varrho_i}{(\varrho_i+\lambda)^2}m_i^2\stackrel{\text{(a)}}\le \frac1{4\lambda}\sum_{i=0}^\infty m_i^2\\
        &=\frac1{4\lambda}\mmd^2(\mu,\pi)\le \frac{1}{4\lambda}(\mmd(\mu_0,\pi)+1)^2,
    \end{aligned}
    \end{equation}
    where (a) uses $u/(u+\lambda)^2=(u+\frac{\lambda^2}{u}+2\lambda)^{-1}\le(4\lambda)^{-1}$ at positive eigenvalues. At zero the ratio in Eq.~\eqref{eq:velocity-bound-2} is zero, so the inequality holds for every $u\ge0$.
By Taylor's formula with integral remainder and the definition of push-forward,
\begin{align*}
    &\quad m_{\rho_t}-m_\pi\\
    &=\int k(x,\cdot)\dd (\Id-t\nabla f_{\mu,\pi})_\# \mu(x)-m_\pi\\
    &=\int k(x-t\nabla f_{\mu,\pi}(x),\cdot)\dd\mu(x)-m_\pi\\
    &=\int \Bigg(k(x,\cdot)-t\sum_{i=1}^d \partial_i f_{\mu,\pi}(x)\partial_{1,i}k(x,\cdot)\\
    &\quad +t^2\sum_{i,j=1}^d\partial_i f_{\mu,\pi}(x)\partial_j f_{\mu,\pi}(x)\int_0^1(1-u)\partial_{1,i}\partial_{1,j}k(x-ut\nabla f_{\mu,\pi}(x),\cdot)\dd u
    \Bigg)\dd\mu(x)-m_\pi\\
    &=m_\mu-m_\pi-t\sum_{i=1}^d\int\partial_i f_{\mu,\pi}(x)\partial_{1,i}k(x,\cdot)\dd\mu(x)\\
    &\quad
    +t^2\sum_{i,j=1}^d\int\partial_i f_{\mu,\pi}(x)\partial_j f_{\mu,\pi}(x)\int_0^1(1-u)\partial_{1,i}\partial_{1,j}k(x-ut\nabla f_{\mu,\pi}(x),\cdot)\dd u
    \dd\mu(x)\\
    &=m_\mu-m_\pi-tS_\mu f_{\mu,\pi}\\
    &\quad+t^2\sum_{i,j=1}^d\int\partial_i f_{\mu,\pi}(x)\partial_j f_{\mu,\pi}(x)\int_0^1(1-u)\partial_{1,i}\partial_{1,j}k(x-ut\nabla f_{\mu,\pi}(x),\cdot)\dd u\dd\mu(x),
\end{align*}
where the last equality follows from the reproducing property and the
definition of $S_\mu$, since
\begin{align*}
    \sum_{i=1}^d
    \int\partial_i f_{\mu,\pi}(x)\partial_{1,i}k(x,\cdot)\dd\mu(x)
    =\sum_{i=1}^d\int\langle f_{\mu,\pi},\partial_{1,i}k(x,\cdot)\rangle_\calH\partial_{1,i}k(x,\cdot)\dd\mu(x)
    =S_\mu f_{\mu,\pi}.
\end{align*}
Therefore, by~\Cref{assm:global-boundedness},
\begin{align*}
    &\quad
    \|m_{\rho_t}-m_\mu+tS_\mu f_{\mu,\pi}\|_\calH\\
    &\le
    t^2\sum_{i,j=1}^d\int|\partial_i f_{\mu,\pi}(x)|
    |\partial_j f_{\mu,\pi}(x)|\int_0^1(1-u) \|\partial_{1,i}\partial_{1,j} k(x-ut\nabla f_{\mu,\pi}(x),\cdot)\|_\calH
    \dd u\dd\mu(x)\\
    &\le\frac{Bt^2}{2}
    \int\sum_{i,j=1}^d|\partial_i f_{\mu,\pi}(x)||\partial_j f_{\mu,\pi}(x)|
    \dd\mu(x)=\frac{Bt^2}{2}\int\pth{\sum_{i=1}^d|\partial_i f_{\mu,\pi}(x)|}^2
    \dd\mu(x)\\
    &\le
    \frac{dBt^2}{2}\int\|\nabla f_{\mu,\pi}(x)\|^2\dd\mu(x)
    =\frac{dBt^2}{2}\|\nabla f_{\mu,\pi}\|_{L_2^d(\mu)}^2.
\end{align*}
The last inequality follows from Cauchy--Schwarz. Eq.~\eqref{eq:velocity-bound-2} ensures the square integrability of $\nabla f_{\mu,\pi}$ and hence the integrability of the Taylor remainder. The linear term is integrable because $\int\|\calD_x\|_{\mathrm{op}}^2\dd\mu(x)<\infty$ and Cauchy--Schwarz applies to $\|\calD_x\|_{\mathrm{op}}\|\nabla f_{\mu,\pi}(x)\|$. Thus Taylor's formula can be integrated as an identity in $\calH$. Moreover,
    \begin{align*}
        &\quad \|m_\mu-m_\pi-tS_\mu f_{\mu,\pi}\|_\calH^2\\
        &=\|m_\mu-m_\pi\|_\calH^2+\langle S_\mu f_{\mu,\pi}, -2t (m_\mu-m_\pi)+t^2 S_\mu (S_\mu+\lambda \Id)^{-1} (m_\mu-m_\pi)\rangle_\calH\\
        &\le \mmd^2(\mu,\pi)+\langle S_\mu f_{\mu,\pi}, -2t (m_\mu-m_\pi)+t^2 (m_\mu-m_\pi)\rangle_\calH\\
        &=\mmd^2(\mu,\pi)-t(2-t)\langle S_\mu f_{\mu,\pi},  m_\mu-m_\pi\rangle_\calH.
    \end{align*}
    In particular, $\|m_\mu-m_\pi-tS_\mu f_{\mu,\pi}\|_\calH\le\mmd(\mu,\pi)\le\mmd(\mu_0,\pi)+1$. Adding the integral remainder to the linear Taylor polynomial, expanding the squared norm, and applying Cauchy--Schwarz to the cross term gives, with $\varepsilon_0=\mmd(\mu_0,\pi)$, 
    \begin{align*}
        &\mmd^2(\rho_t,\pi)\\
        \le&~ \mmd^2(\mu,\pi)-t(2-t)\langle m_\mu-m_\pi, S_\mu f_{\mu,\pi}\rangle_\calH\\
        &~~+dBt^2(\varepsilon_0+1)\|\nabla f_{\mu,\pi}\|_{L_2^d(\mu)}^2+\frac{(dB)^2t^4}{4}\|\nabla f_{\mu,\pi}\|_{L_2^d(\mu)}^4\\
        \le&~ \mmd^2(\mu,\pi)-t(2-t-\frac{dBt(\varepsilon_0+1)}{\lambda}-\frac{d^2B^2t^3(\varepsilon_0+1)^2}{16\lambda^2})\langle m_\mu-m_\pi, S_\mu f_{\mu,\pi}\rangle_\calH\\
        \le&~ \mmd^2(\mu,\pi)-t\langle m_\mu-m_\pi, S_\mu f_{\mu,\pi}\rangle_\calH=(1-t)\mmd^2(\mu,\pi)+t\lambda \langle m_\mu-m_\pi, f_{\mu,\pi}\rangle_\calH.
    \end{align*}
    For the second inequality, Eq.~\eqref{eq:velocity-bound-1} controls the squared velocity norm, and the product of Eqs.~\eqref{eq:velocity-bound-1} and~\eqref{eq:velocity-bound-2} controls its square. For the third inequality, the factor in parentheses is at least $2-\frac12-\frac14-\frac1{512}>1$ by the assumptions on $t$ and $\lambda$. The last equality follows from $S_\mu f_{\mu,\pi}=m_\mu-m_\pi-\lambda f_{\mu,\pi}$. Since $\langle m_\mu-m_\pi,S_\mu f_{\mu,\pi}\rangle_\calH\ge0$, the final inequality also gives $\mmd(\rho_t,\pi)\le\mmd(\mu,\pi)$.
\end{proof}

\begin{lemma}\label{lem:particle-population-decay}
    Under the assumptions of~\Cref{thm:particle-convergence}, $\mmd(\mu_s,\pi)\le\mmd(\mu_0,\pi)$ for every $s\ge0$, and, for $0\le s\le S$,
    \begin{align*}
        \mmd(\mu_s,\pi)\le\pth{\mmd^2(\mu_0,\pi)+\lambda^\alpha Q}^{1/2}\pth{1+\frac{\gamma\max\{s-1,0\}}{\alpha}}^{-\alpha/2}.
    \end{align*}
\end{lemma}

\begin{proof}
Applying~\Cref{prop:one-step-propagation} with $t=\gamma$ proves monotonicity by induction: $\mu_{s+1}=(\Id-\gamma\nabla f_{\mu_s,\pi})_\#\mu_s$, and bounding $\mmd(\mu_s,\pi)$ by its initial value verifies the proposition's hypothesis at each iteration. If $\mmd(\mu_s,\pi)=0$ at some iteration, then $f_{\mu_s,\pi}=0$ and every subsequent population measure equals $\mu_s$, which gives the lemma's conclusion at those iterations. For iterations with a positive residual, we use the source condition as follows.

Recall that $m_{\mu_s}-m_\pi=\sum_{i=0}^\infty m_{i,s}e_{i,s}$, where $\{\varrho_{i,s},e_{i,s}\}_{i\in\mathbb N}$ is an orthonormal eigensystem of $S_{\mu_s}$. For $1\le s\le S$, the source condition $m_{\mu_s}-m_\pi=S_{\mu_s}^{\alpha/2}q_s$ implies that $m_{i,s}=0$ whenever $\varrho_{i,s}=0$, and $\sum_{\varrho_{i,s}>0}\varrho_{i,s}^{-\alpha}m_{i,s}^2\le\|q_s\|_\calH^2\le Q$.  We obtain
\begin{align*}
    \mmd^2(\mu_s,\pi)&=\sum_{\varrho_{i,s}>0}m_{i,s}^2=\sum_{\varrho_{i,s}>0}\pth{\frac{\varrho_{i,s}}{\varrho_{i,s}+\lambda}m_{i,s}^2}^{\frac\alpha{1+\alpha}}\pth{\pth{\frac{\varrho_{i,s}+\lambda}{\varrho_{i,s}}}^\alpha m_{i,s}^2}^{\frac1{1+\alpha}}\\
    &\le\pth{\sum_{\varrho_{i,s}>0}\frac{\varrho_{i,s}}{\varrho_{i,s}+\lambda}m_{i,s}^2}^{\frac\alpha{1+\alpha}}\pth{\sum_{\varrho_{i,s}>0}\pth{\frac{\varrho_{i,s}+\lambda}{\varrho_{i,s}}}^\alpha m_{i,s}^2}^{\frac1{1+\alpha}}\\
    &\le\pth{\sum_{\varrho_{i,s}>0}\frac{\varrho_{i,s}}{\varrho_{i,s}+\lambda}m_{i,s}^2}^{\frac\alpha{1+\alpha}}\pth{\sum_{\varrho_{i,s}>0}(\varrho_{i,s}^\alpha+\lambda^\alpha)\varrho_{i,s}^{-\alpha}m_{i,s}^2}^{\frac1{1+\alpha}}\\
    &\le\langle m_{\mu_s}-m_\pi,S_{\mu_s}f_{\mu_s,\pi}\rangle_\calH^{\frac\alpha{1+\alpha}}\pth{\mmd^2(\mu_s,\pi)+\lambda^\alpha Q}^{\frac1{1+\alpha}}.
\end{align*}
The first inequality is H\"older's inequality with exponents $(1+\alpha)/\alpha$ and $1+\alpha$. The second uses $(u+\lambda)^\alpha\le u^\alpha+\lambda^\alpha$ for $u\ge0$ and $0<\alpha\le1$. The final inequality follows from the source condition in~\Cref{thm:particle-convergence}. Taking $t=\gamma$ and $\mu=\mu_s$ in~\Cref{prop:one-step-propagation}, solving the H\"older inequality for the inner product, and replacing the population MMD in its denominator by its initial value gives, for $1\le s<S$,
\begin{align*}
    \mmd^2(\mu_{s+1},\pi)\le\mmd^2(\mu_s,\pi)-\gamma\frac{\mmd^{2+2/\alpha}(\mu_s,\pi)}{\pth{\mmd^2(\mu_0,\pi)+\lambda^\alpha Q}^{1/\alpha}}.
\end{align*}
We may use $(1-u)^{-1/\alpha}\ge1+u/\alpha$ for $0\le u<1$, or equivalently $1-u\le(1+u/\alpha)^{-\alpha}$, to obtain
\begin{align*}
    \mmd^2(\mu_{s+1},\pi)&\le\mmd^2(\mu_s,\pi)\pth{1-\gamma\frac{\mmd^{2/\alpha}(\mu_s,\pi)}{\pth{\mmd^2(\mu_0,\pi)+\lambda^\alpha Q}^{1/\alpha}}}\\
    &\le\mmd^2(\mu_s,\pi)\pth{1+\gamma\frac{\mmd^{2/\alpha}(\mu_s,\pi)}{\alpha\pth{\mmd^2(\mu_0,\pi)+\lambda^\alpha Q}^{1/\alpha}}}^{-\alpha}.
\end{align*}
Whenever $\mmd(\mu_{s+1},\pi)>0$, the last inequality is equivalent to
\begin{align}\label{eq:particle-iteration}
    \mmd^{-2/\alpha}(\mu_{s+1},\pi)\ge\mmd^{-2/\alpha}(\mu_s,\pi)+\frac{\gamma}{\alpha(\mmd^2(\mu_0,\pi)+\lambda^\alpha Q)^{1/\alpha}}.
\end{align}
We repeat the recursion Eq.~\eqref{eq:particle-iteration}  from $1$ to $s-1$. Monotonicity gives $\mmd^{-2/\alpha}(\mu_1,\pi)\ge(\mmd^2(\mu_0,\pi)+\lambda^\alpha Q)^{-1/\alpha}$, which supplies the initial term of this sum and proves the result for $1\le s\le S$.
\end{proof}

\paragraph{Step 2: Establishing sampling bounds via auxiliary particles.}
Define the auxiliary particles by 
\begin{align*}
    \overline x_{s+1}^{(i)}=\overline x_s^{(i)}-\gamma\nabla f_{\mu_s,\pi}(\overline x_s^{(i)}),\quad \overline x_0^{(i)}=x_0^{(i)}.
\end{align*}
Define the associated empirical distribution $\overline\mu_s=N^{-1}\sum_{i=1}^N\delta_{\overline x_s^{(i)}}$. For each fixed iteration $s$, the auxiliary particles are independent with law $\mu_s$, since they are obtained from the independent initial samples by the same deterministic population updates. Therefore, by~\Cref{lem:particle-moment-bounds},
\begin{align*}
    \expect{\|m_{\overline\mu_s}-m_{\mu_s}\|_\calH^2}\vee\expect{\|S_{\overline\mu_s}-S_{\mu_s}\|_{\mathrm{HS}}^2}\le\frac{\sigma^2}{N}.
\end{align*}
Indeed, expanding the squared norm of either centered sample mean eliminates the cross terms by independence and gives $1/N$ times the centered second moment of one summand. For the operator term, the reproducing property gives $\|\calD_x^*\calD_x\|_{\mathrm{HS}}^2=\sum_{i,j=1}^d|\partial_{1,i}\partial_{2,j}k(x,x)|^2$, and its centered second moment is $\int\|\calD_x^*\calD_x-S_{\mu_s}\|_{\mathrm{HS}}^2\dd\mu_s(x)=\int\|\calD_x^*\calD_x\|_{\mathrm{HS}}^2\dd\mu_s(x)-\|S_{\mu_s}\|_{\mathrm{HS}}^2\le\sigma^2$.

Fix $0<\delta<1$, and $K\in\mathbb N$ to be specified in Step 4, and set $\xi=\sqrt{2(K+1)/(N\delta)}$. For each $0\le s\le K$, Chebyshev's inequality gives $\prob{\|m_{\overline\mu_s}-m_{\mu_s}\|_\calH>\sigma\xi}\le1/(N\xi^2)$ and $\prob{\|S_{\overline\mu_s}-S_{\mu_s}\|_{\mathrm{HS}}>\sigma\xi}\le1/(N\xi^2)$. The union bound gives
\begin{align}\label{eq:concentration-event}
    &\prob{\bigcup_{s=0}^K\{\|m_{\overline\mu_s}-m_{\mu_s}\|_\calH>\sigma\xi\}\,\cup\,\bigcup_{s=0}^K\{\|S_{\overline\mu_s}-S_{\mu_s}\|_{\mathrm{HS}}>\sigma\xi\}}\le\frac{2(K+1)}{N\xi^2}=\delta.
\end{align}
On the complement of the union in Eq.~\eqref{eq:concentration-event}, which has probability at least $1-\delta$, both $\|m_{\overline\mu_s}-m_{\mu_s}\|_\calH\le\sigma\xi$ and $\|S_{\overline\mu_s}-S_{\mu_s}\|_{\mathrm{op}}\le\sigma\xi$ hold for every $0\le s\le K$; the second inequality uses $\|\cdot\|_{\mathrm{op}}\le\|\cdot\|_{\mathrm{HS}}$. We compare the two particle systems on this complement.

For the comparison up to the chosen iteration, assume $\xi\le1$ and $\sigma\xi\le(K+1)^{-\alpha/2}/4$. Step 4 relates $K$ and $N$ so that these two inequalities hold. Since $\hat\mu_0=\overline\mu_0$, Eq.~\eqref{eq:concentration-event} and the triangle inequality give $\mmd(\hat\mu_0,\pi)\le\mmd(\mu_0,\pi)+1$ on the complement of the union. Applying~\Cref{prop:one-step-propagation} with $t=\gamma$ to successive empirical measures shows by induction that $\mmd(\hat\mu_s,\pi)$ is nonincreasing for all $s\ge0$. Moreover, Jensen's inequality and~\Cref{lem:particle-moment-bounds} give $\|S_{\mu_s}\|_{\mathrm{op}}\le\|S_{\mu_s}\|_{\mathrm{HS}}\le\sigma$, so $\xi\le1$ implies $\|S_{\overline\mu_s}\|_{\mathrm{op}}\le2\sigma$ for $s\le K$ on the complement of the union in Eq.~\eqref{eq:concentration-event}.

\paragraph{Step 3: Establishing the error recursion.}
Let $\mathsf e_s=\pth{\frac1N\sum_{i=1}^N \|x_s^{(i)}-\overline x_s^{(i)}\|^2}^{1/2}$. For a fixed iteration $s$, consider the interpolated particles defined by $x_t^{(i)}=(1-t)\overline x_s^{(i)}+t x_s^{(i)}$ and $\nu_t=\frac1N\sum_{i=1}^N \delta_{x_t^{(i)}}$, where $0\le t\le 1$. The induction hypothesis for $0\le s\le S\wedge K$ is $\mathsf e_s\le 1\wedge\frac{(K+1)^{-\alpha/2}}{4(\sqrt{2\sigma}+dB/2)}$.
Taylor's formula at each $\overline x_s^{(i)}$, using the integral remainder in the proof of~\Cref{prop:one-step-propagation}, yields
\begin{align*}
    \|m_{\nu_t}-m_{\overline \mu_s}\|_\calH&=\frac1N\left\|\sum_{i=1}^N\pth{k(\overline x_s^{(i)}+t(x_s^{(i)}-\overline x_s^{(i)}),\cdot)-k(\overline x_s^{(i)},\cdot)}\right\|_\calH\\
    &\le \left\|\frac tN \sum_{i=1}^N \calD_{\overline x_s^{(i)}}^*(x_s^{(i)}-\overline x_s^{(i)})\right\|_\calH +\frac{dB}2 t^2\mathsf e_s^2\\
    &\le \sqrt{2\sigma} t \mathsf e_s+\frac{dB}2 t^2 \mathsf e_s^2\le \frac{(K+1)^{-\alpha/2}}{4}.
\end{align*}
 The linear term is estimated by duality and Cauchy--Schwarz in the particle index, using the operator norm of the empirical derivative covariance from Eq.~\eqref{eq:concentration-event} and~\Cref{lem:particle-moment-bounds}. The remainder is bounded by $dBt^2\mathsf e_s^2/2$ by~\Cref{assm:global-boundedness}. The last inequality uses $t\le1$, $\mathsf e_s\le1$, and the induction hypothesis $\mathsf e_s\le 1\wedge\frac{(K+1)^{-\alpha/2}}{4(\sqrt{2\sigma}+dB/2)}$. The triangle inequality and Eq.~\eqref{eq:concentration-event}, on the complement of its union, then give $\mmd(\nu_t,\pi)\le\mmd(\mu_s,\pi)+(K+1)^{-\alpha/2}/2$ for every $0\le t\le1$.

\begin{lemma}
    Under~\Cref{ass:particle-conditions}, along the straight-line interpolation between the auxiliary and interacting particles, for $0\le\gamma\le1$, we have 
\begin{align*}
\pth{\frac1N\sum_{i=1}^N \left\|\frac\dd{\dd t}\sth{x_t^{(i)}-\gamma \nabla f_{\nu_t,\pi}(x_t^{(i)})}\right\|^2}^{1/2} \le \pth{1+\frac{5\gamma dB}{4\lambda}\mmd(\nu_t,\pi)}\mathsf e_s.
\end{align*}
\end{lemma}
\begin{proof}
Direct differentiation gives $\frac\dd{\dd t} f_{\nu_t,\pi}=(S_{\nu_t}+\lambda \Id)^{-1} (\frac\dd{\dd t}m_{\nu_t}-\frac\dd{\dd t} S_{\nu_t} f_{\nu_t,\pi})$, where 
$$\frac\dd{\dd t}m_{\nu_t}=\frac1N \sum_{j=1}^N \sum_{j=1}^N \calD_{x_t^{(j)}}(x_s^{(j)}-\overline x_s^{(j)}), \quad \frac\dd{\dd t}S_{\nu_t}=\frac1N\sum_{j=1}^N \pth{\frac{\dd \calD_{x_t^{(j)}}^*}{\dd t}\calD_{x_t^{(j)}}+\calD_{x_t^{(j)}}\frac{\dd \calD_{x_t^{(j)}}}{\dd t}}. $$
The chain rule then gives
\begin{align*}
    &\quad \frac{\dd}{\dd t}\sth{x_t^{(i)}-\gamma\nabla f_{\nu_t,\pi}(x_t^{(i)})}\\
    &=x_s^{(i)}-\overline x_s^{(i)}-\frac\gamma N\sum_{j=1}^N\calD_{x_t^{(i)}}(S_{\nu_t}+\lambda\Id)^{-1}\calD_{x_t^{(j)}}^*(x_s^{(j)}-\overline x_s^{(j)})\\
    &\quad-\gamma\frac{\dd\calD_{x_t^{(i)}}}{\dd t}f_{\nu_t,\pi}+\frac\gamma N\sum_{j=1}^N\calD_{x_t^{(i)}}(S_{\nu_t}+\lambda\Id)^{-1}\calD_{x_t^{(j)}}^*\frac{\dd\calD_{x_t^{(j)}}}{\dd t}f_{\nu_t,\pi}\\
    &\quad+\frac\gamma N\sum_{j=1}^N\calD_{x_t^{(i)}}(S_{\nu_t}+\lambda\Id)^{-1}\frac{\dd\calD_{x_t^{(j)}}^*}{\dd t}\calD_{x_t^{(j)}}f_{\nu_t,\pi}.
\end{align*}
By Minkowski's inequality,
\begin{align*}
    &\quad\sth{\frac1N\sum_{i=1}^N\left\|\frac{\dd}{\dd t}\sth{x_t^{(i)}-\gamma\nabla f_{\nu_t,\pi}(x_t^{(i)})}\right\|^2}^{1/2}\\
    &\le\underbrace{\sth{\frac1N\sum_{i=1}^N\left\|x_s^{(i)}-\overline x_s^{(i)}-\frac\gamma N\sum_{j=1}^N\calD_{x_t^{(i)}}(S_{\nu_t}+\lambda\Id)^{-1}\calD_{x_t^{(j)}}^*(x_s^{(j)}-\overline x_s^{(j)})\right\|^2}^{1/2}}_{\textnormal{(I)}}\\
    &\quad+\underbrace{\gamma\sth{\frac1N\sum_{i=1}^N\left\|-\frac{\dd\calD_{x_t^{(i)}}}{\dd t}f_{\nu_t,\pi}+\frac1N\sum_{j=1}^N\calD_{x_t^{(i)}}(S_{\nu_t}+\lambda\Id)^{-1}\calD_{x_t^{(j)}}^*\frac{\dd\calD_{x_t^{(j)}}}{\dd t}f_{\nu_t,\pi}\right\|^2}^{1/2}}_{\textnormal{(II)}}\\
    &\quad+\underbrace{\gamma\sth{\frac1N\sum_{i=1}^N\left\|\calD_{x_t^{(i)}}(S_{\nu_t}+\lambda\Id)^{-1}\frac1N\sum_{j=1}^N\frac{\dd\calD_{x_t^{(j)}}^*}{\dd t}\calD_{x_t^{(j)}}f_{\nu_t,\pi}\right\|^2}^{1/2}}_{\textnormal{(III)}}.
\end{align*}

We first estimate \textnormal{(I)}. With the operator $D_X$ introduced in Appendix~A.5 at $x^{(i)}=x_t^{(i)}$, the $Nd\times Nd$ block matrix whose $(i,j)$ block is $N^{-1}\calD_{x_t^{(i)}}(S_{\nu_t}+\lambda\Id)^{-1}\calD_{x_t^{(j)}}^*$ equals $D_X(D_X^*D_X+\lambda\Id)^{-1}D_X^*=D_XD_X^*(D_XD_X^*+\lambda\Id)^{-1}$. Since $D_XD_X^*$ is positive semidefinite, the eigenvalues of this matrix are $u/(u+\lambda)\in[0,1]$ for $u\ge0$. As $0\le\gamma\le1$, the matrix $\Id-\gamma D_X(D_X^*D_X+\lambda\Id)^{-1}D_X^*$ also has eigenvalues in $[0,1]$ and hence operator norm at most one. Therefore, $\textnormal{(I)}\le\mathsf e_s$.

For \textnormal{(II)}, Item~1 of Assumption~\ref{ass:particle-conditions} and the reproducing property give $\|\frac{\dd\calD_{x_t^{(i)}}}{\dd t}\|_{\mathrm{op}}\le dB\|x_s^{(i)}-\overline x_s^{(i)}\|$. Indeed, differentiating $\calD_{x_t^{(i)}}f=\nabla f(x_t^{(i)})$ along $t\mapsto x_t^{(i)}$ gives a sum involving $\partial_{1,k}\partial_{1,l}k(x_t^{(i)},\cdot)$, whose $\calH$-norm is bounded by $B$. Cauchy--Schwarz bounds each component by $\sqrt d B\|x_s^{(i)}-\overline x_s^{(i)}\|\|f\|_\calH$; summing its square over the $d$ components gives the factor $dB$. The two terms in \textnormal{(II)} are obtained by applying $-\gamma[\Id-D_X(D_X^*D_X+\lambda\Id)^{-1}D_X^*]$ to the vector with $i$-th component $(\frac{\dd\calD_{x_t^{(i)}}}{\dd t})f_{\nu_t,\pi}$. The operator norm of the matrix in brackets is at most one. Moreover, $\|f_{\nu_t,\pi}\|_\calH\le\mmd(\nu_t,\pi)/\lambda$. Consequently,
\begin{align*}
    \textnormal{(II)}\le\gamma dB\mathsf e_s\|f_{\nu_t,\pi}\|_\calH\le\frac{\gamma dB}{\lambda}\mathsf e_s\mmd(\nu_t,\pi).
\end{align*}

For \textnormal{(III)}, the definition of $S_{\nu_t}$ and $\|S_{\nu_t}^{1/2}(S_{\nu_t}+\lambda\Id)^{-1}\|_{\mathrm{op}}\le1/(2\sqrt\lambda)$ give
\begin{align*}
    \textnormal{(III)}
    &\le\frac{\gamma}{2\sqrt\lambda}\left\|\frac1N\sum_{j=1}^N\frac{\dd\calD_{x_t^{(j)}}^*}{\dd t}\calD_{x_t^{(j)}}f_{\nu_t,\pi}\right\|_\calH.
\end{align*}
The operator norm estimate follows by applying spectral calculus to $u/(u+\lambda)^2\le1/(4\lambda)$ for $u\ge0$. Applying~\Cref{assm:global-boundedness} to the differentiated adjoints and then using Cauchy--Schwarz gives
\begin{align*}
    &\left\|\frac1N\sum_{j=1}^N\frac{\dd\calD_{x_t^{(j)}}^*}{\dd t}\calD_{x_t^{(j)}}f_{\nu_t,\pi}\right\|_\calH\le\frac{dB}{N}\sum_{j=1}^N\|x_s^{(j)}-\overline x_s^{(j)}\|\|\calD_{x_t^{(j)}}f_{\nu_t,\pi}\|\\
    &\le dB\mathsf e_s\|S_{\nu_t}^{1/2}f_{\nu_t,\pi}\|_\calH
    \le\frac{dB}{2\sqrt\lambda}\mathsf e_s\mmd(\nu_t,\pi).
\end{align*}
The last inequality uses $f_{\nu_t,\pi}=(S_{\nu_t}+\lambda\Id)^{-1}(m_{\nu_t}-m_\pi)$ and the spectral calculation in Eq.~\eqref{eq:velocity-bound-2}, applied to the interpolated empirical measure. Multiplying by the outer resolvent factor in \textnormal{(III)} gives
\begin{align*}
    \textnormal{(III)}\le\frac{\gamma dB}{4\lambda}\mathsf e_s\mmd(\nu_t,\pi).
\end{align*}

Combining the estimates for \textnormal{(I)}, \textnormal{(II)}, and \textnormal{(III)}, we obtain
\begin{equation}\label{eq:particle-stability}
    \sth{\frac1N\sum_{i=1}^N\left\|\frac{\dd}{\dd t}\sth{x_t^{(i)}-\gamma\nabla f_{\nu_t,\pi}(x_t^{(i)})}\right\|^2}^{1/2}
    \le\sth{1+\frac{5\gamma dB}{4\lambda}\mmd(\nu_t,\pi)}\mathsf e_s.
\end{equation}
\end{proof}

\begin{proof}[Proof of~\Cref{thm:particle-convergence}]
Fix a deterministic integer $K$ in Eq.~\eqref{eq:concentration-event}. On the complement of the union in that equation, we first compare the velocities generated by the two witness functions at the auxiliary particles. By the definition of the witness functions,
\begin{align*}
    (S_{\bar\mu_s}+\lambda\Id)(f_{\bar\mu_s,\pi}-f_{\mu_s,\pi})&=(m_{\bar\mu_s}-m_\pi)-(S_{\bar\mu_s}+\lambda\Id)f_{\mu_s,\pi}\\
    &=m_{\bar\mu_s}-m_{\mu_s}+(S_{\mu_s}-S_{\bar\mu_s})f_{\mu_s,\pi},
\end{align*}
where the last equality follows from $(S_{\mu_s}+\lambda\Id)f_{\mu_s,\pi}=m_{\mu_s}-m_\pi$. Consequently,
\begin{align*}
    f_{\bar\mu_s,\pi}-f_{\mu_s,\pi}
    =(S_{\bar\mu_s}+\lambda\Id)^{-1}
    \big(m_{\bar\mu_s}-m_{\mu_s}+(S_{\mu_s}-S_{\bar\mu_s})f_{\mu_s,\pi}\big).
\end{align*}

By the definition of $S_{\bar\mu_s}$, the root mean squared difference between the two witness velocities at the auxiliary particles satisfies
\begin{align*}
    \sth{\frac1N\sum_{i=1}^N\|\nabla f_{\bar\mu_s,\pi}(\bar x_s^{(i)})-\nabla f_{\mu_s,\pi}(\bar x_s^{(i)})\|^2}^{1/2}&=\|S_{\bar\mu_s}^{1/2}(f_{\bar\mu_s,\pi}-f_{\mu_s,\pi})\|_\calH\\
    &\le\frac{\|m_{\bar\mu_s}-m_{\mu_s}\|_\calH+\|S_{\bar\mu_s}-S_{\mu_s}\|_{\mathrm{op}}\|f_{\mu_s,\pi}\|_\calH}{2\sqrt\lambda}\\
    &\le\frac{\sigma\xi}{2\sqrt\lambda}\sth{1+\frac{\mmd(\mu_0,\pi)}{\lambda}}.
\end{align*}
The first equality follows from $\nabla f(x)=\calD_xf$ and $S_{\bar\mu_s}=N^{-1}\sum_{i=1}^N\calD_{\bar x_s^{(i)}}^*\calD_{\bar x_s^{(i)}}$. The first inequality uses $\|S_{\bar\mu_s}^{1/2}(S_{\bar\mu_s}+\lambda\Id)^{-1}\|_{\mathrm{op}}\le 1/(2\sqrt\lambda)$, which follows from $u/(u+\lambda)^2\le 1/(4\lambda)$ for $u\ge0$. The last inequality follows from Eq.~\eqref{eq:concentration-event} on the complement of its union and from $\|f_{\mu_s,\pi}\|_\calH\le\mmd(\mu_s,\pi)/\lambda\le\mmd(\mu_0,\pi)/\lambda$, using~\Cref{lem:particle-population-decay} in the final step.

To apply Eq.~\eqref{eq:particle-stability}, we separate the difference between the interacting particle updates from the difference between the witness velocities evaluated at the auxiliary particles. Since $x_0^{(i)}=\bar x_0^{(i)}$, we have $\mathsf e_0=0$. The update definitions give
\begin{align*}
    &\quad x_{s+1}^{(i)}-\bar x_{s+1}^{(i)}\\
    &=x_s^{(i)}-\gamma\nabla f_{\hat\mu_s,\pi}(x_s^{(i)})
    -\bar x_s^{(i)}+\gamma\nabla f_{\mu_s,\pi}(\bar x_s^{(i)})\\
    &=\big(x_s^{(i)}-\gamma\nabla f_{\hat\mu_s,\pi}(x_s^{(i)})\big)
    -\big(\bar x_s^{(i)}-\gamma\nabla f_{\bar\mu_s,\pi}(\bar x_s^{(i)})\big)+\gamma\big(\nabla f_{\mu_s,\pi}(\bar x_s^{(i)})-\nabla f_{\bar\mu_s,\pi}(\bar x_s^{(i)})\big).
\end{align*}
Since $\nu_0=\bar\mu_s$ and $\nu_1=\hat\mu_s$, Newton--Leibniz formula gives
\begin{align*}
    &\quad \big(x_s^{(i)}-\gamma\nabla f_{\hat\mu_s,\pi}(x_s^{(i)})\big)
    -\big(\bar x_s^{(i)}-\gamma\nabla f_{\bar\mu_s,\pi}(\bar x_s^{(i)})\big)=\int_0^1\frac{\dd}{\dd t}\sth{x_t^{(i)}-\gamma\nabla f_{\nu_t,\pi}(x_t^{(i)})}\dd t.
\end{align*}
Taking the root mean squared norm, applying Minkowski's inequality, and then using Eq.~\eqref{eq:particle-stability}, we obtain
\begin{align*}
    \mathsf e_{s+1}
    &\le\int_0^1\sth{\frac1N\sum_{i=1}^N\left\|\frac{\dd}{\dd t}\sth{x_t^{(i)}-\gamma\nabla f_{\nu_t,\pi}(x_t^{(i)})}\right\|^2}^{1/2}\dd t\\
    &\quad+\gamma\sth{\frac1N\sum_{i=1}^N\|\nabla f_{\bar\mu_s,\pi}(\bar x_s^{(i)})-\nabla f_{\mu_s,\pi}(\bar x_s^{(i)})\|^2}^{1/2}\\
    &\le\sth{1+\frac{5\gamma dB}{4\lambda}\sup_{0\le t\le1}\mmd(\nu_t,\pi)}\mathsf e_s
    +\frac{\gamma\sigma\xi}{2\sqrt\lambda}\sth{1+\frac{\mmd(\mu_0,\pi)}{\lambda}}.
\end{align*}
For the interpolated measure, we estimate $\mmd(\nu_t,\pi)$ using its distance to the auxiliary empirical measure. Since $\nu_t=N^{-1}\sum_{i=1}^N\delta_{x_t^{(i)}}$ and $x_t^{(i)}=\bar x_s^{(i)}+t(x_s^{(i)}-\bar x_s^{(i)})$, the triangle inequality gives
\begin{align*}
    \mmd(\nu_t,\pi)
    &\le\mmd(\mu_s,\pi)+\|m_{\bar\mu_s}-m_{\mu_s}\|_\calH+\|m_{\nu_t}-m_{\bar\mu_s}\|_\calH.
\end{align*}
On the complement of the union in Eq.~\eqref{eq:concentration-event}, $\|m_{\bar\mu_s}-m_{\mu_s}\|_\calH\le\sigma\xi\le(K+1)^{-\alpha/2}/4$. Taylor's formula along the particle segments gives $\|m_{\nu_t}-m_{\bar\mu_s}\|_\calH\le(K+1)^{-\alpha/2}/4$ whenever $\mathsf e_s\le1\wedge[(K+1)^{-\alpha/2}/(4(\sqrt{2\sigma}+dB/2))]$, by the linear-term and remainder calculations at the start of Step 3. Adding these two explicit terms in the triangle inequality gives $\mmd(\nu_t,\pi)\le\mmd(\mu_s,\pi)+(K+1)^{-\alpha/2}/2$ for all $0\le t\le1$.

Integrating Eq.~\eqref{eq:particle-stability} and substituting the explicit estimates for the MMD of the interpolated measure and for the witness-velocity difference gives, for $s<\min\{S,K\}$,
\begin{equation}\label{eq:particle-error-recursion}
    \begin{aligned}
        \mathsf e_{s+1}
        &\le\left[1+\frac{5\gamma dB}{4\lambda}\sth{\mmd(\mu_s,\pi)+\frac{(K+1)^{-\alpha/2}}2}\right]\mathsf e_s+\frac{\gamma\sigma\xi}{2\sqrt\lambda}\sth{1+\frac{\mmd(\mu_0,\pi)}{\lambda}}.
    \end{aligned}
\end{equation}

\paragraph{Step 4: Iterating the error recursion and choosing $K$.}
Throughout Step 4, $C\ge1$ depends only on the parameters listed in~\Cref{thm:particle-convergence} and may be increased from one estimate to the next. By~\Cref{lem:particle-population-decay}, $\mmd(\mu_s,\pi)\le C(1+s)^{-\alpha/2}$ for $0\le s\le S$. Comparing the decreasing function $(1+t)^{-\alpha/2}$ with its integral gives
\begin{align*}
    \sum_{s=0}^{\min\{S,K\}-1}\mmd(\mu_s,\pi)\le C\pth{1+\int_0^K(1+t)^{-\alpha/2}\dd t}\le C(K+1)^{1-\alpha/2}.
\end{align*}
The last inequality uses $1-\alpha/2\ge1/2$. We prove the induction hypothesis stated at the start of Step 3. The common initialization gives $\mathsf e_0=0$. Assume that this hypothesis holds through iteration $j-1$, where $1\le j\le\min\{S,K\}$. Repeated substitution in Eq.~\eqref{eq:particle-error-recursion}, followed by $1+u\le\exp(u)$ for $u\ge0$, gives
\begin{align*}
    \mathsf e_j&\le\frac{\gamma\sigma\xi}{2\sqrt\lambda}\pth{1+\frac{\mmd(\mu_0,\pi)}{\lambda}}\sum_{i=0}^{j-1}\prod_{l=i+1}^{j-1}\sth{1+\frac{5\gamma dB}{4\lambda}\pth{\mmd(\mu_l,\pi)+\frac{(K+1)^{-\alpha/2}}2}}\\
    &\le\frac{\gamma j\sigma\xi}{2\sqrt\lambda}\pth{1+\frac{\mmd(\mu_0,\pi)}{\lambda}}\exp\sth{\frac{5\gamma dB}{4\lambda}\sum_{l=0}^{j-1}\pth{\mmd(\mu_l,\pi)+\frac{(K+1)^{-\alpha/2}}2}}\\
    &\le\frac{C(K+1)^{3/2}}{\sqrt{N\delta}}\exp\sth{C(K+1)^{1-\alpha/2}}.
\end{align*}
An empty product equals one. The final inequality follows by summing the inequality in~\Cref{lem:particle-population-decay} over the iterations and using $j\le K$, $j(K+1)^{-\alpha/2}\le(K+1)^{1-\alpha/2}$, and $\xi=\sqrt{2(K+1)/(N\delta)}$. The induction hypothesis in Step 3 has right-hand side $1\wedge[(K+1)^{-\alpha/2}/(4(\sqrt{2\sigma}+dB/2))]$, which is at least $(K+1)^{-\alpha/2}/\max\{1,4(\sqrt{2\sigma}+dB/2)\}$. Squaring the expression obtained by iterating Eq.~\eqref{eq:particle-error-recursion} and comparing it with the square of this explicit right-hand side gives, after increasing $C$, the sufficient condition
\begin{align*}
    N\delta\ge C(K+1)^{3+\alpha}\exp\sth{C(K+1)^{1-\alpha/2}}
\end{align*}
for $\mathsf e_j$ to satisfy the induction hypothesis in Step 3. The power $3+\alpha$ comes from squaring the expression for $\mathsf e_j$ obtained from Eq.~\eqref{eq:particle-error-recursion} and dividing by a constant multiple of $(K+1)^{-\alpha}$; the resulting exponential coefficient and fixed prefactor are included in $C$. This particle number requirement also ensures $\xi\le1$ and $\sigma\xi\le(K+1)^{-\alpha/2}/4$, since their squared forms are $N\delta\ge2(K+1)$ and $N\delta\ge32\sigma^2(K+1)^{1+\alpha}$, respectively. The induction starts at $\mathsf e_0=0$ and proves the hypothesis in Step 3 for every $0\le s\le\min\{S,K\}$.

At $t=1$, the interpolation satisfies $\nu_1=\hat\mu_s$. Taylor's formula at the auxiliary particles, the triangle inequality, and Eq.~\eqref{eq:concentration-event} on the complement of its union give $\mmd(\hat\mu_s,\mu_s)\le(K+1)^{-\alpha/2}/2$ for $0\le s\le\min\{S,K\}$. If $S\le K$, this covers every required iteration. If $S>K$, then for $K<s\le S$, applying~\Cref{prop:one-step-propagation} to the empirical and population sequences gives
\begin{align*}
    \mmd(\hat\mu_s,\mu_s)&\le\mmd(\hat\mu_s,\pi)+\mmd(\mu_s,\pi)\\
    &\le\mmd(\hat\mu_K,\pi)+\mmd(\mu_K,\pi)\\
    &\le\mmd(\hat\mu_K,\mu_K)+2\mmd(\mu_K,\pi)\le C(K+1)^{-\alpha/2}.
\end{align*}
The last inequality uses the estimate for the MMD between the particle and population measures at iteration $K$, together with~\Cref{lem:particle-population-decay}. Hence the particle number requirement on $N\delta$ obtained from Eq.~\eqref{eq:particle-error-recursion} gives $\max_{0\le s\le S}\mmd(\hat\mu_s,\mu_s)\le C(K+1)^{-\alpha/2}$ with probability at least $1-\delta$. The same constant applies to every final iteration $S$ with a fixed source bound $Q$.

We choose $K$ as follows: whenever $\log(N\delta)\ge2C$,
\begin{align*}
    K=\left\lfloor\pth{\frac{\log(N\delta)}{2C}}^{2/(2-\alpha)}\right\rfloor.
\end{align*}

Since $\log(N\delta)\ge2C$, we have $K\ge1$ and $K+1\le2(\log(N\delta)/(2C))^{2/(2-\alpha)}$. Raising the latter inequality to the power $1-\alpha/2$ and using $2^{1-\alpha/2}\le2$, we obtain $C(K+1)^{1-\alpha/2}\le\log(N\delta)$. The particle number requirement derived from Eq.~\eqref{eq:particle-error-recursion} therefore holds for this $K$. Moreover, $K+1\ge(\log(N\delta)/(2C))^{2/(2-\alpha)}$. Therefore,
\begin{align*}
    \max_{0\le s\le S}\mmd(\hat\mu_s,\mu_s)\le C(K+1)^{-\alpha/2}\le C\pth{\frac{\log(N\delta)}{2C}}^{-\alpha/(2-\alpha)}.
\end{align*}
Absorbing the fixed factor $(2C)^{\alpha/(2-\alpha)}$ into the constant in~\Cref{thm:particle-convergence} gives the stated rate. The definition of $K$ uses the constant selected before separating the parameter ranges at $2C$. Because $K$ is determined by $N,\delta$ and the fixed problem parameters, Eq.~\eqref{eq:concentration-event} applies for every final iteration $S$ with success probability at least $1-\delta$.

For $1\le\log(N\delta)<2C$, the lower endpoint follows from $N\delta\ge e$. Since $N\delta\ge\sigma^2$, \Cref{lem:particle-moment-bounds} and Chebyshev's inequality give $\prob{\|m_{\hat\mu_0}-m_{\mu_0}\|_\calH>1}\le\sigma^2/N\le\delta$. On the complementary event, $\mmd(\hat\mu_0,\pi)\le\mmd(\mu_0,\pi)+1$, so~\Cref{prop:one-step-propagation} gives empirical monotonicity. The triangle inequality and~\Cref{lem:particle-population-decay} yield $\max_{0\le s\le S}\mmd(\hat\mu_s,\mu_s)\le2\mmd(\mu_0,\pi)+1$. Since $[\log(N\delta)]^{-\alpha/(2-\alpha)}\ge(2C)^{-\alpha/(2-\alpha)}$, increasing the theorem's constant proves its conclusion in this range. Combining the two parameter ranges gives the conclusion with probability at least $1-\delta$.

\end{proof}

}
\begin{figure}[t]
    \centering

    \hspace{-8mm}
    \begin{minipage}[t]{0.49\linewidth}
        \centering
        \includegraphics[height=3.0cm,keepaspectratio]{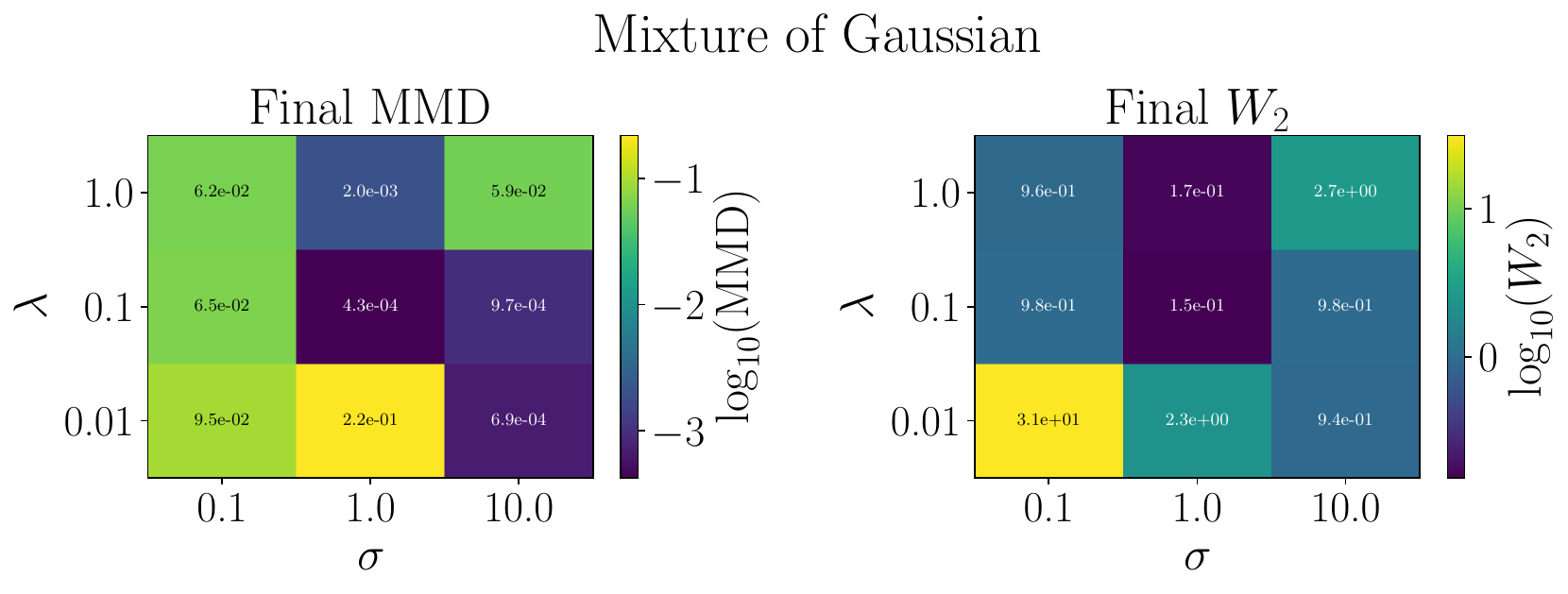}
    \end{minipage}
    \hspace{8mm}
    \begin{minipage}[t]{0.25\linewidth}
        \centering
        \includegraphics[height=3.0cm,keepaspectratio,trim=30 0 0 0,clip]{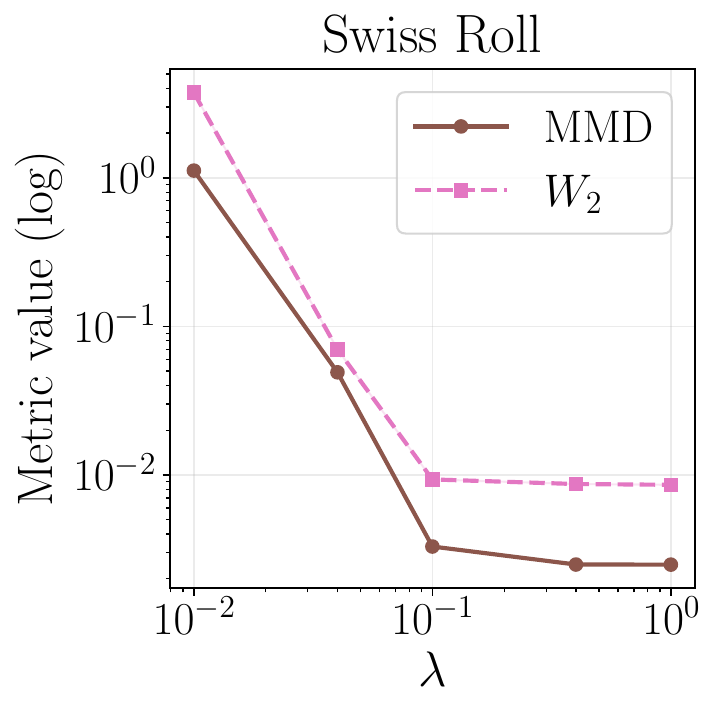}
    \end{minipage}
    \hspace{-8mm}
    \begin{minipage}[t]{0.25\linewidth}
        \centering
        \includegraphics[height=3.0cm,keepaspectratio,trim=30 0 0 0,clip]{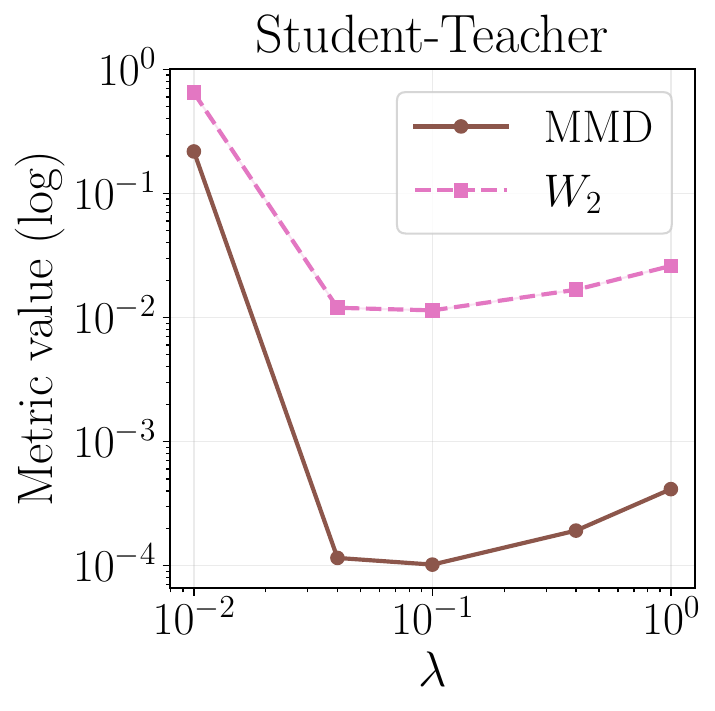}
    \end{minipage}
    \hspace{-8mm}
    \caption{Ablation studies of SrMMD flow on Mixture of Gaussians, Swiss roll, and student-teacher network, with respect to kernel bandwidth $\sigma$ and regularization parameter $\lambda$.}
    \label{fig:ablation_mog}
\end{figure}

\section{Additional Experimental Details}\label{sec:add_exp}
In this section, we provide additional details for reproducing the experiments in \Cref{sec:experiments}, together with supplementary experimental results. 
The experiments in \Cref{sec:experiments} were run on a workstation equipped with an AMD Ryzen 9 5950X 16-Core Processor and an NVIDIA GeForce RTX 4090 GPU. 
\subsection{Toy examples}\label{sec:more_toy}
The mixture of Gaussian consists of four  Gaussian distributions with mean $(\pm2, \pm2)$ and the same covariance matrix $1.2 \Id$. 
The swiss roll distribution is taken from the \textsc{sklearn} package. 
For the gaussian mixture setting, the kernel mean embedding under a Gaussian kernel $k(x,y)=\exp(-(2\sigma^2)^{-1} \|x-y\|^2)$ admits a close-form expression~\citep{chen2024conditional}. Specifically, for a Gaussian distribution $\nu=\calN({m}, {\Sigma})$: 
\begin{align*}
    m_\nu(x)  = \det\left(\Id + \sigma^{-2} {\Sigma}\right)^{-1 / 2} \exp \left(-\frac{1}{2}(x-{m})^{\top}\left({\Sigma}+\sigma^2 \Id \right)^{-1}(x-{m})\right) . 
\end{align*}
In contrast, the kernel mean embedding for the Swiss roll is approximated as an empirical average over $500$ i.i.d samples from the target. 

In the left panel of \Cref{fig:ablation_mog}, we report ablation studies on the kernel bandwidth $\sigma$ and the regularization parameter $\lambda$ for the SrMMD flow on the Gaussian mixture benchmark. We also report an ablation study on $\lambda$ for the Swiss roll benchmark, where the kernel is chosen as $k(x,x')=-\|x-x'\|$ and therefore does not involve a lengthscale parameter. 
For the mixture of Gaussians, the configuration achieving the smallest MMD and $W_2$ distance is $(\sigma,\lambda)=(1.0,0.1)$. For the Swiss roll benchmark, the best-performing value is $\lambda=0.1$. The hyperparameters of the MMD flow and DrMMD flow baselines are selected analogously by searching over the candidate sets $\{0.01, 0.03, 0.1, 0.3, 1.0, 3.0\}$ for $\sigma$ and $\{0.01, 0.03, 0.1, 0.3, 1.0, 3.0, 10.0\}$ for $\lambda$. We omit the corresponding ablation plots, as these baselines are not the primary focus of this study.

\subsection{Computational cost and low-rank approximations}\label{sec:low_rank}
We further evaluate the computational cost of SrMMD flow and the effect of low-rank approximations on the Gaussian mixture benchmark.
We use $N=M=500$ source and target particles in $d=2$ dimensions, with Gaussian kernel length scale $\sigma=1.0$, regularization parameter $\lambda=0.1$, and step size $\gamma=0.1$.
We compare exact SrMMD flow with Nystr\"om approximations~\citep{williams2000using} of ranks $r\in\{32,64,128,256\}$ and random Fourier features (RFF)~\citep{rahimi2007random} with $R\in\{32,64,128,256\}$ features, together with the vanilla MMD flow and DrMMD flow.
All methods are run for $200$ iterations over $10$ random seeds, giving $110$ runs in total.
Performance is evaluated using MMD computed with the exact Gaussian kernel for all methods.

\vspace{1mm}
\textbf{Low-rank approximations.}\hspace{4mm}
In \Cref{tab:low_rank}, we report the final MMD, time per iteration, and relative error of the final local velocity for each approximation.
The time per iteration measures synchronized steady-state optimization time on the GPU after JAX compilation.
The local velocity error compares the approximate velocity with the exact SrMMD velocity evaluated at the same particle configuration.
Results are reported as mean $\pm$ standard deviation over $10$ random seeds.

\begin{table}[!htbp]
    \centering
    \caption{Ablation studies of Nystr\"om and RFF approximations for SrMMD flow on the Gaussian mixture benchmark. Final MMD is computed with the exact kernel after $200$ iterations. Results are reported as mean $\pm$ standard deviation over $10$ random seeds.}
    \label{tab:low_rank}
    \small
    \setlength{\tabcolsep}{4pt}
    \begin{tabular}{@{}lcrrr@{}}
        \toprule
        Method & Rank/features & Time/iteration (s) & Final MMD & Relative velocity error \\
        \midrule
        Exact SrMMD & -- & $0.003454\pm0.000017$ & $0.01093\pm0.00087$ & $0$ \\
        Nystr\"om & $32$ & $0.002453\pm0.000060$ & $0.01837\pm0.01410$ & $0.3464\pm0.4350$ \\
        Nystr\"om & $64$ & $0.001578\pm0.000005$ & $0.01098\pm0.00085$ & $0.01165\pm0.00813$ \\
        Nystr\"om & $128$ & $0.004154\pm0.000009$ & $0.01094\pm0.00087$ & $(1.06\pm0.58)\times10^{-4}$ \\
        Nystr\"om & $256$ & $0.006333\pm0.000019$ & $0.01093\pm0.00087$ & $(2.05\pm1.51)\times10^{-6}$ \\
        \midrule
        RFF & $32$ & $0.000372\pm0.000002$ & $0.09479\pm0.05341$ & $0.9508\pm0.0298$ \\
        RFF & $64$ & $0.000410\pm0.000001$ & $0.05839\pm0.01465$ & $0.9204\pm0.0466$ \\
        RFF & $128$ & $0.000442\pm0.000002$ & $0.03344\pm0.00533$ & $0.8562\pm0.0502$ \\
        RFF & $256$ & $0.000992\pm0.000014$ & $0.02169\pm0.00258$ & $0.7541\pm0.0599$ \\
        \bottomrule
    \end{tabular}
\end{table}

From \Cref{tab:low_rank}, the rank-$64$ Nystr\"om approximation achieves a final MMD comparable to exact SrMMD flow, while reducing the time per iteration by a factor of $2.19$.
The rank-$32$ approximation also reduces the cost relative to exact SrMMD, but has a higher final MMD and substantially larger variability across seeds.
Ranks $128$ and $256$ closely match the final MMD of exact SrMMD, although their time per iteration is higher, suggesting that the additional cost of the approximation outweighs its benefit at this matrix size.
The relative velocity error decreases from $0.3464$ at rank $32$ to $0.01165$ at rank $64$, and to approximately $10^{-4}$ and $10^{-6}$ at ranks $128$ and $256$, respectively.

For RFF, increasing the number of features from $32$ to $256$ reduces the final MMD from $0.09479$ to $0.02169$.
With $256$ features, the time per iteration is lower than that of exact SrMMD by a factor of $3.48$.
The initial relative Frobenius error of the kernel approximation decreases from $0.1176\pm0.0818$ to $0.0293\pm0.0243$, and the initial relative velocity error decreases from $0.282\pm0.119$ to $0.0810\pm0.0295$.
However, the final relative velocity error remains $0.7541\pm0.0599$ with $256$ features.
Thus, increasing the number of random features improves the final MMD, but the resulting velocity approximation remains less accurate than Nystr\"om on this benchmark.

\vspace{1mm}
\textbf{Wall-clock time and memory.}\hspace{4mm}
In \Cref{tab:runtime_memory}, we compare the runtime and peak memory usage of exact SrMMD flow and selected approximations with the vanilla MMD flow and DrMMD flow.
The total method time includes one-time JAX compilation and the full $200$-iteration optimization.
After $200$ iterations, exact SrMMD flow achieves a final MMD approximately $10.2$ and $13.9$ times smaller than MMD flow and DrMMD flow, respectively, at a higher cost per iteration.
The rank-$64$ Nystr\"om approximation retains comparable final accuracy at a lower cost per iteration.
RFF with $128$ and $256$ features reduces the time per iteration by factors of $7.82$ and $3.48$, respectively, while yielding final MMD values approximately $3.06$ and $1.98$ times those of exact SrMMD.
Both configurations achieve smaller final MMD values than the vanilla MMD flow and DrMMD flow.

\begin{table}[!htbp]
    \centering
    \caption{Runtime, peak memory usage, and final MMD on the Gaussian mixture benchmark. Time per iteration is measured after JAX compilation; total time includes compilation and $200$ optimization steps. Runtime and memory entries are means over $10$ random seeds; final MMD is reported as mean $\pm$ standard deviation.}
    \label{tab:runtime_memory}
    \small
    \setlength{\tabcolsep}{4pt}
    \begin{tabular}{@{}lrrrrr@{}}
        \toprule
        Method & Time/iteration & Total time & Peak CPU & Peak GPU & Final MMD \\
        & (s) & (s) & (MiB) & (MiB) & \\
        \midrule
        MMD & $0.000050$ & $0.1299$ & $480.53$ & $406$ & $0.11187\pm0.00539$ \\
        DrMMD & $0.000197$ & $0.7690$ & $689.07$ & $456$ & $0.15192\pm0.00465$ \\
        Exact SrMMD & $0.003454$ & $1.1337$ & $646.50$ & $482$ & $0.01093\pm0.00087$ \\
        SrMMD, Nystr\"om $r=64$ & $0.001578$ & $0.7998$ & $705.50$ & $504$ & $0.01098\pm0.00085$ \\
        SrMMD, RFF $R=128$ & $0.000442$ & $0.7376$ & $691.64$ & $450$ & $0.03344\pm0.00533$ \\
        SrMMD, RFF $R=256$ & $0.000992$ & $0.8603$ & $694.79$ & $450$ & $0.02169\pm0.00258$ \\
        \bottomrule
    \end{tabular}
\end{table}

The rank-$64$ Nystr\"om approximation does not reduce the measured peak memory usage relative to exact SrMMD.
RFF reduces peak GPU memory from $482$ MiB to $450$ MiB, while its peak CPU memory is higher.
These results therefore do not indicate a general memory advantage for the approximations at the scale considered here.

\begin{table}[!htbp]
    \centering
    \caption{MMD at the common post-compilation optimization-time horizon of $9.22$ ms on the Gaussian mixture benchmark. Entries show the median and the $25\%$--$75\%$ percentile range over $10$ random seeds. All MMD values are computed with the exact kernel.}
    \label{tab:fixed_budget}
    \small
    \begin{tabular}{@{}lr@{}}
        \toprule
        Method & MMD at the common time horizon \\
        \midrule
        MMD & $0.12660\;[0.11637,0.13007]$ \\
        DrMMD & $0.15216\;[0.14928,0.15482]$ \\
        Exact SrMMD & $0.95005\;[0.94722,0.95088]$ \\
        SrMMD, Nystr\"om $r=32$ & $0.93552\;[0.93146,0.93757]$ \\
        SrMMD, Nystr\"om $r=64$ & $0.83603\;[0.83244,0.83971]$ \\
        SrMMD, Nystr\"om $r=128$ & $0.95687\;[0.95407,0.95768]$ \\
        SrMMD, Nystr\"om $r=256$ & $0.98113\;[0.97857,0.98199]$ \\
        SrMMD, RFF $R=32$ & $0.38136\;[0.32539,0.41268]$ \\
        SrMMD, RFF $R=64$ & $0.38784\;[0.34902,0.42402]$ \\
        SrMMD, RFF $R=128$ & $0.42884\;[0.40841,0.44168]$ \\
        SrMMD, RFF $R=256$ & $0.72221\;[0.70681,0.73046]$ \\
        \bottomrule
    \end{tabular}
\end{table}

\vspace{1mm}
\textbf{Comparison under a fixed time budget.}\hspace{4mm}
We also compare all methods at a common post-compilation optimization-time horizon of $9.22$ ms, determined by the shortest completed trajectory.
In \Cref{tab:fixed_budget}, we report the median MMD and the $25\%$--$75\%$ percentile range over $10$ random seeds.
The vanilla MMD flow achieves the smallest MMD under this short time budget.
Among the SrMMD variants, RFF with $32$ features performs best at this horizon, although its final MMD after $200$ iterations is higher than that of the other SrMMD variants.
Thus, the lower final MMD of exact SrMMD after a fixed number of iterations does not imply better performance under every wall-clock budget.

\vspace{1mm}
\textbf{Empirical second moments.}\hspace{4mm}
Finally, we examine the final empirical second moments as a numerical stability check.
Across all $110$ runs, the final empirical second moments are finite and range from $7.878$ to $8.868$.
The mean $\pm$ standard deviation is $8.425\pm0.170$ for exact SrMMD, $8.424\pm0.169$ for the rank-$64$ Nystr\"om approximation, and $8.382\pm0.184$ for RFF with $256$ features.

\subsection{Student--teacher network}\label{sec:more_student}
The two-layer neural network used in this experiment is
$\psi(z,\theta)=G\bigl(b^1+W^1\varsigma(W^0z+b^0)\bigr)$,
where $\varsigma$ here is the ReLU activation function and $G(x)=\exp(-x^2/4)$. Here, $\theta=(b^1,W^1,b^0,W^0)\in\R^d$ denotes the concatenation of all network parameters, with $d=1+1+50+1=53$. The teacher distribution is $\pi=\calN(0,\Id)$ on $\R^{53}$, from which we observe $M=10$ i.i.d. samples and then fixed througout. 
The student distribution is initialized with $N=100$ particles sampled independently from $\calN(0,0.1\Id)$.

As shown in the main text, training the mean-field studnet teacher network here amounts to MMD flow with kernel $k(\theta,\theta')=\E_{z\sim\Pb_{\mathrm{data}}}[\psi(z,\theta)^\top\psi(z,\theta')]$. Specifically, the squared discrepancy between the mean teacher and student networks can be written as:
\begin{align}
&\E_{z\sim\Pb_{\mathrm{data}}}
\left[
\left\|
\Psi_\pi^\mathsf{T}(z)-\Psi_\mu^\mathsf{S}(z)
\right\|^2
\right] =
\E_{z\sim\Pb_{\mathrm{data}}}
\left[
\left\|
\E_{\theta\sim\pi}[\psi(z,\theta)]
-
\E_{\theta'\sim\mu}[\psi(z,\theta')]
\right\|^2
\right] \notag \\
&\quad =
\E_{\substack{\theta\sim\pi\\ \theta'\sim\pi}}
[k(\theta,\theta')]
+
\E_{\substack{\theta\sim\mu\\ \theta'\sim\mu}}
[k(\theta,\theta')]
-
2\E_{\substack{\theta\sim\pi\\ \theta'\sim\mu}}
[k(\theta,\theta')] \notag \\
&\quad =
\iint k(\theta,\theta')\,\dd(\pi-\mu)(\theta)\,\dd(\pi-\mu)(\theta')
=
\mmd^2(\mu,\pi).
\end{align}
In practice, the expectation with respect to $\Pb_{\mathrm{data}}$ is approximated using an empirical average over $1000$ i.i.d. training samples, and performance is evaluated on a held-out set of $1000$ independent validation samples. During training, the kernel $k(\theta,\theta')$ is approximated at each iteration by averaging over a randomly selected subset of $100$ training samples.

In the right panel of \Cref{fig:ablation_mog}, we report an ablation study on the regularization parameter $\lambda$ for the SrMMD flow. Since the kernel used in this experiment is defined through the network features, it does not involve a kernel lengthscale. The best-performing value is $\lambda=0.1$. For the DrMMD flow baseline, we use the same value $\lambda=0.1$, following the experimental results of \citet{chen2025regularized}. 

\subsection{Color transfer}\label{sec:more_color}
All methods use a fixed Gaussian kernel
$k_\sigma(x,y)=\exp(-0.5\|x-y\|^2)$ and use a fixed step size $\gamma=0.01$. The regularization parameter $\lambda$ is selected by searching over the pre-specified candidate sets
$\{0.01, 0.03, 0.1, 0.3, 1.0, 3.0, 10.0\}$. 
The final selected regularization parameter for SrMMD flow is $\lambda=0.01$. 
To transfer colors, for each source pixel with color $c$, we identify the nearest particle in the initial particle set $Y_0$, $i = \arg\min_j \|c - Y_{0,j}\|^2$, 
and assign the pixel the color of the corresponding transported particle $Y_{T,i}$. The resulting image is shown in the right panel of \Cref{fig:st_and_color}.

    
\begin{figure}
    \centering
    \includegraphics[height=0.25\linewidth]{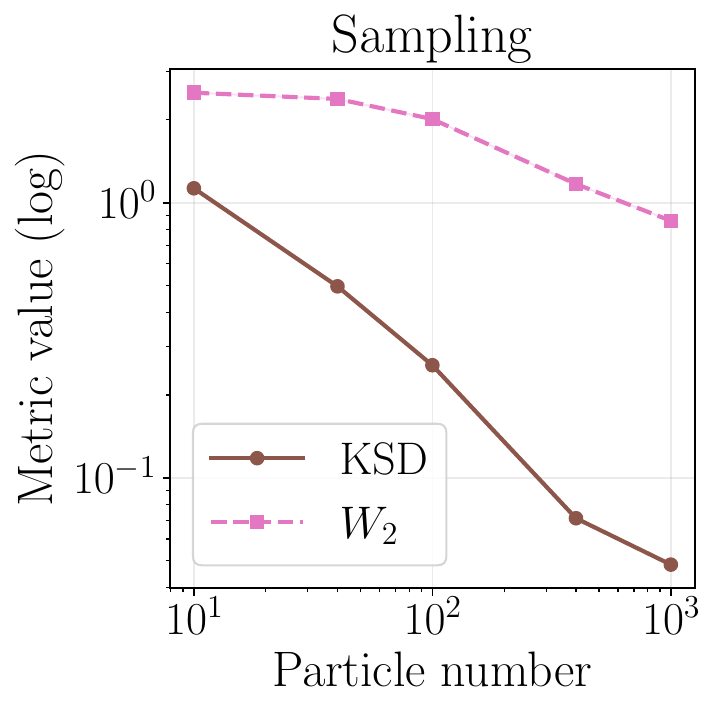}
    \includegraphics[height=0.25\linewidth] {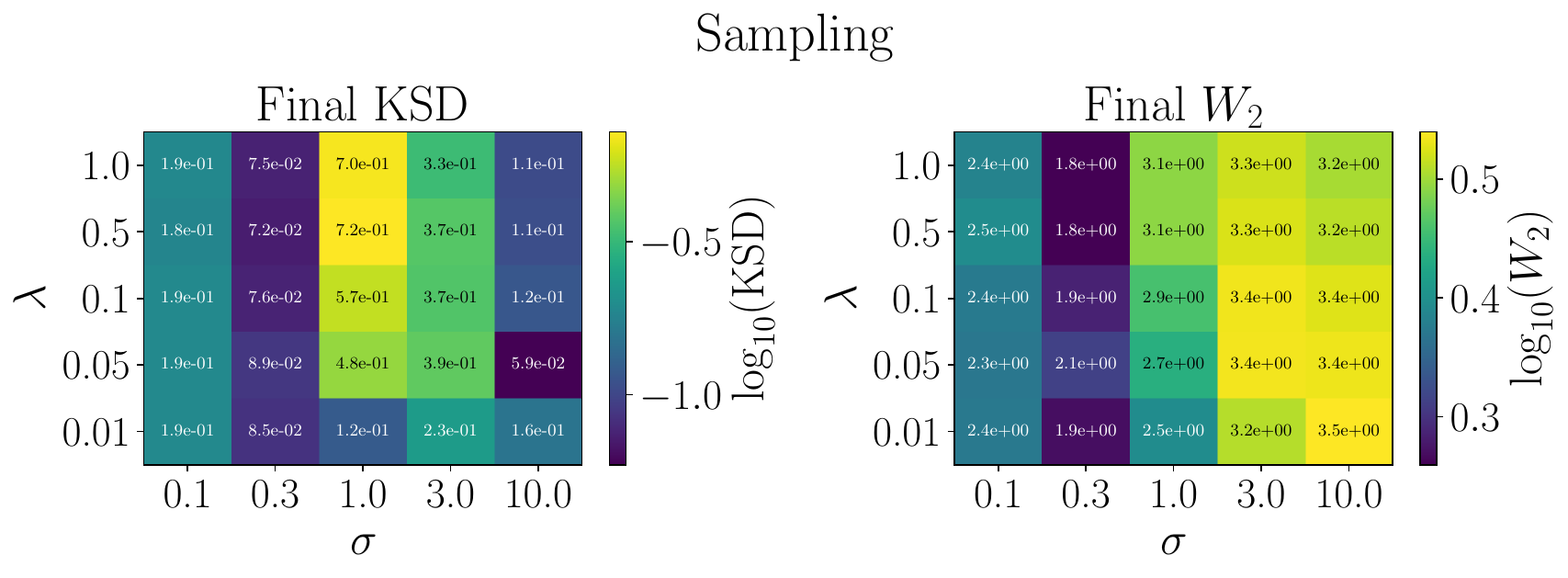}
    \caption{Ablation studies of SrMMD flow as a \emph{sampling} method on the mixture of Gaussians, with respect to the number of particles (\textbf{Left}), lengthscale $\sigma$ and regularization strength $\lambda$ (\textbf{Middle \& Right}). }
    \label{fig:ablation_sampling}
\end{figure}

\subsection{Sampling: Gaussian mixture and Bayesian logistic regression}\label{sec:more_sampling}

\vspace{1mm}
\textbf{Gaussian mixture}
Our Gaussian mixture setting is the same as the ones considered in \cite{chen2025stationary,chen2010super}, which is a mixture of 10 Gaussians in $d=2$ dimensions. In the sampling literature, this is the simplest setting where the target distribution $\pi$ satisfies a  Log-Sobolev or Poincar\'{e} inequality yet the corresponding constants increase exponentially with the distance between distinct modes. The score function is computed by taking the gradient of the log density. 
For SrMMD, KSD and HrMMD flow, we use the Stein kernel constructed from the Gaussian kernel as the base. For SVGD and R-SVGD, we use the Gaussian kernel. For a fair comparison, all these methods use the same kernel lengthscale $\sigma=0.3$. 
For SrMMD, HrMMD, and KSD flows, we use the best tuned step sizes $\gamma=0.01$ for KSD and $\gamma=0.1$ for SrMMD and HrMMD. 
For SVGD and R-SVGD, we use the best-tuned step size $\gamma=0.01$. All the step sizes are tuned from a pre-specified candidate set $\{0.01, 0.03, 0.1, 0.3, 1.0\}$.

In \Cref{fig:ablation_sampling}, we provide an ablation study of our SrMMD flow. 
The left panel shows that
both final KSD and $W_2$ decrease as the particle number $N$ increases, indicating that our particle implementation benefits from a more accurate empirical approximation of the population descent scheme, as desired. 
The middle and right panel shows that the best configuration for SrMMD flow is $(\sigma,\lambda)=(0.3, 0.5)$. 

\vspace{1mm}
\textbf{Bayesian logistic regression}
We use four datasets from the UCI Machine Learning Repository~\citep{kelly2024uci}:  \textsc{Breast Cancer} ($d=30$, $569$ samples), \textsc{Ionosphere} ($d=34$, $351$ samples), \textsc{German Credit} ($d=24$, $1000$ samples), and \textsc{Covtype} ($d=54$, $30,000$ samples). 
Fro all the datasets, $\frac{2}{3}$ are used as training samples and $\frac{1}{3}$ are used as validation samples. 
For SrMMD and KSD flow, we use the Stein kernel constructed from the Gaussian kernel as the base. For SVGD, we use the Gaussian kernel. For a fair comparison, all these methods use the same kernel lengthscale $\sigma=1.0$. 
We use the L-BFGS optimizer~\citep{liu1989limited} for KSD flow as suggested by~\cite{korba2021kernel}, with initial step size 0.01 and convergence tolerance $10^{-3}$. SrMMD and SVGD use a fixed-step Euler update with step size $\gamma=0.1$, selected from the candidate set $\{0.01,0.03,0.1,0.3,1.0\}$. All methods operate in the full-batch regime (no minibatching) and are run for 3,000 iterations until convergence is empirically observed.

\begin{figure}[t]
    \centering
    \includegraphics[width=0.49\linewidth]{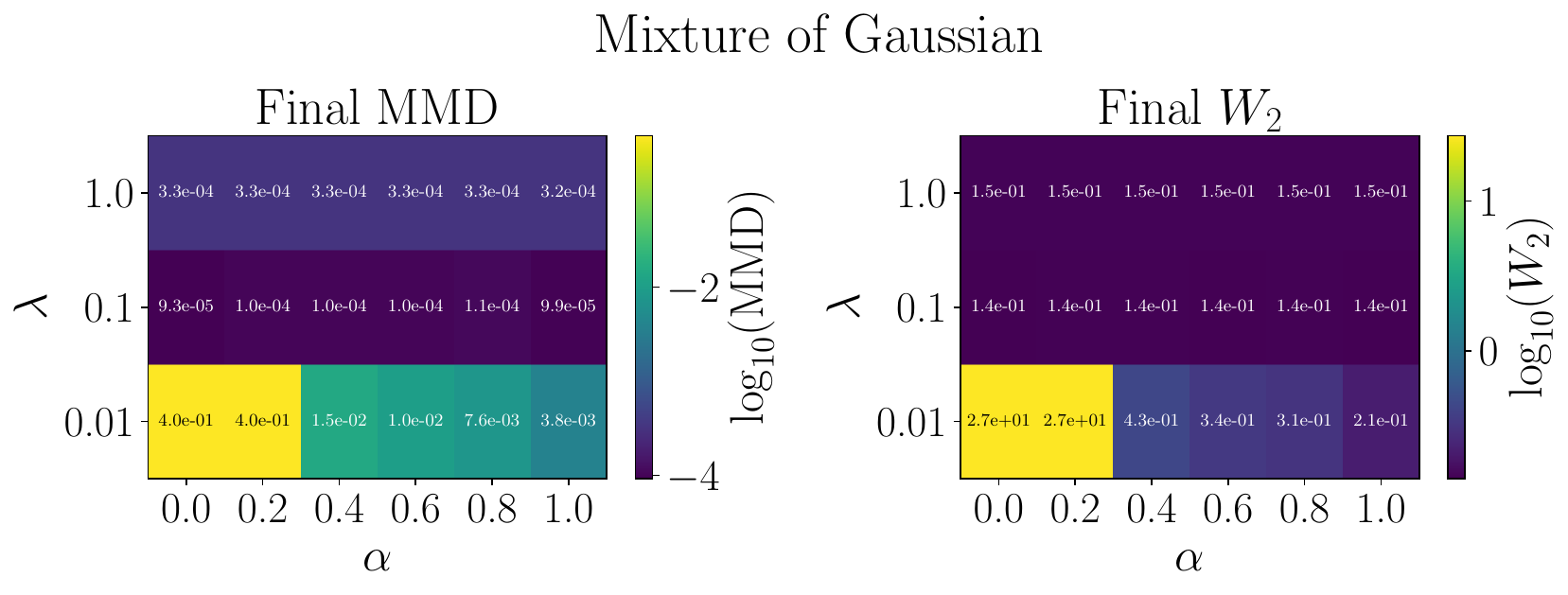}
    \includegraphics[width=0.49\linewidth]{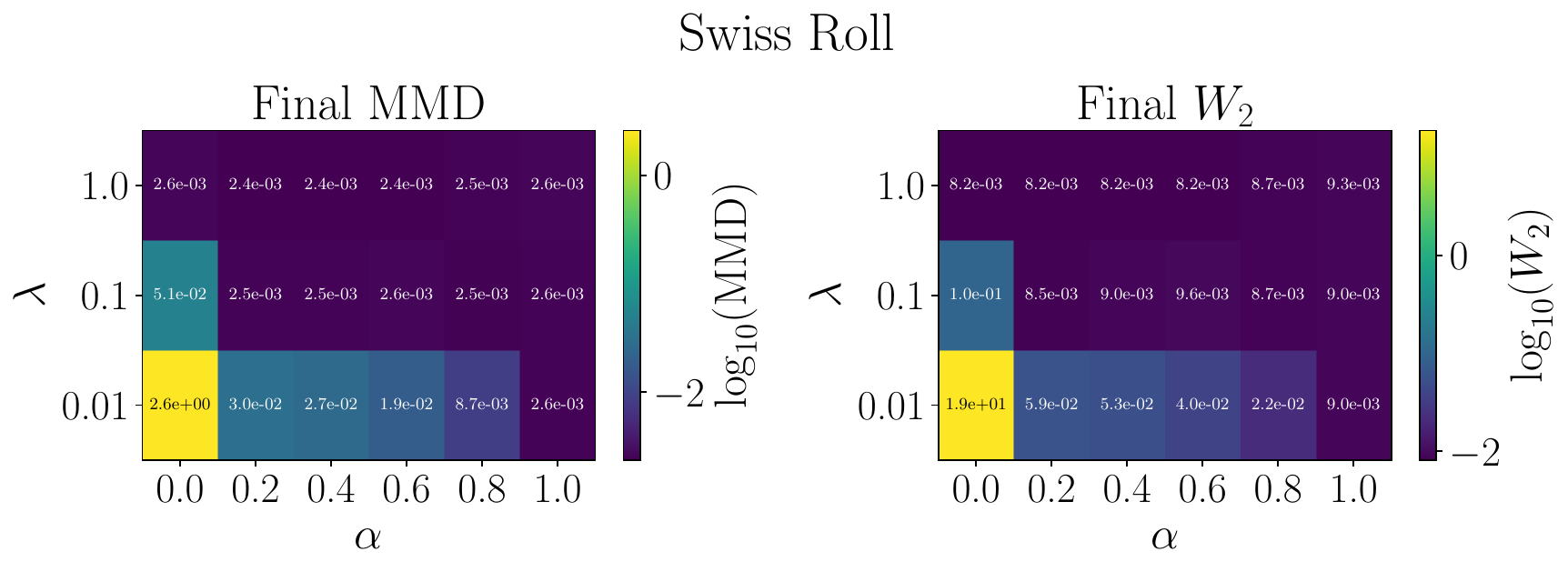}
     \includegraphics[height=0.20\linewidth]{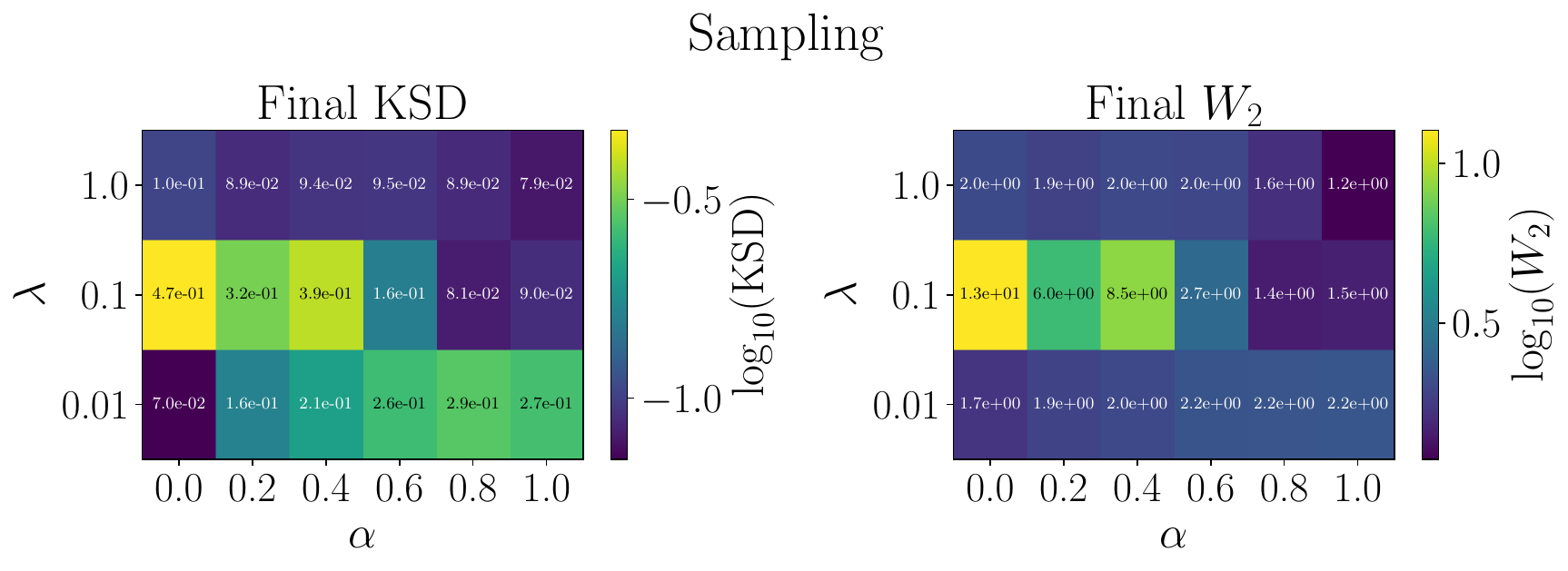}
    
    \caption{Ablation of HrMMD flow with respect to $\lambda$ and $\alpha$ on mixture of Gaussians, Swiss roll and sampling. }
    \label{fig:hrmmd_heatmaps}
\end{figure}

\subsection{Hybrid-regularized maximum mean discrepancy (HrMMD) flow}\label{sec:hrmmd}
In this section, we discuss our Hybrid-regularized maximum mean discrepancy (HrMMD), which is another regularized variant of MMD with both $L_2$ norm and gradient norm regularization on the witness function. 
Specifically, for $\alpha \in [0,1]$ and $\lambda>0$, we define
\begin{align*}
\mathrm{HrMMD}(\mu,\pi)
:=
\sup_{f\in\mathcal H:\;
\alpha\|\nabla f\|_{L_2^d(\mu)}^2
+
(1-\alpha)\|f\|_{L_2(\mu)}^2
+
\lambda\|f\|_{\mathcal H}^2
\le 1}
\int f\, \dd(\mu-\pi).
\end{align*}
This formulation interpolates continuously between two previously discussed regularizations. When $\alpha=1$, HrMMD reduces to our SrMMD in \Cref{defi:smmd}. 
When $\alpha=0$, only the $L_2$ norm remains and it reduces to DrMMD~\citep{chen2025regularized}. 
However, a slight distinction here remains since DrMMD penalizes the $L_2(\pi)$ norm while here it penalizes the $L_2(\mu)$ norm. 
HrMMD provides an interpolation of these two types of regularization, and it is also exactly the regularization used in the actual MMD GANs~\citep{arbel2018gradient}. 
\textcolor{red}{The HrMMD objective is the same as the Sobolev-Fisher discrepancy studied by 
\cite{mroueh2020unbalanced}, where it was subsequently used as the objective to be minimized in the 
Wasserstein--Fisher--Rao geometry.}

The corresponding witness function admits the closed-form representation
\begin{align*}
f_{\mu,\pi}^{\mathrm{HrMMD}}
=
(\alpha S_\mu + (1-\alpha) \Sigma_\mu + \lambda \Id)^{-1}(m_\mu-m_\pi),
\end{align*}
where $S_\mu=D_\mu^*D_\mu$ is the gradient covariance operator defined in \Cref{sec:convergence} in the main text, and $\Sigma_\mu=\E_{x\in\mu}[k(x,\cdot)\otimes k(x,\cdot)]$ is the covariance operator from $\calH\to\calH$ associated with the kernel $k$~\citep{muandet2017kernel}. 
Similarly as our SrMMD, the witness function $f_{\mu,\pi}^{\mathrm{HrMMD}}$ of HrMMD also admits a tractable particle-level closed form expression. Let $\hat\mu=\frac1N\sum_{i=1}^N \delta_{x^{(i)}}$
be the empirical distribution of the particles, and let
\begin{align*}
[K_{XX}]_{ij}=k(x^{(i)},x^{(j)}), \;
[D_{XX}]_{(i,\ell),j}=\partial_{1,\ell}k(x^{(i)},x^{(j)}), \;
[H_{XX}]_{(i,\ell),(j,m)}=\partial_{1,\ell}\partial_{2,m}k(x^{(i)},x^{(j)}).
\end{align*}
We also denote
\begin{align*}
\mathbb{K}_X(z) = [k(x^{(1)},z),\dots,k(x^{(N)},z)]^\top,\quad
\mathbb{D}_X(z) = [\partial_{1,1}k(x^{(1)},z),\dots,\partial_{1,d}k(x^{(N)},z)]^\top,
\end{align*}
and let $\mathbf 1_N\in\mathbb R^N$ be the vector of all ones. Then the empirical witness function can be written as
\begin{align*}
f_{\hat\mu,\hat\pi}^{(\alpha)}(z)
=
\frac1\lambda\left(
\frac1N\mathbf 1_N^\top \mathbb{K}_X(z)-\E_{\pi}[k(Y, z)]
\right)
-\frac1\lambda\Big(
\sqrt{1-\alpha}\,\beta_k^\top \mathbb{K}_X(z)
+
\sqrt{\alpha}\,\beta_d^\top \mathbb{D}_X(z)
\Big),
\end{align*}
where $(\beta_k,\beta_d)$ is determined by the linear system
\begin{align*}
\left(
\begin{bmatrix}
(1-\alpha)K_{XX} & \sqrt{\alpha(1-\alpha)}\,D_{XX}^\top\\
\sqrt{\alpha(1-\alpha)}\,D_{XX} & \alpha H_{XX}
\end{bmatrix}
+
N\lambda \Id
\right)
\begin{bmatrix}
\beta_k\\
\beta_d
\end{bmatrix}
=
\begin{bmatrix}
\sqrt{1-\alpha}\,g\\
\sqrt{\alpha}\,r
\end{bmatrix},
\end{align*}
with
$$g=\frac1N K_{XX}\mathbf 1_N-\E_{z\sim\pi}[\mathbb{K}_X(z)],\quad
r=\frac1N D_{XX}\mathbf 1_N - \E_{z\sim\pi}[\mathbb{D}_X(z)].$$
In the generative modeling setting, where $\pi$ is represented by i.i.d. target samples $\{y^{(j)}\}_{j=1}^M$, all the expectations with respect to $\pi$ above are replaced by empirical averages over these samples. In the sampling setting with a Stein kernel, the target mean embedding vanishes under Stein's identity, which further simplifies the expression. 

In \Cref{fig:hrmmd_heatmaps}, we report ablation studies for HrMMD flow with respect to $\alpha$ and $\lambda$ on the mixture of Gaussians, Swiss roll, and sampling experiments. 
We vary the interpolation parameter $\alpha$ over $\{0.0, 0.2, 0.4, 0.6, 0.8, 1.0\}$ and the regularization parameter $\lambda$ over $\{0.01, 0.1, 1.0\}$. 
Overall, HrMMD flow tends to perform better when the gradient penalty is stronger. 
This suggests that the gradient penalty plays a more important role than the $L_2$ penalty, consistent with empirical observations in MMD GANs~\citep{arbel2018gradient,binkowski2018demystifying}. 
Nevertheless, the optimum is not always attained at $\alpha=1$, indicating that a small additional $L_2$ penalty can further improve performance.



\end{appendices}

\newpage
\section*{NeurIPS Paper Checklist}

\begin{enumerate}

\item {\bf Claims}
    \item[] Question: Do the main claims made in the abstract and introduction accurately reflect the paper's contributions and scope?
    \item[] Answer: \answerYes{}
    \item[] Justification: The abstract and introduction state the main contributions: the SrMMD flow, its continuous- and discrete-time convergence guarantees, its relation to MMD/KSD flows, and its empirical evaluation. 

\item {\bf Limitations}
    \item[] Question: Does the paper discuss the limitations of the work performed by the authors?
    \item[] Answer: \answerYes{}
    \item[] Justification: The paper discusses several limitations, including the weaker metric used for sampling guarantees, the source-condition regularity assumption, the saturation effect of Tikhonov regularization, excluded singular kernels, and the absence of a finite-particle convergence analysis for the implemented particle scheme.

\item {\bf Theory assumptions and proofs}
    \item[] Question: For each theoretical result, does the paper provide the full set of assumptions and a complete (and correct) proof?
    \item[] Answer: \answerYes{}
    \item[] Justification: The main theoretical statements explicitly state their assumptions, including growth, local boundedness, moment, and source-condition assumptions. Proofs of the main results are provided in the appendix, together with supporting lemmas and propositions.

\item {\bf Experimental result reproducibility}
    \item[] Question: Does the paper fully disclose all the information needed to reproduce the main experimental results of the paper to the extent that it affects the main claims and/or conclusions of the paper (regardless of whether the code and data are provided or not)?
    \item[] Answer: \answerYes{}
    \item[] Justification: The main algorithms, baselines, metrics, and many hyperparameters are described in the main text and appendix. A fully anonymous github link to the code is also provided. 

\item {\bf Open access to data and code}
    \item[] Question: Does the paper provide open access to the data and code, with sufficient instructions to faithfully reproduce the main experimental results, as described in supplemental material?
    \item[] Answer: \answerYes{}
    \item[] Justification: A fully anonymous github link to the code is also provided. 

\item {\bf Experimental setting/details}
    \item[] Question: Does the paper specify all the training and test details (e.g., data splits, hyperparameters, how they were chosen, type of optimizer) necessary to understand the results?
    \item[] Answer: \answerYes{}
    \item[] Justification: Section 6 and Appendix C describe the experimental settings, including the benchmarks, baselines, kernels, particle numbers, iteration counts, step sizes, regularization parameters, optimizers, metrics, and ablation studies. 

\item {\bf Experiment statistical significance}
    \item[] Question: Does the paper report error bars suitably and correctly defined or other appropriate information about the statistical significance of the experiments?
    \item[] Answer: \answerYes{}
    \item[] Justification: The experimental figures aggregate results over multiple random seeds where applicable and report median curves with shaded 25\%--75\% percentile regions. The captions specify the aggregation convention used for the plotted variability.

\item {\bf Experiments compute resources}
    \item[] Question: For each experiment, does the paper provide sufficient information on the computer resources (type of compute workers, memory, time of execution) needed to reproduce the experiments?
    \item[] Answer: \answerNo{}
    \item[] Justification: The paper describes the algorithms and experimental hyperparameters, but it does not currently report the hardware type, memory, per-run execution time, or total compute used for the experiments. These details should be added to the appendix or supplementary material if available.

\item {\bf Code of ethics}
    \item[] Question: Does the research conducted in the paper conform, in every respect, with the NeurIPS Code of Ethics \url{https://neurips.cc/public/EthicsGuidelines}?
    \item[] Answer: \answerYes{}
    \item[] Justification: The work is methodological and theoretical, uses synthetic or standard benchmark data, and does not involve human-subject experiments, private data, or high-risk deployment. To the best of our knowledge, the research conforms to the NeurIPS Code of Ethics.

\item {\bf Broader impacts}
    \item[] Question: Does the paper discuss both potential positive societal impacts and negative societal impacts of the work performed?
    \item[] Answer: \answerNA{}
    \item[] Justification: The paper is primarily foundational work on convergence and regularization for MMD/KSD gradient flows, and does not propose a deployed system or a model intended for direct real-world decision making. Any societal impacts are indirect and depend on downstream applications of sampling or generative modelling methods.

\item {\bf Safeguards}
    \item[] Question: Does the paper describe safeguards that have been put in place for responsible release of data or models that have a high risk for misuse (e.g., pre-trained language models, image generators, or scraped datasets)?
    \item[] Answer: \answerNA{}
    \item[] Justification: The paper does not release high-risk pretrained models, scraped datasets, or other assets with substantial misuse potential. The experiments use synthetic examples and standard benchmark tasks.

\item {\bf Licenses for existing assets}
    \item[] Question: Are the creators or original owners of assets (e.g., code, data, models), used in the paper, properly credited and are the license and terms of use explicitly mentioned and properly respected?
    \item[] Answer: \answerYes{}
    \item[] Justification: Existing methods, software packages, and benchmark datasets are credited through citations. 

\item {\bf New assets}
    \item[] Question: Are new assets introduced in the paper well documented and is the documentation provided alongside the assets?
    \item[] Answer: \answerYes{}
    \item[] Justification: A fully anonymous github link to the code is also provided, including usage instructions, environment information, and a license.

\item {\bf Crowdsourcing and research with human subjects}
    \item[] Question: For crowdsourcing experiments and research with human subjects, does the paper include the full text of instructions given to participants and screenshots, if applicable, as well as details about compensation (if any)?
    \item[] Answer: \answerNA{}
    \item[] Justification: The paper does not involve crowdsourcing or research with human subjects.

\item {\bf Institutional review board (IRB) approvals or equivalent for research with human subjects}
    \item[] Question: Does the paper describe potential risks incurred by study participants, whether such risks were disclosed to the subjects, and whether Institutional Review Board (IRB) approvals (or an equivalent approval/review based on the requirements of your country or institution) were obtained?
    \item[] Answer: \answerNA{}
    \item[] Justification: The paper does not involve crowdsourcing or research with human subjects, so IRB approval or equivalent review is not applicable.

\item {\bf Declaration of LLM usage}
    \item[] Question: Does the paper describe the usage of LLMs if it is an important, original, or non-standard component of the core methods in this research? Note that if the LLM is used only for writing, editing, or formatting purposes and does \emph{not} impact the core methodology, scientific rigor, or originality of the research, declaration is not required.
    \item[] Answer: \answerNA{}
    \item[] Justification: The core method development, theory, and experiments do not use LLMs as an important, original, or non-standard component. Any use of LLMs, if present, is limited to writing, editing, or formatting and does not affect the scientific methodology.

\end{enumerate}

\end{document}